\documentclass{article}
\usepackage[accepted]{icml2022}
\usepackage{tabularx}
\usepackage{xparse}
\usepackage{amsmath}
\usepackage{amsthm}
\usepackage{amssymb}
\usepackage{xspace}
\usepackage{relsize}
\usepackage{graphicx}
\usepackage{multirow}
\usepackage{comment}
\usepackage{math_commands}
\usepackage{enumitem}
\setlist[enumerate]{leftmargin=2em,noitemsep,topsep=0pt,parsep=0pt,partopsep=0pt}
\setlist[itemize]{noitemsep,topsep=0pt,parsep=0pt,partopsep=0pt} %leftmargin=1.5em
\usepackage{booktabs} % for professional tables
\usepackage{tikz}
\usetikzlibrary{decorations.pathreplacing}
\usetikzlibrary{backgrounds}

\usetikzlibrary{external}
\usepackage{mathabx}
\usepackage{mathrsfs}
\usepackage{xfrac}

\usepackage{multirow}

\usepackage{microtype}
\usepackage{subfig}

\usepackage{xr}

\usepackage{pgfplots}
\usepackage{xcolor}

\usetikzlibrary{arrows,intersections}
\usetikzlibrary{shapes}

\usepgfplotslibrary{fillbetween}
\pgfplotsset{compat=1.17}

\newcommand{\betweenline}[3]{\draw [black,->] (#1, {f(#1)}) node[dot]{} -- node[nodelbl] {#2} (#1, {#3(#1)}) ;}
\newcommand{\deltaline}[3]{\draw [black,<->] (#1, {f(-0.3413)}) -- node[nodelbl] {#2} (#1, {clipping(#1,#3)}) ;}

\pgfmathsetmacro\targets{-0.5}
\pgfmathsetmacro\targete{0.5}

\pgfmathsetmacro\domains{-1}
\pgfmathsetmacro\domaine{1}

\pgfmathsetmacro\domainxmax{1.1}
\pgfmathsetmacro\domainxdomain{2}

\pgfmathsetmacro\flblpos{0.32}
\pgfmathsetmacro\illscale{0.60}

\newcommand{\illustration}[4]{{\begin{tikzpicture}[declare function={
	f(\x)=-sin((12*\x-4)*180/3.141592653589793)*(6*\x-2)^2;
  h_rnd(\x)= 0 + (\x<1) * 1.2 * exp(-1/(1-(\x)^2)+1);
  h_sq(\x)= 0 + (\x<1) * 1.2;
	agg(\x)=f(\x)-34.55205898257652-7;
  aggt(\x)=f(\x)-min((abs(\x)-0.5)*4,1)*(34.55205898257652+10);
  clipping(\x,\d)=and(\targets<\x, \x<\targete) * f(\x) + or(\targets>\x, \x>\targete) * min(f(\x),15.908745487888613-\d);
  f_rnd(\x)=f(\x)-34.55205898257652*h_rnd((\x+0.856363)/0.15);
  f_sq(\x)=f(\x)-34.55205898257652*h_sq((\x+0.856363)/0.15);
  },
    dot/.style = {
      solid,
      draw,
      fill = white,
      circle,
      inner sep = 0pt,
      minimum size = 4pt
    },
    nodelbl/.style = {
      scale=\illscale,
      rectangle,
      fill = white,
      inner xsep=0cm, 
      inner ysep=1pt,
      midway,
    },]
\begin{axis}[
    samples=512, %TODO Change to higher number for final print
    domain=\domains:\domaine,
    ymin=#3, ymax=60,
    xmax=\domainxmax,
    axis lines=left,
    y=2cm/#4,
    x=7.75cm/\domainxdomain,
    grid=none,
    ticks=none,
    compat=newest,
    thick,
    axis line style = thick,
    axis on top=true,
]
	
	\addplot [black,smooth,name path=F] {f(x)} node[right,pos=\flblpos] {$f$};

  #1
  
  \draw [black,dotted,semithick] (\pgfkeysvalueof{/pgfplots/xmin}, {f(-0.3413)}) -- (-0.3413, {f(-0.3413)});
  
  #2

	\path[name path=ax_x] (\pgfkeysvalueof{/pgfplots/xmin},\pgfkeysvalueof{/pgfplots/ymin}) -- (\pgfkeysvalueof{/pgfplots/xmax},\pgfkeysvalueof{/pgfplots/ymin});

  \addplot [black,opacity=0,smooth,name path=F_limited,domain=\targets:\targete] {f(x)};
  \addplot [black,opacity=0,smooth,name path=ax_x_limited,domain=\targets:\targete] {#3};
	\addplot [gray!60] fill between [of=F_limited and ax_x_limited];
	
  \coordinate (0s) at (\targets,\pgfkeysvalueof{/pgfplots/ymin});
  \coordinate (0e) at (\targete,\pgfkeysvalueof{/pgfplots/ymin});
	\coordinate (As) at (\targets,{f(\targets)});
	\coordinate (Ae) at (\targete,{f(\targete)});
	\node at (barycentric cs:0s=1,0e=1,Ae=1,As=1) {$\targetr$};
\end{axis}
  \path
  ([shift={(-5\pgflinewidth,-5\pgflinewidth)}]current bounding box.south west);
\end{tikzpicture}}}

\makeatletter
\let\@myref\ref

\newcommand{\refsec}[1]{Sec.\,\@myref{#1}}
\newcommand{\refseq}[1]{Sec.\,\@myref{#1}}
\newcommand{\refig}[1]{Fig.\,\@myref{#1}}
\newcommand{\refigs}[2]{Fig.\,\@myref{#1}-\@myref{#2}}
\newcommand{\reftbl}[1]{Table \@myref{#1}}
\newcommand{\refstep}[1]{Step \@myref{#1}}
\newcommand{\refalgo}[1]{Alg. \@myref{#1}}
\newcommand{\refchap}[1]{Chapter \@myref{#1}}
\newcommand{\reflst}[1]{List \@myref{#1}}
\newcommand{\refeq}[1]{(\@myref{#1})}
\newcommand{\refthm}[1]{Thm. \@myref{#1}}
\newcommand{\refcor}[1]{Cor. \@myref{#1}}

\makeatother

\newcounter{list}[section]

\theoremstyle{plain}

\theoremstyle{plain}

\theoremstyle{plain}

\theoremstyle{plain}
\newtheorem{thm}{\protect\theoremname}
\theoremstyle{plain}
\newtheorem{cor}{\protect\corollaryname}  
\theoremstyle{definition}

\theoremstyle{definition}

\theoremstyle{definition}

\makeatother
  
\usepackage{babel} 

\providecommand{\claimname}{Claim}
\providecommand{\lemmaname}{Lemma}
\providecommand{\propositionname}{Proposition}
\providecommand{\theoremname}{Theorem}
\providecommand{\corollaryname}{Corollary} 
\providecommand{\definitionname}{Definition}
\providecommand{\assumptionname}{Assumption}
\providecommand{\remarkname}{Remark}

\newcommand{\todo}[1]{{\bf \color{red} [TODO] #1 }}

\newcommand{\kMat}{k_{\text{Mat}}}

\newcommand{\xvtilde}{\widetilde{\mathbf{x}}}
\newcommand{\xv}{\mathbf{x}}

\newcommand{\yv}{\mathbf{y}}

\newcommand{\Fc}{\mathcal{F}}

\newcommand{\RR}{\mathbb{R}}

\newcommand{\Rc}{\mathcal{R}}

\newcommand{\Iv}{\mathbf{I}}

\newcommand{\Kv}{\mathbf{K}}
\newcommand{\kv}{\mathbf{k}}

\newcommand{\bzero}{\boldsymbol{0}}

\newcommand{\ie}{i.e.\xspace}
\newcommand{\eg}{e.g.\xspace}

\newcommand{\ft}{\tilde{f}}
\newcommand{\targetr}{\Rc_{\rm target}}

\newcommand{\N}{\mathcal{N}}

\newcommand{\fmin}{f_{\min}}
\newcommand{\fmax}{f_{\max}}

\newcommand{\matern}{Mat\'ern\xspace}
\newcommand{\rbf}{RBF\xspace}
\newcommand\maternx [2]{Mat\'ern\,$\sfrac{#1}{#2}$\xspace}

\newcommand{\gpucb}{GP-UCB\xspace}
\newcommand{\maxvar}{MaxVar\xspace}
\newcommand{\noprefit}{Online-Learned Kernel\xspace}
\newcommand{\defense}{Defense\xspace}
\newcommand{\dynamic}{Dynamic\xspace}

\newcommand{\synthetic}{Synthetic1D\xspace}
\newcommand{\forrester}{Forrester1D\xspace}
\newcommand{\levy}{Levy1D\xspace}
\newcommand{\levyhard}{Levy-Hard1D\xspace}
\newcommand{\bohachevsky}{Bohachevsky2D\xspace}
\newcommand{\bohachevskyhard}{Bohachevsky-Hard2D\xspace}
\newcommand{\branin}{Branin2D\xspace}
\newcommand{\camelback}{Camelback2D\xspace}
\newcommand{\hartmann}{Hartmann6D\xspace}
\newcommand{\robotsmall}{Robot3D\xspace}
\newcommand{\robotlarge}{Robot4D\xspace}
\newcommand{\npsynthetic}{Synthetic1D-Online\xspace}
\newcommand{\npforrester}{Forrester1D-Online\xspace}

\newcommand{\npbohachevsky}{Bohachevsky2D-Online\xspace}
\newcommand\levydef [2]{Levy1D $\left(C=#1,\eta^2=#2\right)$\xspace}
\newcommand{\varforrester}{MaxVar-Forrester1D\xspace}
\newcommand{\varcamelback}{MaxVar-Camelback2D\xspace}
\newcommand{\varrobotsmall}{MaxVar-Robot3D\xspace}

\newcommand{\random}{$\mathsf{Random}$\xspace}
\newcommand{\noattack}{$\mathsf{No~Attack}$\xspace}
\newcommand{\aggressivesubtraction}{$\mathsf{Aggressive~Subtraction}$\xspace}
\newcommand{\clipping}{$\mathsf{Clipping}$\xspace}
\newcommand{\subtractionrnd}{$\mathsf{Subtraction~Rnd}$\xspace}
\newcommand{\shortsq}{$\mathsf{Sq}$\xspace}
\newcommand{\subtractionsq}{$\mathsf{Subtraction~Sq}$\xspace}

\DeclareMathOperator*{\successrate}{Success-Rate}
\DeclareMathOperator*{\normalizedcost}{Normalized-Cost}

\newcommand{\intersection}{\cap}

\newcommand{\braces}[1]{{\left\{#1\right\}}}
\newcommand{\parens}[1]{{\left(#1\right)}}
\newcommand{\sqparens}[1]{{\left[#1\right]}}
\newcommand{\bars}[1]{{\left|#1\right|}}
\newcommand{\dbars}[1]{{\left\|#1\right\|}}

\newcommand{\defun}[1]{%
\makeatletter
\expandafter\def\csname the#1\endcsname{\text{\it #1}}
\expandafter\def\csname #1\endcsname ##1{\csname the#1\endcsname\left(##1\right)}%
\makeatother
}

\newcommand{\defsetop}[2]{%
\makeatletter% variable, set, descriptor
 \expandafter\def\csname #1\endcsname ##1##2##3{%
  \expandafter\def\csname #1arg\endcsname{##1}%
  \expandafter\def\csname #1set\endcsname{##2}%
  \expandafter\def\csname #1cond\endcsname{##3}%
  \braces{##1##2\mid #2 ##3}%
 }%
\makeatother%
}

\defun{precond}
\defun{condition}
\defun{params}
\defun{type}
\defun{proc}
\defun{cost}

\defun{push}
\defun{pref}

\defun{init}
\defun{goal}
\defun{objects}
\defun{task}

\defun{parent}
\defun{plateau}

\defsetop{filter}{}
\defsetop{map}{}

\def\_{\\[-0.3em]}

\makeatletter

\newcommand{\newheuristic}[2]{%
 \def#1{%
  \ifmmode%
  h^\text{#2}\xspace%
  \else%
  \text{#2}\xspace%
  \fi%
 }%
}

\newheuristic{\lmcut}{LMcut}
\newheuristic{\mands}{M\&S}
\newheuristic{\pdb}{PDB}
\newheuristic{\ff}{FF}
\newheuristic{\ce}{CEA}
\newheuristic{\cg}{CG}
\newheuristic{\ad}{add}
\newheuristic{\lc}{LC}

\newcommand{\newUnitCostHeuristic}[2]{%
 \def#1{%
  \ifmmode%
  \hat{h}^\text{#2}\xspace%
  \else%
  \text{#2}\xspace%
  \fi%
 }%
}

\newUnitCostHeuristic{\lmcuto}{LMcut}
\newUnitCostHeuristic{\mandso}{M\&S}
\newUnitCostHeuristic{\ffo}{FF}
\newUnitCostHeuristic{\ceo}{CEA}
\newUnitCostHeuristic{\cgo}{CG}
\newUnitCostHeuristic{\ado}{add}
\newUnitCostHeuristic{\gco}{GoalCount}
\newUnitCostHeuristic{\lco}{LC}

\makeatother

\def\ref{\todo{Do not use ``ref'' directly!}}

\icmltitlerunning{Adversarial Attacks on Gaussian Process Bandits}

\begin{document}

\twocolumn[
\icmltitle{Adversarial Attacks on Gaussian Process Bandits}

\icmlsetsymbol{equal}{*}

\begin{icmlauthorlist}
\icmlauthor{Eric Han}{soc}
\icmlauthor{Jonathan Scarlett}{soc,ids}
\end{icmlauthorlist}

\icmlaffiliation{soc}{School of Computing, National University of Singapore}
\icmlaffiliation{ids}{Department of Mathematics \& Institute of Data Science, National University of Singapore}

\icmlcorrespondingauthor{Eric Han}{eric{\textunderscore}han@nus.edu.sg}
\icmlcorrespondingauthor{Jonathan Scarlett}{scarlett@comp.nus.edu.sg}

\icmlkeywords{Machine Learning, Online Learning \& Bandits, Adversarial Learning \& Robustness, ICML}

\vskip 0.3in
]

\printAffiliationsAndNotice{}  % leave blank if no need to mention equal contribution

\begin{abstract}
Gaussian processes (GP) are a widely-adopted tool used to sequentially optimize black-box functions, where evaluations are costly and potentially noisy. Recent works on GP bandits have proposed to move beyond random noise and devise algorithms robust to {\em adversarial attacks}. This paper studies this problem from the attacker's perspective, proposing various adversarial attack methods with differing assumptions on the attacker's strength and prior information. Our goal is to understand adversarial attacks on GP bandits from theoretical and practical perspectives. We focus primarily on \emph{targeted} attacks on the popular GP-UCB algorithm and a related elimination-based algorithm, based on adversarially perturbing the function $f$ to produce another function $\ft$ whose optima are in some target region $\targetr$. Based on our theoretical analysis, we devise both white-box attacks (known $f$) and black-box attacks (unknown $f$), with the former including a Subtraction attack and Clipping attack, and the latter including an Aggressive subtraction attack. We demonstrate that adversarial attacks on GP bandits can succeed in forcing the algorithm towards $\targetr$ even with a low attack budget, and we test our attacks' effectiveness on a diverse range of objective functions.
\end{abstract}

\section{Introduction}

Gaussian Processes (GPs) are commonly used to sequentially optimize unknown objective functions whose evaluations are costly.
This method has successfully been applied to a litany of applications, such as hyperparameter tuning \cite{snoek2012practical,swersky2013multi},
robotics \cite{jaquier2020bayesian}, recommender systems \cite{10.1145/2645710.2645733} and more.  In many of these applications, corruptions in the measurements are not sufficiently well captured by random noise alone.  For instance, one may be faced with rare outliers (e.g., due to equipment failures), or bad actors may influence the observations (e.g., malicious users in recommender systems).

These uncertainties have been addressed via the consideration of an \emph{adversary} in the GP bandit optimization problem; 
function observations are not only subject to random noise, but also adversarial noise. 
With this additional requirement, the optimization not only becomes more robust to uncertainty, but can also maintain robustness in the presence of malicious adversaries.
The notion of adversarial attacks on bandit algorithms appears to have been inspired by the extensive literature on adversarial attacks on deep neural networks \cite{42503}, although the associated approaches are generally different.

\subsection{Related Work}

A myriad of GP optimization methods have been developed in the literature to overcome the various forms of uncertainty, and achieve some associated notion of robustness:
\begin{itemize}
    \item \textit{Presence of outliers}, where some function evaluations are highly unreliable \cite{Mar18}.
    \item \textit{Random perturbations to sampled points}, where the sampled points are subject to random uncertainty \cite{Bel17,Nog16}.
    \item \textit{Adversarial perturbations to the final point}, where the final recommendation $x$ may be perturbed up to some level $\delta$ \cite{Ber10,Bog18}.
    \item \textit{Adversarial perturbations to samples}, where the observations are adversarially corrupted up to some maximum budget \cite{Bog20}.
\end{itemize}
These methods are primarily focused on proposing methods that defend against the proposed uncertainty model to improve robustness for GP optimization.
There has been minimal work studying the problem from an attacker's perspective in the literature, 
\ie, what kinds of attacks would be successful against non-robust algorithms.
We examine such a perspective in this work, focusing on adversarial perturbations to samples \cite{Bog20}.

Our study is related to that of attacks on stochastic linear bandits \cite{Gar20}, 
but we move to the GP setting which is inherently non-linear and poses substantial additional challenges.  We focus in particular popular algorithms into selecting from a specific set of (typically suboptimal) {\em target actions} as much as possible, using GP-UCB \cite{Sri09} and a related elimination-based algorithm as representative examples.

Prior to \cite{Gar20}, analogous works studied attacks and defenses for multi-armed bandits \cite{Jun18,Lyk18}.
These are less related here, since they assume finite domains with independent arms.

\subsection{Contributions}

The main contributions of this paper are as follows:
\begin{enumerate}
    \item We theoretically characterize conditions under which an adversarial attack can succeed against GP-UCB or elimination even with an attack budget far smaller than the time horizon, both in the cases of the function being known (white-box attack) and unknown (black-box attack) to the attacker.
    \item We present various attacks inspired by our analysis:
        \begin{enumerate}
          \item with knowledge of the function: Subtraction Attack (two variants), Clipping Attack.
          \item without knowledge of the function: Aggressive Subtraction Attack (two variants).
        \end{enumerate}
    We demonstrate the effectiveness of these attacks via experiments on a diverse range of objective functions.
\end{enumerate}
More broadly, we believe that our work fills an important gap in the literature by moving beyond finite domains and/or linear functions, and giving the first study of robust GP bandits from the attacker's perspective.

\section{Setup}

We consider the setup proposed in \cite{Bog20}, described as follows.
The player seeks to maximize an unknown function $f(\xv)$ over $\xv \in D$, 
and we model the smoothness of $f$ by assuming that it has RKHS norm at most $B$ according to some kernel $k$, 
\ie, $f \in \Fc_k(B)$ where $\Fc_k(B) = \{ f \,:\, \|f\|_k \le B\}$.  
We focus primarily on the widely-adopted \matern kernel, defined as
\begin{equation}
    \kMat(x,x') = \dfrac{2^{1-\nu}}{\Gamma(\nu)} \bigg(\dfrac{\sqrt{2\nu}\dbars{x - x'}}{l}\bigg)^{\nu}  J_{\nu}\bigg(\dfrac{\sqrt{2 \nu}\dbars{x - x'}}{l} \bigg), \label{eq:kMat}
\end{equation}
where $l>0$ denotes the length-scale, 
$\nu > 0$ is the smoothness parameter, 
and $J_{\nu}$ denotes the modified Bessel function.
This kernel provides useful properties that we can exploit in our theoretical analysis, 
and is also one of the most widely-adopted kernels in practice. 
In addition, it closely resembles the squared exponential (SE) kernel in the case that $\nu$ is large \cite{Ras06}.

In the non-corrupted setting (\eg, see \cite{Sri09,Cho17}), 
the player samples $\xv_t$ and observes $y_t = f(\xv_t) + z_t$, 
where $z_t \sim \N(0,\sigma^2)$ is random noise.  
In the presence of adversarial noise, we consider corrupted observations taking the form
\begin{equation}
    y_t  = f(\xv_t) + c_t + z_t,
\end{equation}
where $z_t \sim \N(0,\sigma^2)$, and $c_t$ is adversarial noise injected at time $t$. 
We consider $c_t$ as being chosen by an {\em adversary} or {\em attacker}.
In order to make the problem meaningful, the adversary's power should be limited, so we constrain
\begin{equation}
    \sum_{t=1}^n \bars{c_t} \le C \label{eq:budget}
\end{equation}
for some total corruption budget $C$.  
Following \cite{Bog20}, we adopt the same definition of regret as the uncorrupted setting: $R_T = \sum_{t=1}^T r_t$, where $r_t = f(\xv^*) - f(\xv_t)$, and 
where $\xv^*$ is any maximizer of $f$.  

We note that $c_t$ could also potentially be considered as part of the objective when defining regret; 
the two notions are closely related, and differ by at most $O(C)$ \cite{Lyk18,Bog20}.
For all of our results, this connection ensures that a successful attack (\ie linear regret in $T$) with respect to the above notion implies the same with respect to the alternative regret notion.

\subsection{Knowledge Available to the Adversary}

Naturally, the ability to attack/defend in the preceding setup may vary significantly depending on what is assumed to be known to the adversary. 
For instance, the adversary may or may not know $f$, know $\xv_t$, know which algorithm the player is using, and so on. 
In addition, as noted in the literature on robust bandit problems (\eg, \cite{Lyk18,Bog20a}), 
one may consider the case that the player can randomize $\xv_t$, and the adversary knows the distribution but not the specific choice.\footnote{In such a case, $c_t$ in \refeq{eq:budget} is typically replaced by $\max_{\xv} |c_t(\xv)|$, where $c_t(\cdot)$ is the adversary's corruption function at time $t$.}

In this paper, our focus is on attacking widely-adopted deterministic algorithms such as GP-UCB \cite{Sri09}, 
so we do not consider such randomization.
In this case, knowing the ``distribution of $\xv_t$'' becomes equivalent to knowing $\xv_t$, 
and we assume that such knowledge is available throughout the paper. 
We consider both the cases of $f$ being known and unknown to the attacker.

\subsection{Targeted vs.~Untargeted Attacks}

At this stage, we find it useful to distinguish between two types of attack.
In an {\em untargeted attack}, the adversary's sole aim is to make the player's cumulative regret as high as possible (e.g., $R_T = \Omega(T)$). 
In contrast, in a {\em targeted attack}, the adversary's goal is to make the player choose actions in a particular region $\targetr \subseteq D$. 
This generalizes the finite-arm notion of seeking to make the player pull a particular target arm \cite{Jun18}.

Of course, a targeted attack can also be used to ensure a large regret: 
If $\targetr$ satisfies the property that every $\xv \in \targetr$ has $r(\xv) = \Omega(1)$, 
then any attack that forces $\Omega(T)$ selections in $\targetr$ also ensures $R_T = \Omega(T)$.

More generally, to assess the performance of a targeted attack, we define the quantity after $T$ rounds
\begin{equation}
    N^{\rm target}_T = \sum_{t=1}^T \braces{ \xv_t \in \targetr },
\end{equation}
counting the number of arm pulls within the target region.

\section{Theoretical Study} \label{sec:theory}

This section introduces some attacks and gives conditions under which they provably succeed even when the budget $C$ is small compared to the time horizon $T$.
These results are not only of interest in their own right, 
but will also motivate several of our other attacks (without theory) in \refsec{sec:attacks}. 

Motivated by successful attacks on bandit problems \cite{Jun18,Gar20}, 
we adopt the idea of perturbing the function {\em outside} the $\targetr$ in a manner such that the perturbed function's maximizer is in the $\targetr$.
Then, a (non-robust) optimization algorithm will steer towards $\Rc_{\rm target}$ and stay there, with any points sampled in $\Rc_{\rm target}$ remaining unperturbed.

This idea is somewhat trickier to implement in the GP bandit setting than the finite-arm or linear bandit setting, and the details are given below.

\subsection{Optimization Algorithms} \label{sec:opt_algs}

Our attack methods can be applied to any GP bandit algorithm, and \refthm{thm:main} below states general conditions under which the attack succeeds. As two specific examples, we will consider the following widely-used algorithms that are representative of broader techniques in the literature:
\begin{itemize}
    \item \emph{GP-UCB} \cite{Sri09}: The $t$-th point is selected according to \begin{equation}
        \xv_t = \argmax_{\xv \in D} \mu_{t-1}(\xv) + \beta^{1/2}_t \sigma_{t-1}(\xv),
    \end{equation}
    for some suitably-chosen exploration parameter $\beta_t$, 
    where $\mu_{t-1}(\cdot)$ and $\sigma_{t-1}(\cdot)$ are the posterior mean and standard deviation after sampling $t-1$ points \cite{Ras06}.
    \item \emph{MaxVar + Elimination} \cite{contal2013parallel}: At each time, define the set of {\em potential maximizers} $M_t$ to contain all points whose UCB 
    (\ie, $\mu_{t-1}(\xv) + \beta^{1/2}_t \sigma_{t-1}(\xv)$) is at least as high as the highest LCB (\ie, $\mu_{t-1}(\xv) - \beta^{1/2}_t \sigma_{t-1}(\xv)$).
    Then, select the point in $M_t$ with the highest posterior variance.
\end{itemize}
For both of these algorithms, the exploration parameter $\beta_t$ is set to the value used in theoretical studies (e.g., \cite{Cho17}):
\begin{equation}
    \beta_t^{1/2} = B + \sigma \lambda^{-1/2} \sqrt{2(\lgamma_{t-1} + \ln(1 / \delta))}.
\end{equation}
where $\lambda$ is a free parameter (e.g., $\lambda = 1$), 
$\delta$ is the target error probability, 
and $\gamma_t$ is the maximum information gain at time $t$ \cite{Sri09}. 
The latter quantity scales as $\gamma_t = \tilde{O}\big( T^{\frac{d}{2\nu + d}} \big)$ for the \matern kernel \cite{Vak20a}.

The reason for focusing on these algorithms is that they have well-known guarantees on the regret in the uncorrupted setting, 
and such guarantees turn out to be a key ingredient in establishing the success of attacks in the corrupted setting. 
While these algorithms are far from exhaustive, they serve as representative examples of general non-robust algorithms.

\subsection{Conditions for a Successful Attack}

Our main theoretical result motivating our attacks is stated as follows.
This result applies to {\em any} algorithm that has non-robust guarantees of not playing too many strictly suboptimal actions.  The asymptotic notation $o(1)$ is defined with respect to the limit $T \to \infty$.

\begin{thm} \label{thm:main}
    Consider an adversary that performs an attack shifting the original function $f \in \Fc_k(B)$ to another function $\tilde{f}$ (\ie, set $c_t = \tilde{f}(\xv_t) - f(\xv_t)$ at time $t$ as long as the corruption budget permits it).
    Suppose that the following properties hold for some $\Delta > 0$ and $B_0 > 0$:
    \begin{enumerate}[label=(\roman*)]
        \item For any $\xv$ that is $\Delta$-optimal for $\tilde{f}$, it holds that both $\xv \in \targetr$ and $\tilde{f}(\xv) = f(\xv)$;
        \item For all $\xv \in D$, it holds that $|f(\xv) - \tilde{f}(\xv)| \le B_0$;
        \item It holds that $\|\tilde{f}\|_k \le B$ (\ie, $\tilde{f} \in \Fc_k(B)$).
    \end{enumerate}
    In addition, suppose that in the absence of an adversary, the optimization algorithm guarantees, for any $f \in \Fc_k$ and with probability at least $1-\delta$, that at most $N_0$ played actions are $\Delta$-suboptimal, for some $N_0$ depending on $(T,\Delta,\delta)$.

    Then, in the presence of the adversary, with probability at least $1-\delta$, the attack succeeds in forcing $T(1-o(1))$ actions to be played from $\targetr$, while using an attack budget $C$ of at most $B_0 N_0$.
\end{thm}

\begin{proof}
By the first assumed property, whenever the algorithm plays actions that are $\Delta$-optimal {\em with respect to $\tilde{f}$}, 
it holds that $c_t = 0$, and no budget is spent.
The third property allows us to bound the number of actions failing to satisfy such a property by $N_0$, 
where this bound holds with probability at least $1-\delta$ by the assumption on the algorithm.
The second property ensures that each action uses an amount of budget satisfying $|c_t| \le B_0$, and the result follows.
\end{proof}

Combining this result with recently-established results from \cite{Cai21}, 
we obtain the following corollary for the GP-UCB and elimination algorithms. 

\begin{cor} \label{cor:main}
    Under the setup of \refthm{thm:main}, 
    with probability at least $1-\delta$ (for $\delta$ used in the algorithms), 
    the attack succeeds in forcing $T(1-o(1))$ actions to be played from $\targetr$ against GP-UCB 
    (respectively, the elimination algorithm) using an attack budget $C$ of at most $B_0 N_{\max}(\Delta,T)$ (respectively, $B_0 N'_{\max}(\Delta)$), where
    \begin{align}
        N_{\max}(\Delta,T) &= \max \Big\{ N \,:\, N\leq \frac{C_1\gamma_N\beta_T}{\Delta^2} \Big\}, \label{eq:Nmax} \\
        N'_{\max}(\Delta) &= \max \Big\{ N \,:\, N\leq \frac{4C_1\gamma_N\beta_N}{\Delta^2} \Big\}, \label{eq:Nmax'}
    \end{align}
    with $C_1  = \frac{8\lambda^{-1}}{\log(1+\lambda^{-1})}$.
\end{cor}

\begin{proof}
This result follows readily by combining \refthm{thm:main} with existing bounds on the number of $\Delta$-suboptimal actions (i.e., having $f(\xv) < \max_{\xv'} f(\xv')-\Delta$) from \cite{Cai21}; 
in the standard (non-adversarial) GP bandit setting, we have the following:
\begin{enumerate}[label=(\roman*)]
    \item For \gpucb, the number of actions chosen in $T$ rounds that are $\Delta$-suboptimal is at most $N_{\max}$ with probability at least $1-\delta$. % , where $N_{\max}$ is given in \refeq{eq:Nmax}
    \item For the elimination algorithm, this can further be reduced to $N'_{\max}$. %, where $N'_{\max}$ is given in \refeq{eq:Nmax'}
\end{enumerate}\end{proof}

While the preceding results are phrased as though the adversary had perfect knowledge of $f$, 
we will highlight further special cases below where this need not be the case.

As discussed in \cite{Cai20}, as long as the smoothness parameter $\nu$ is not too small, $N_{\max}(\Delta,T))$ scales slowly with $T$ (e.g., $\sqrt{T}$ or $T^{0.1}$), and $N'_{\max}(\Delta)$ has no dependence on $T$ at all. 
Hence, for such sufficiently smooth functions with the \matern kernel, 
\refthm{thm:main} gives conditions under which an attack succeeds with a far smaller budget than the time horizon.
Note that as $\nu$ grows large for the \matern kernel, $N_{\max}$ behaves as $\frac{T^{\epsilon'}}{\Delta^{2+\epsilon}}$ for some small $\epsilon$ and $\epsilon'$, 
and $N’_{\max}$ behaves as $\frac{1}{\Delta^{2+\epsilon}}$ for some small $\epsilon$.

Moreover, in view of \refthm{thm:main}, the stronger the guarantee on the number of $\Delta$-suboptimal actions in the uncorrupted setting, the smaller the attack budget will be in a counterpart of \refcor{cor:main}.  For instance, if we were to attack the recent {\em batched} elimination algorithm of \cite{Li21}, the required budget would be sublinear in $T$ for all $\nu > \frac{1}{2}$, regardless of $d$.  However, we prefer to focus on the above non-batched algorithm, since it simpler and more standard in the literature.

\paragraph{Sufficiency vs. Necessity.}
We note that \refthm{thm:main} only states sufficient conditions for a successful attack; 
analogous necessary conditions are difficult, and appear to be absent even in simpler settings studied previously (e.g., linear bandits).  We note that 
condition (ii) is not always necessary, 
because the two functions could differ drastically at some far-away suboptimal point that is never queried.
Condition (iii) is also not always necessary, because \clipping succeeds without it.
We believe that condition (i) is “closest to necessary”, 
but even so, although the attacker’s budget would get exhausted without it, 
it could be the case that the “damage has already been done”, and the attack still succeeds.

\paragraph{The Role of RKHS Norm.}
As noted by a reviewer, the role of RKHS on the attack success is subtle. 
On the one hand, a higher RKHS norm could be associated with more local optima that the attacker can exploit.
Furthermore, it could additionally help the attacker if they can perturb $f$ using a constant fraction of the total budget $B$. 
On the other hand, if one asks which functions are the hardest to attack, then a higher $B$ implies more flexibility in coming with such a function. 
In general, we expect that there is no direct correspondence between $B$ and the attack difficulty; 
in particular, among functions with a large RKHS norm, some are easier to attack, and some are harder.

\subsection{Applications to Specific Scenarios} \label{sec:specific}

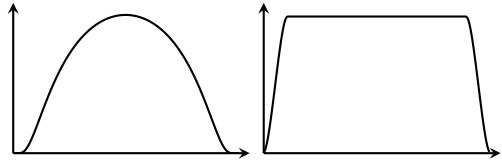
\begin{figure}[t!]
    \centering
    \begin{tikzpicture}[declare function={
        h(\x)= 0 + (\x<1) * exp(-1/(1-(\x)^2));
      },
        dot/.style = {
          solid,
          draw,
          fill = white,
          circle,
          inner sep = 0pt,
          minimum size = 4pt
        },
        nodelbl/.style = {
          scale=\illscale,
          rectangle,
          fill = white,
          inner xsep=0cm, 
          inner ysep=1pt,
          midway,
        },]
    \begin{axis}[
        samples=512, %TODO Change to higher number for final print
        domain=-1:1,
        ymin=0, ymax=0.4,
        xmax=\domainxmax,
        axis lines=left,
        y=2cm/0.4,
        x=3cm/\domainxdomain,
        grid=none,
        ticks=none,
        compat=newest,
        thick,
        axis line style = thick,
        axis on top=true,
    ]
        \addplot [black,smooth,name path=F] {h(x)}; %  node[above,pos=1] {$h$}
    \end{axis}
    \path
    ([shift={(-5\pgflinewidth,-5\pgflinewidth)}]current bounding box.south west);
    \end{tikzpicture}
    \begin{tikzpicture}[declare function={
        h(\x)= 0 + (\x<1) * exp(-1/(1-(\x)^2));
      },
        dot/.style = {
          solid,
          draw,
          fill = white,
          circle,
          inner sep = 0pt,
          minimum size = 4pt
        },
        nodelbl/.style = {
          scale=\illscale,
          rectangle,
          fill = white,
          inner xsep=0cm, 
          inner ysep=1pt,
          midway,
        },]
    \begin{axis}[
        samples=512, %TODO Change to higher number for final print
        domain=-1:1,
        ymin=0, ymax=1.1,
        xmax=\domainxmax,
        axis lines=left,
        y=2cm/1.1,
        x=3cm/\domainxdomain,
        grid=none,
        ticks=none,
        compat=newest,
        thick,
        axis line style = thick,
        axis on top=true,
    ]
        \addplot [black,smooth,name path=CBUMP] table [x=x, y=y, col sep=comma] {CBUMP.csv};
    \end{axis}
    \path
    ([shift={(-5\pgflinewidth,-5\pgflinewidth)}]current bounding box.south west);
    \end{tikzpicture}
    \caption{Illustration of the various function constructions -- bump function (left) and convolved bump function (right).\label{fig:appendix-bump}}
\end{figure}

Before we discuss two natural approaches to ensuring the conditions in \refthm{thm:main},
we describe some useful function constructions from the existing literature (see \refig{fig:appendix-bump}):
\begin{itemize}
    \item \emph{Bump function in spatial domain}: It is known from \cite{Bul11,Cai20} that the function
    \begin{equation}
        h(\xv) = \exp\bigg(\frac{-1}{1-\|\xv\|^2}\bigg) \boldsymbol{1}\braces{ \|\xv\| \le 1 }, \label{eq:h_bump}
    \end{equation}
    has bounded RKHS norm under the \matern kernel, and is non-negative with bounded support and maximum at $\xv = \bzero$. 
    Moreover, we can form a scaled version of $h$ with height $\epsilon$ and support of radius $w$, 
    and the resulting RKHS norm scales as $O\big(\frac{\epsilon}{w^\nu}\big)$ 
    (\ie, it can be made at most $B$ with $w = \Theta\big(\big( \frac{\epsilon}{B} \big)^{1/\nu}\big)$).
    \item \emph{Convolved bump function}: Following \cite{Cai20}, we consider taking the width-$w$ height-$1$ bump function from the previous dot point, and convolving it with a function that equals $1$ for $\xv \in S$, and $0$ for $\xv \notin S$, 
    where $S$ is some compact subset of $\RR^d$. Note that in \citep[Lemma 6]{Cai20}, $S$ was ball-shaped, but the analysis extends to the general case. Then, after suitable scaling to make the new function's maximum height equal to some value $\tau$, we obtain a function $g(\xv)$ satisfying the following:
    \begin{itemize}
        \item If the ball of radius $w$ around $\xv$ is contained in $S$, then $g(\xv) = \tau$;
        \item If the ball of radius $w$ around $\xv$ is completely outside $S$, then $g(\xv) = 0$;
        \item In all cases in between these, $g(\xv) \in (0,\tau)$.
    \end{itemize}
    Moreover, the RKHS norm satisfies $\|g\|_k \le O\big( \frac{\tau \cdot {\rm vol}(S)}{w^{\nu}} \big)$ \cite{Cai20}, so if ${\rm vol}(S) = O(1)$ we can ensure an RKHS norm of at most $B$ with a choice of the form $w = \Theta\big( \big( \frac{\tau}{B} \big)^{1/\nu} \big)$.
\end{itemize}

We now discuss two general attack approaches, which we will build on later in \refsec{sec:attacks}.

\paragraph{Approach 1.} To guarantee that $\|\tilde{f}\|_k \le B$, it is useful to adopt the decomposition
\begin{equation}
    \|\tilde{f}\|_k \le \|f\|_k + \| \tilde{f} - f \|_k,
\end{equation}
which follows from the triangle inequality.  Hence, if the ``RKHS norm budget'' $B$ is not entirely used up by $f$ (\eg, $\|f\|_k = \frac{B}{2}$), then we can maintain $\|\tilde{f}\|_k \le B$ by setting
\begin{equation}
    \tilde{f}(\xv) = f(\xv) - h(\xv)
\end{equation}
for some $h(\xv)$ with suitably bounded norm (\eg, $\|h\|_k \le \frac{B}{2}$). 
The convolved bump function described above (which approximates a rectangular function) is useful for constructing $h$.

This approach is immediately applicable to the case that the interior of $\targetr$ contains a local maximum, \eg, see \refig{fig:att_sub}.  In this case, we can simply form bumps to ``swallow'' any higher maxima (or maxima within $\Delta$ of the target peak).  As long as the convolved bump functions used for this purpose have a small enough height and transition width to maintain that $\|h\|_k \le B - \|f\|_k$, the resulting function will satisfy the conditions of \refthm{thm:main}.

However, finding the precise locations of such bumps and setting their parameters may be difficult, or even entirely impossible if $f$ is unknown, even if there is plenty of RKHS norm budget that can be utilized.  

\paragraph{Approach 2.} In view of the above discussion, we propose a more aggressive approach that seeks to push the function values downward equally (or approximately equally) for {\em all} points outside $\targetr$.  This can be done by taking a convolved bump function that covers the entire domain with some height $h_{\max}$ (\ie, $h(\xv) = -h_{\max}$), but then ``re-adding'' a bump that covers $\targetr$. 
Depending on which is more convenient, the transition region could lie either inside or outside $\targetr$.
See \refig{fig:att_aggr} for an illustration.

Once again, the preceding attack can be applied while maintaining the conditions \refthm{thm:main} in cases where $\targetr$ contains a local maximum, provided that the bumps used in its construction do not exceed the RKHS norm budget. 

\paragraph{Example Difficult Case.} \label{sec:difficult}
\begin{figure}[t!]
    \centering
    {
        \pgfmathsetmacro\targets{-1.}
        \pgfmathsetmacro\targete{-0.87}
        \pgfmathsetmacro\domains{-1}
        \pgfmathsetmacro\domaine{-0.3}
        \pgfmathsetmacro\domainxmax{-0.265}
        \pgfmathsetmacro\domainxdomain{0.7}
        \pgfmathsetmacro\flblpos{1}
        
        \illustration{
        }{
          \draw[] (-0.85636275, {f(-0.85636275)}) node[dot]{};
        }{-40}{100}
    }
    \caption{Difficult case where $f$ is increasing in $\targetr$.} \label{fig:difficult}
\end{figure}
To highlight the consideration of cases with local maxima in $\Rc_{\rm target}$ above, consider the 1D (counter-)example in \refig{fig:difficult}. 
Here, there is no apparent way to construct $h(\xv)$ to satisfy the conditions of \refthm{thm:main};
the function continues to increase when leaving $\targetr$, 
and performing any {\em smooth} perturbation to the function to push those values downward would either keep the maximizer outside $\targetr$, 
or amount to having $\Delta$-optimal points for $\tilde{f}$ such that $\tilde{f}(\xv) \ne f(\xv)$. 
Despite this difficulty, we will see that our attacks can still be effective in such situations experimentally, albeit requiring a larger attack budget.

\section{Attack Methods} \label{sec:attacks}

\subsection{Overview}% attacks that have been developed for finite-arm bandit problems \cite{Jun18}, 
In this section, we introduce our main proposed adversarial attacks on GP bandits. 
Based on \refthm{thm:main}, we adopt the idea of perturbing the function $f$ to construct another function $\tilde{f}$ such that ideally the following properties hold for some set $\Rc_0 \subseteq \targetr$ 
(often taking the full set, $\Rc_0 = \targetr$):
\begin{enumerate}[label=(\roman*)]
    \item $\tilde{f}(\xv) = f(\xv)$ for all $\xv \in \Rc_0$;
    \item All maximizers of $\tilde{f}$ lie inside $\Rc_0$, and ideally all points outside $\targetr$ are suboptimal by at some strictly positive value $\Delta > 0$;
    \item $\tilde{f}$ satisfies the RKHS norm constraint, \ie, $\|\tilde{f}\|_k \le B$.
\end{enumerate}
We use the \synthetic objective function in \refeq{eq:1d} to illustrate the attack methods in \refigs{fig:att_sub}{fig:att_aggr}.

\subsection{Subtraction Attack (Known $f$)} \label{sec:att_subtract}

For this attack, we follow the ideas discussed in Approach 1 of \refsec{sec:specific}.  To maintain property (iii), 
\ie, $\tilde{f} \in \Fc_k(B)$, we assume that $f \in \Fc_k(B/2)$,\footnote{The factor $\frac{1}{2}$ can be replaced by other constants in $(0,1)$.} 
and seek to find a perturbation function $h \in \Fc_k(B/2)$ such that
\begin{equation}
    \tilde{f}(\xv) = f(\xv) - h(\xv) \label{eq:subtraction}
\end{equation}
satisfies properties (i) and (ii).  The triangle inequality then gives that $\|\tilde{f}\|_k \le B$, as desired for property (iii).

We let $h(\xv)$ be a sum of support-bounded, with the support lying entirely outside $\Rc_0$ designed to ``swallow'' the peaks outside $\Rc_0$.  For \subtractionrnd we construct these functions using the bump function $h_{\rm bump}(\xv) = \exp\big(\frac{-1}{1-\|\xv\|^2}\big) \boldsymbol{1}\braces{ \|\xv\| \le 1 }$ (see \refsec{sec:specific} for details), whereas for \subtractionsq we use the simpler indicator function $h_{\rm ind}(\xv) = \boldsymbol{1}\{ \|\xv\| \le 1 \}$.  Examples are shown in \refig{fig:att_sub}.

Based on \refsec{sec:theory}, \subtractionrnd has strong theoretical guarantees under fairly mild assumptions, but a disadvantage is requiring knowledge of $f$.  Moreover, there may be many ways to construct $h$ satisfying the requirements given, and finding a good choice may be difficult, particularly in higher dimensions where the function cannot be visualized.

\begin{figure}
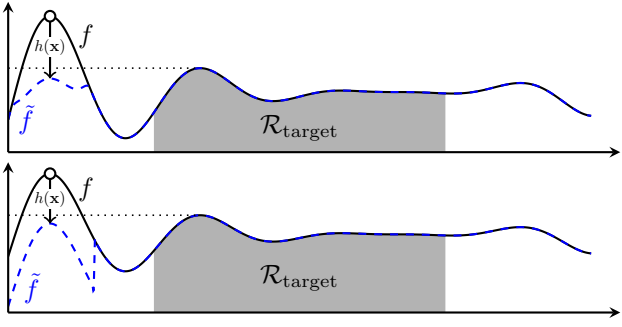

    \centering
    \illustration{
        \addplot [blue,dashed,smooth,name path=SUB_RND] {f_rnd(x)} node[right,pos=0] {$\tilde{f}$};
      }{
        \betweenline{-0.85636275}{$h(\xv)$}{f_rnd};
      }{-40}{100}
      
      \illustration{
        \addplot [blue,dashed,smooth,name path=SUB_SQ] {f_sq(x)} node[right,pos=0.05] {$\tilde{f}$};
      }{
        \betweenline{-0.85636275}{$h(\xv)$}{f_sq};
      }{-65}{125}
    \caption{Subtraction Attack for known $f$ with different $h(\xv)$ -- \subtractionrnd (top) and \shortsq (bottom).\label{fig:att_sub}} 
\end{figure}

\subsection{Clipping Attack (Known $f$)} \label{sec:att_clip}

The subtraction attack in \refsec{sec:att_subtract} was developed for the primary purpose of obtaining theoretical guarantees, but in practice, 
finding $h(\cdot)$ with the desired properties may be challenging.
Here we provide a more practical attack without theoretical guarantees, but with a similar general idea.
Specifically, we propose to directly attain properties (i) and (ii) above by setting
\begin{equation}
    \tilde{f}(\xv) = 
    \begin{cases}       
        f(\xv) & \xv \in \targetr \\
        \min\braces{f(\xv), f(\xvtilde^*) - \Delta} & \xv \notin \targetr,
    \end{cases}
\end{equation}
where $\xvtilde^* = \argmax_{\xv \in \targetr} f(\xv)$ (illustrated in \refig{fig:att_clip}).

Here, unlike for the subtraction attack, property (iii) does not hold, 
which is why the theoretical analysis does not follow. 
Despite this disadvantage (and the requirement of knowing $f$), 
this attack has the clear advantage of being simple and easy to implement, 
and we will see in \refsec{sec:experiments} that it can be highly effective in experiments.

\begin{figure}
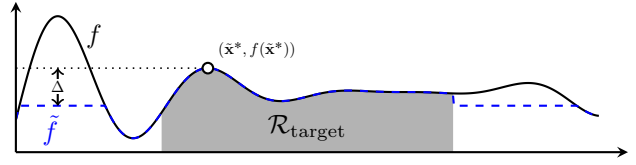

    \centering
    \illustration{
        \addplot [blue,dashed,name path=SUB_RND] {clipping(x,25)} node[below,pos=0.076] {$\tilde{f}$};
      }{
        \draw[] (-0.3413, {f(-0.3413)}) node[dot,label={north east,scale=\illscale:$\parens{\tilde{\xv}^*,f\parens{\tilde{\xv}^*}}$}]{};
        \deltaline{-0.85636275}{$\Delta$}{25};
      }{-40}{100}
    \caption{\clipping attack for known $f$. \label{fig:att_clip}}
\end{figure}

\subsection{Aggressive Subtraction Attack (Unknown $f$)} \label{sec:att_aggr}

In the case that $f$ is unknown, we propose to build on the idea of the subtraction attack, but to be highly aggressive and subtract {\em all} points outside $\targetr$ by roughly the same value $h_{\max}$.  This is a special case of \refeq{eq:subtraction}, but here we consider $h(\cdot)$ having a much wider support than described above, so we find it best to highlight separately.

To ensure that $\|h\|_k \le \frac{B}{2}$, we need to consider a ``transition region'' where $h(\cdot)$  transitions from zero to $h_{\max}$.  Overall, we are left with $h(\xv)$ equaling zero within $\Rc_0$, $h_{\max}$ outside $\targetr$, and intermediate values within $\targetr \setminus \Rc_0$. 

Again based on \refsec{sec:theory}, this attack has strong theoretical guarantees.  A notable advantage is not requiring precise knowledge of $f$; it suffices that $\Rc_{\rm target}$ has a suitable local maximum and $h_{\max}$ is large enough, but no specific details of the function need to be known.  In addition, this attack is fairly straightforward to implement, only requiring the selection of $h_{\max}$ and a transition region width.  

A disadvantage is that a higher budget $C$ may be needed in practice due to being highly aggressive.  On the other hand, similar aggressive approaches have shown success in related bandit problems \cite{Jun18,Gar20}.

\begin{figure}
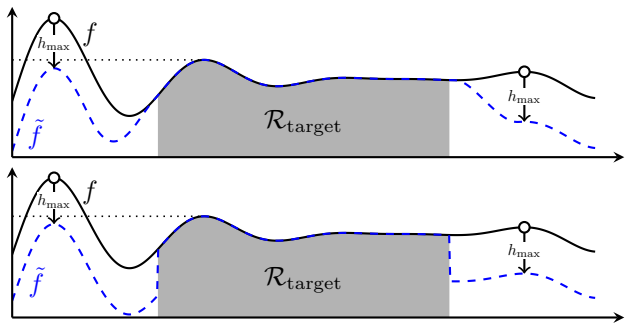

    \centering
    \illustration{
        \addplot [blue,smooth,name path=AGGT,dashed] table [x=x, y=y, col sep=comma] {AGGT.csv} node[right,pos=0.05] {$\tilde{f}$};
      }{
        \betweenline{-0.85636275}{$h_{\max}$}{agg};
        \betweenline{0.75718}{$h_{\max}$}{agg};
      }{-65}{125}
    \illustration{
        \addplot [blue,name path=AGG, dashed] {(\x<\targets)*agg(x) + and(\x>\targets, \x<\targete)*f(x) + (\x>\targete)*agg(x)} node[right,pos=0.05] {$\tilde{f}$};
      }{
        \betweenline{-0.85636275}{$h_{\max}$}{agg};
        \betweenline{0.75718}{$h_{\max}$}{agg};
      }{-75}{135}
    \caption{\aggressivesubtraction attack for unknown $f$ -- with ``transition region'' (top) and without (bottom).\label{fig:att_aggr}}
\end{figure}

\paragraph{Simplified Aggressive Subtraction Attack.} \label{sec:att_const}

In the same way that \subtractionsq simplifies \subtractionrnd, we can simplify the aggressive subtraction attack as follows, for some $h_{\max} > 0$:
\begin{equation}
    \tilde{f}(\xv) = 
    \begin{cases}       
        f(\xv) & \xv \in \targetr \\
        f(\xv) - h_{\max} & \xv \notin \targetr.
    \end{cases}
\end{equation}
This can again be implemented without knowledge of $f$.  

Similar to above, we would like to choose $h_{\max}$ large enough so that the maximizer of $\tilde{f}$ is in $\targetr$. 
This again has the disadvantage of potentially requiring a large corruption budget, 
but compared to the clipping attack, we gain the advantage of not needing to know $f$.  

Both versions of \aggressivesubtraction are illustrated in \refig{fig:att_aggr}.  In the rest of the paper, whenever we mention this attack, we are referring to the simplified variant.

\section{Experimental Results\label{sec:experiments}} 
For convenience, our experiments\footnote{The code is available at \url{https://github.com/eric-vader/Attack-BO}.} consider targeted attacks and do not constrain the attack to any particular budget;
we instead explore the trade-off between attack success rate (\ie, fraction of actions played that fall in the $\targetr$) and attack cost (e.g., sum of perturbation levels used).

To the best of our knowledge, our work is the first for this setting, so there are no existing attacks to compare against.
Hence, our focus is only comparing our proposed attacks to each other (\ie \subtractionrnd, \subtractionsq, \clipping, \aggressivesubtraction) along with two trivial baselines.
The baselines are \random, which perturbs the function evaluation with $\N\parens{\mu_a,{\sigma_a}^2}$ and,
\noattack which does not perturb the function evaluation at all.

We exclude \subtractionrnd, \subtractionsq from experiments on higher dimensional functions (\ie $\geq 3$),
due to the difficulty discussed in \refsec{sec:att_subtract}.

We compare the attacks on a variety of objective functions of up to $6$ dimensions.
For the experiments on synthetic functions, we manually add $z_t \sim \N\parens{0, 0.01^2}$ noise.
Detailed experimental details are given in \refsec{sec:addexp-details} (appendix) and our complete code is included in the supplementary.

\subsection{Setup}
In order to understand the behavior of each attack, we run each attack $300$ times, varying $30$ different hyperparameter choices (\ie $\Delta, h_{\max}, (\mu_a,{\sigma_a}), h(\xv)$) with $10$ differing conditions (initial points, instances of the objective function, and random seeds).
For both the \subtractionrnd and \subtractionsq attacks, we vary only the $y$-transformation hyperparameter, which we refer to as $h_{\max}$.  The different hyperparameter settings are chosen such that,  as much as possible, 
we aim to cover the entire spectrum of the behavior (starting with a lower success rate and ending with a high success rate) of the attack method for each experiment.
We run $100$ or $250$ optimization steps depending on the dimensionality, and adopt the \maternx{5}{2} kernel to match our theory; 
see \refsec{sec:kernel} and \refsec{sec:diffkernels} for further details.

We focus on attacking the GP-UCB algorithm and defer the elimination algorithm to \refsec{sec:addexp-details} (appendix), since GP-UCB is much more widespread
and \refcor{cor:main} indicates that it should be the harder algorithm to attack.  (One reason for this is that
\maxvar explicitly removes points from further consideration, so it can be forced to  permanently discard the true optimum.)
Following typical choices adopted in existing works' experiments (e.g., \cite{Sri09,rolland2018high}), we select the exploration parameter as $\beta_t = 0.5\log{\parens{2t}}$.
There are very few defense algorithms in the literature that are appropriate here; 
in particular, the main algorithm in \cite{Bog20} (Fast-Slow GP-UCB) was introduced for theoretical purposes and seemingly not intended to be practical.
However, \cite{Bog20} also shows that GP-UCB with {\em larger choices of $\beta_t$ depending on $C$} can be provably more robust, 
and we provide experimental support for this in \refsec{sec:defense_exp}.

\subsection{Metrics}
We measure the following up to iteration $t$, with $X_t = (\xv_1,\dotsc,\xv_t)$ and adversarial noise $A_t = (|c_1|,\dotsc,|c_t|)$:
\begin{itemize}
\item $\successrate\parens{t} = \frac{|\targetr \intersection X_t|}{t}$ is the proportion of actions played that lie within $\Rc_{\rm target}$.
\item $\normalizedcost\parens{t} = \sum_{a \in A_t}{\frac{a}{\fmax - \fmin}}$ is the sum over the history of adversarial noise, normalized by the function range for comparison across experiments.
\end{itemize}

\subsection{Synthetic Experiments}

\begin{figure*}[!t]
    \centering
    \subfloat[\aggressivesubtraction on \synthetic]{
        \includegraphics[width=0.32\textwidth]{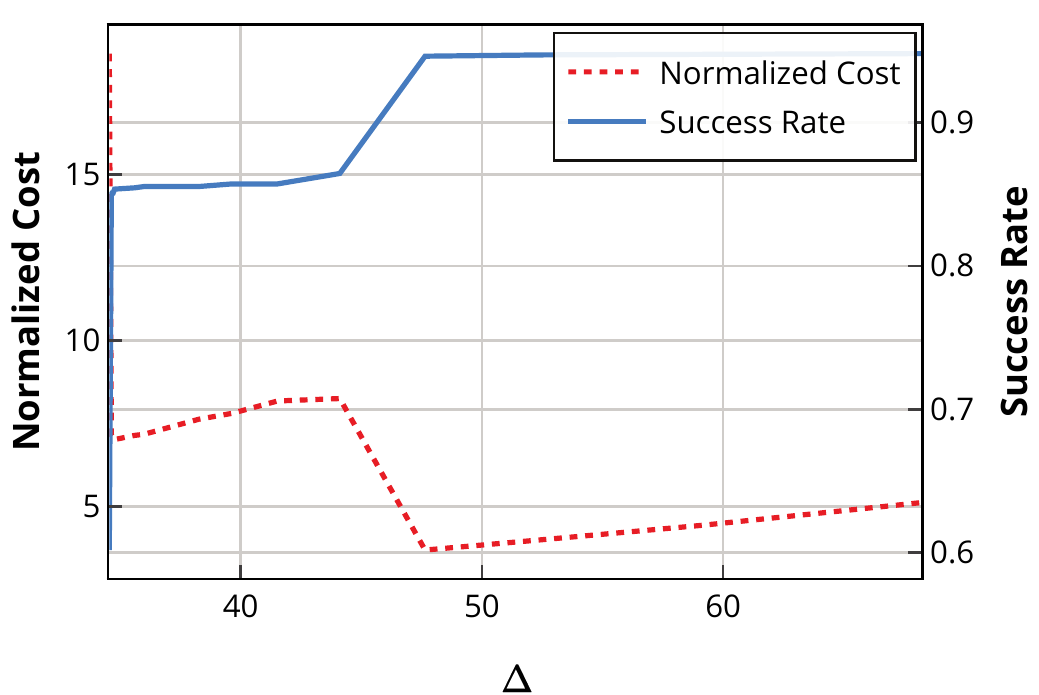} \label{fig:1d-aggressive-analysis}}
    \subfloat[\clipping on \forrester]{
        \includegraphics[width=0.32\textwidth]{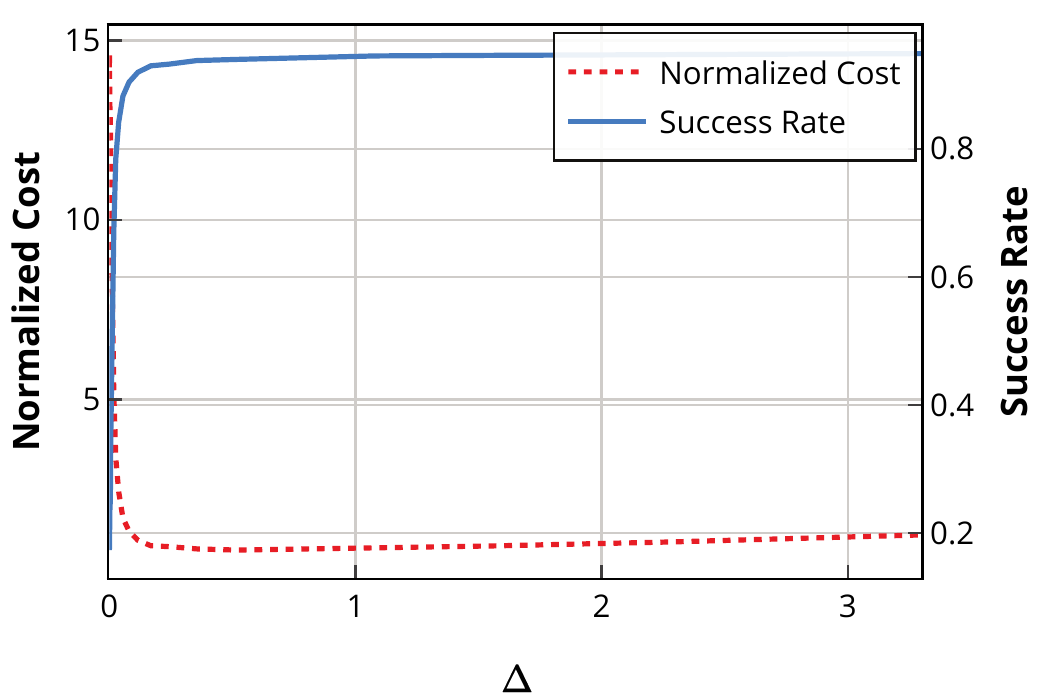} \label{fig:forrester-clipping-analysis}}
    \subfloat[\subtractionrnd on \levy]{
        \includegraphics[width=0.32\textwidth]{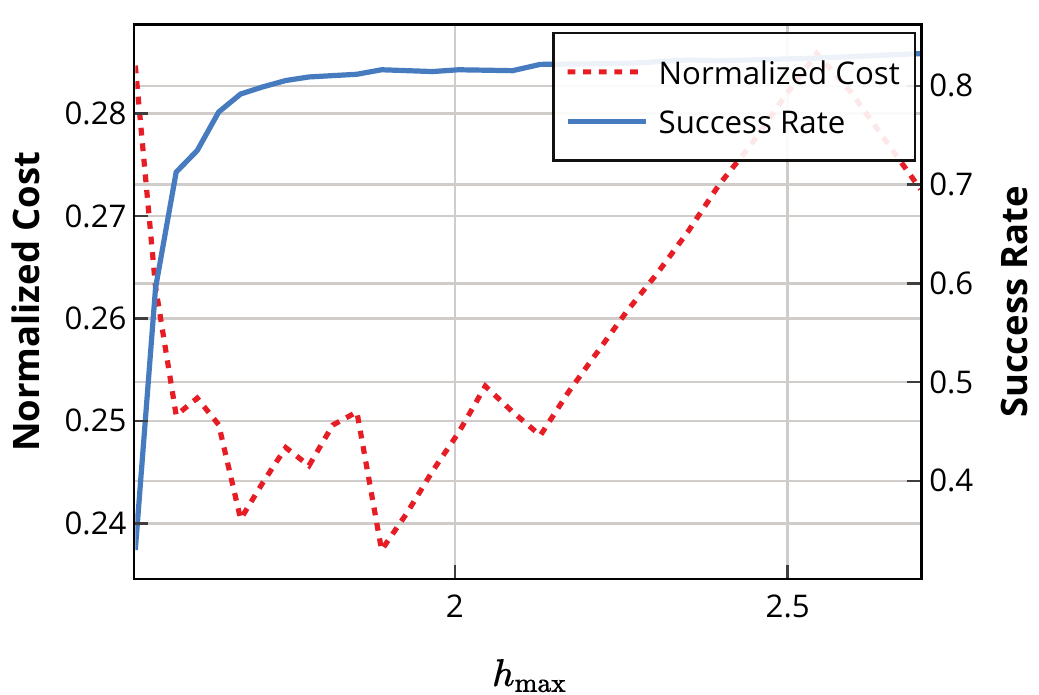} \label{fig:levy-subtractionrnd-analysis}}\\
    \caption{Effects of hyperparameters on the different attacks.\label{fig:hyperparameter}}
\end{figure*}

\begin{figure*}[!t]
    \centering
    \includegraphics[height=0.25cm]{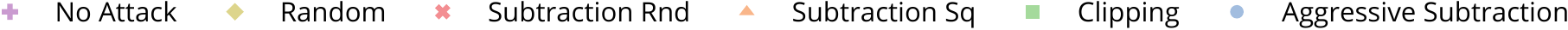} \label{fig:Legend}\\
    \subfloat[\synthetic]{
        \includegraphics[width=0.32\textwidth]{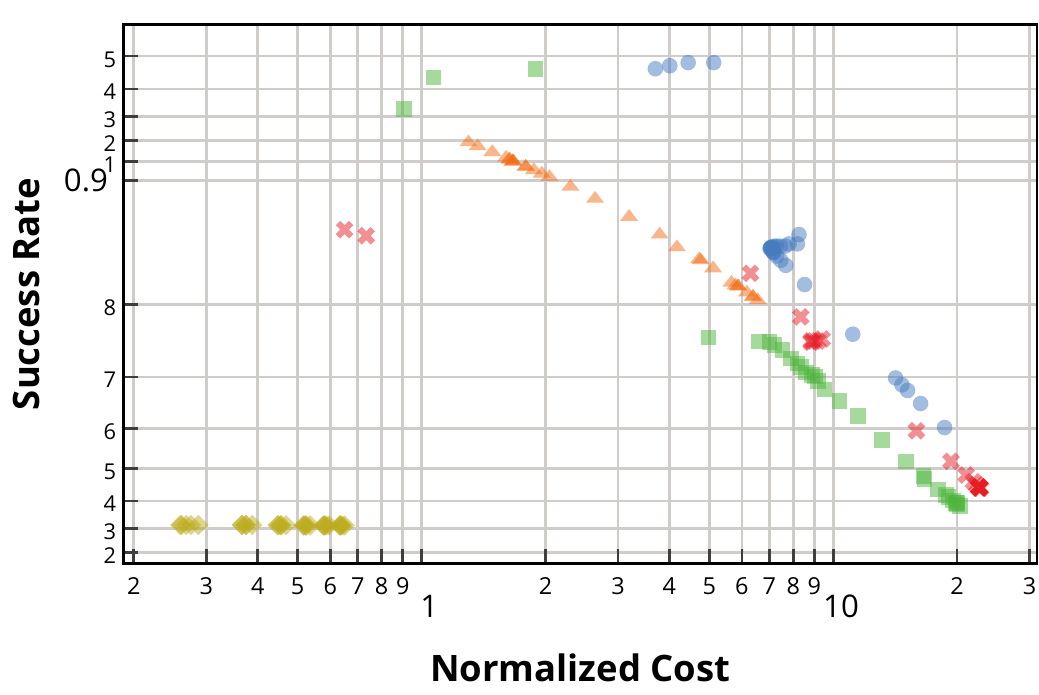} \label{fig:1d-overall}}
    \subfloat[\forrester]{
        \includegraphics[width=0.32\textwidth]{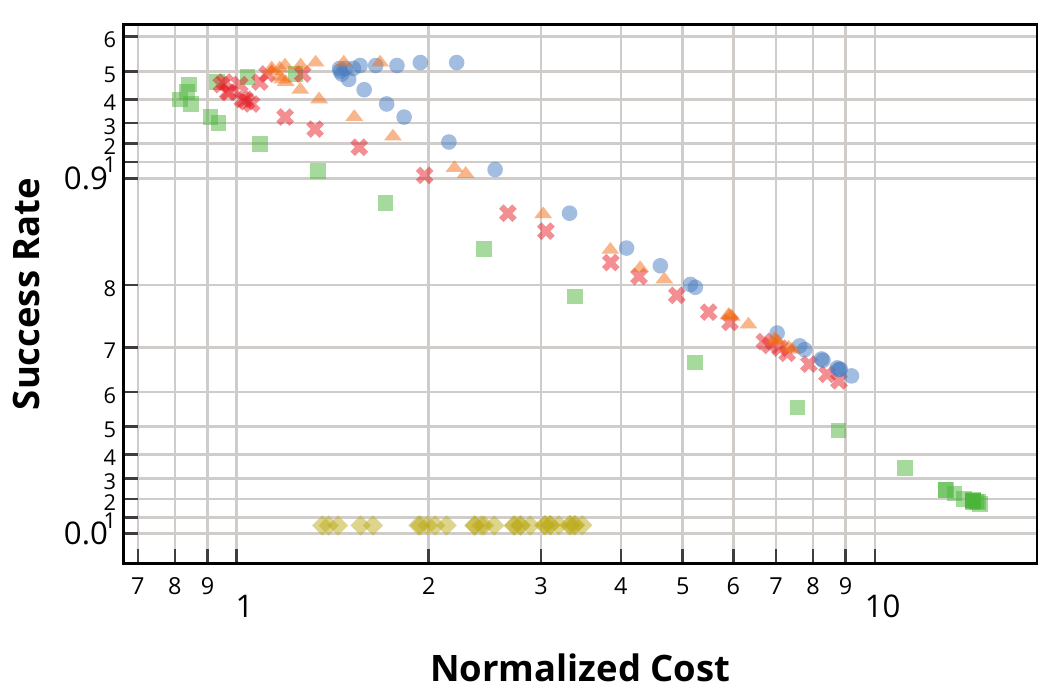}\label{fig:forrester-overall}}
    \subfloat[\levyhard]{
        \includegraphics[width=0.32\textwidth]{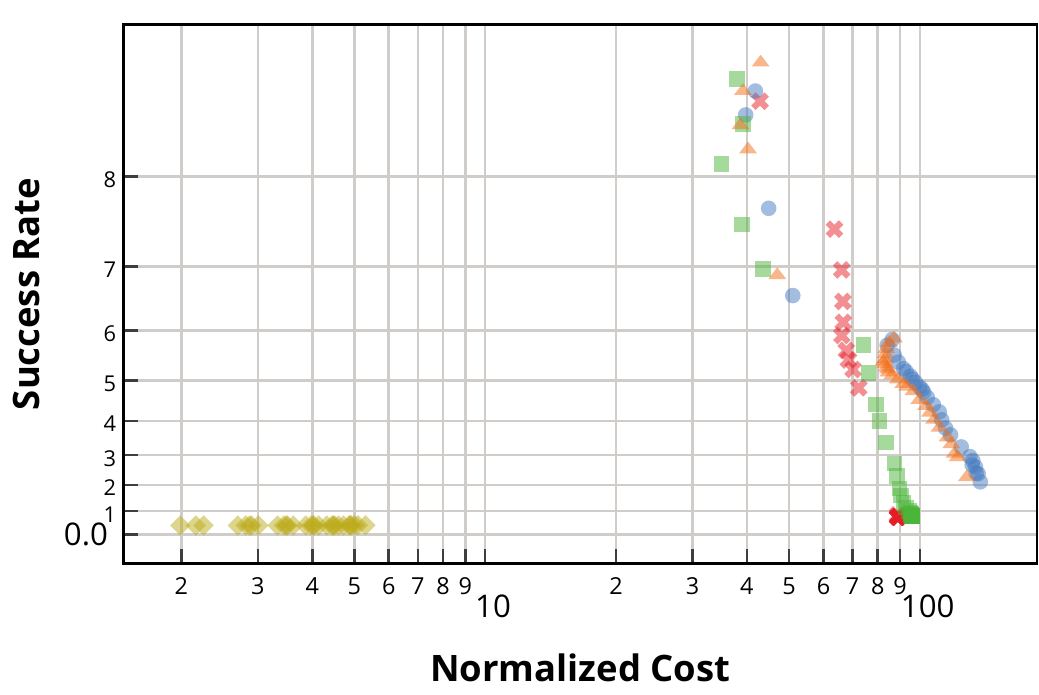} \label{fig:levyhard-overall}}\\
    \subfloat[\camelback]{
        \includegraphics[width=0.32\textwidth]{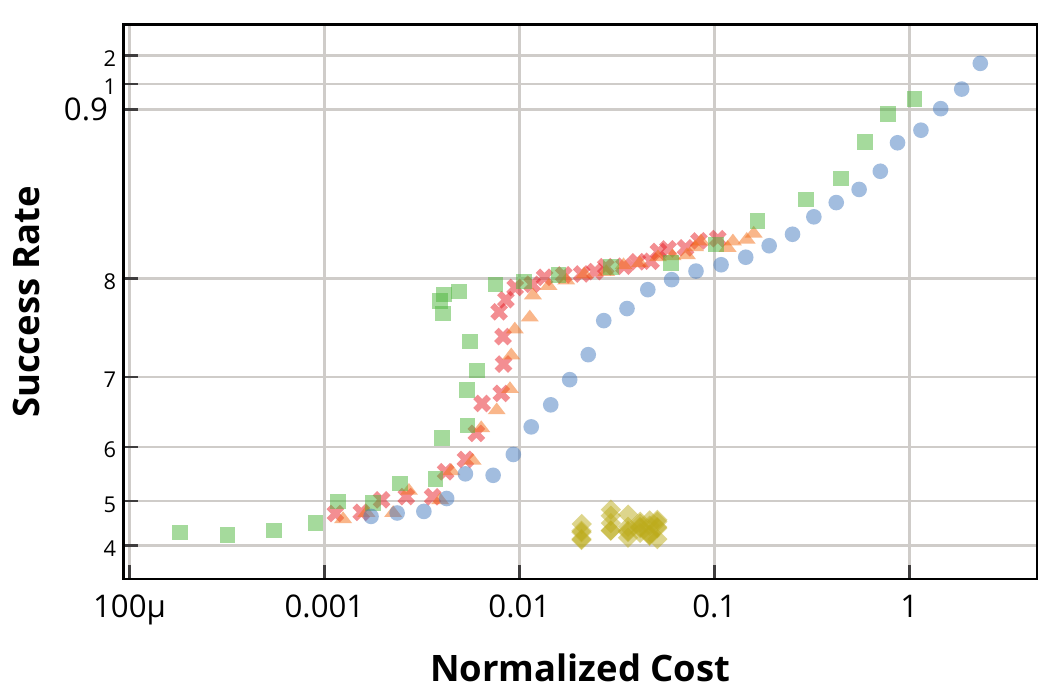} \label{fig:camelback-overall}}
    \subfloat[\bohachevsky]{
        \includegraphics[width=0.32\textwidth]{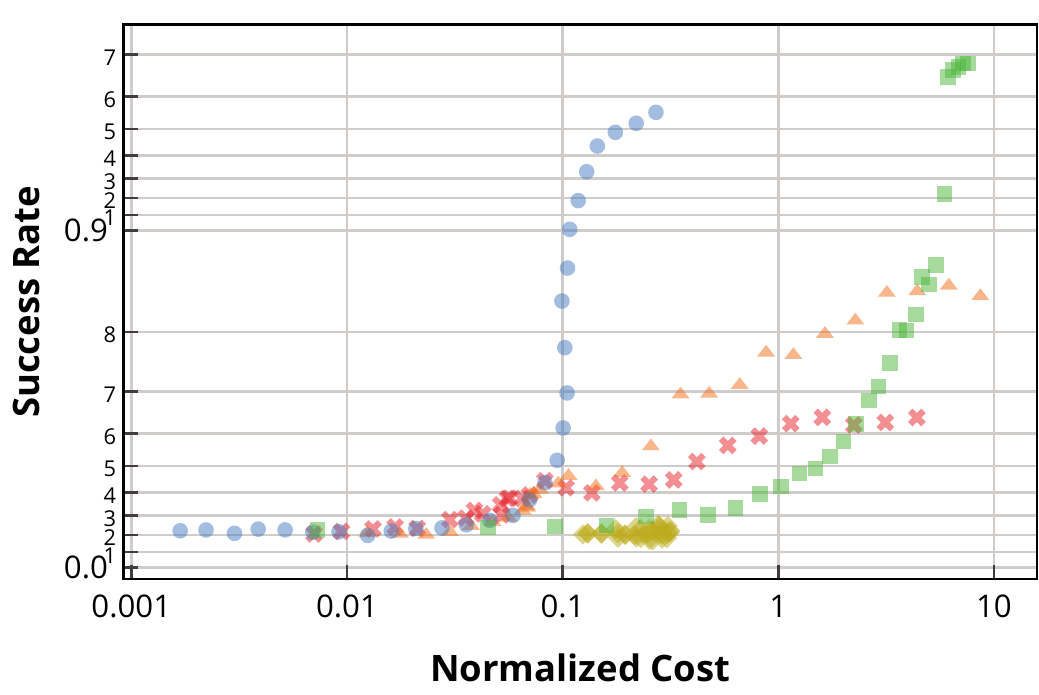} \label{fig:bohachevsky-overall}}
    \subfloat[\robotsmall]{
        \includegraphics[width=0.32\textwidth]{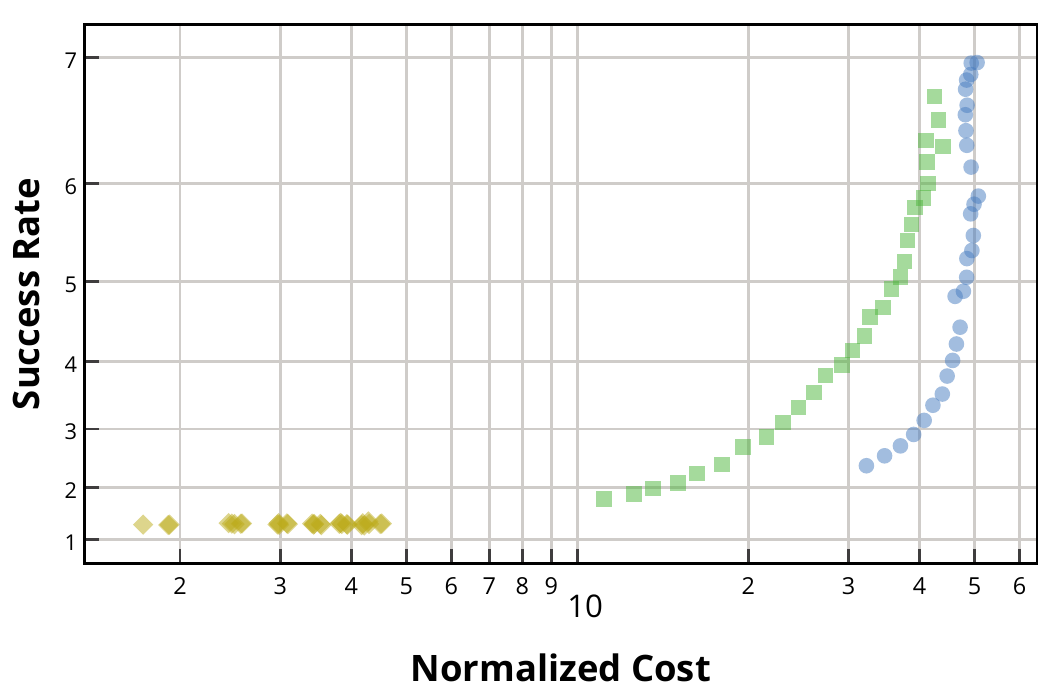} \label{fig:robot3d-overall}}
    \caption{Scatter plots of runs averaged over random seeds for various functions.\label{fig:overall}}
\end{figure*}

\begin{figure*}[!t]
    \centering
    \subfloat[\bohachevsky]{
        \includegraphics[width=0.32\textwidth]{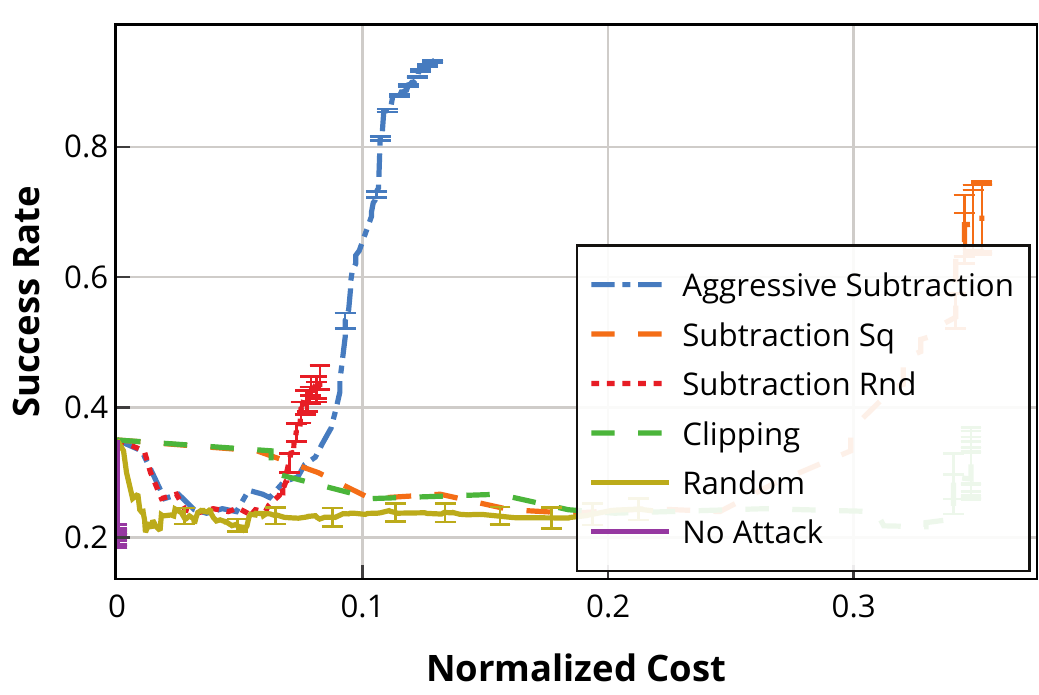} \label{fig:bohachevsky-bestparam}}
    \subfloat[\branin]{
        \includegraphics[width=0.32\textwidth]{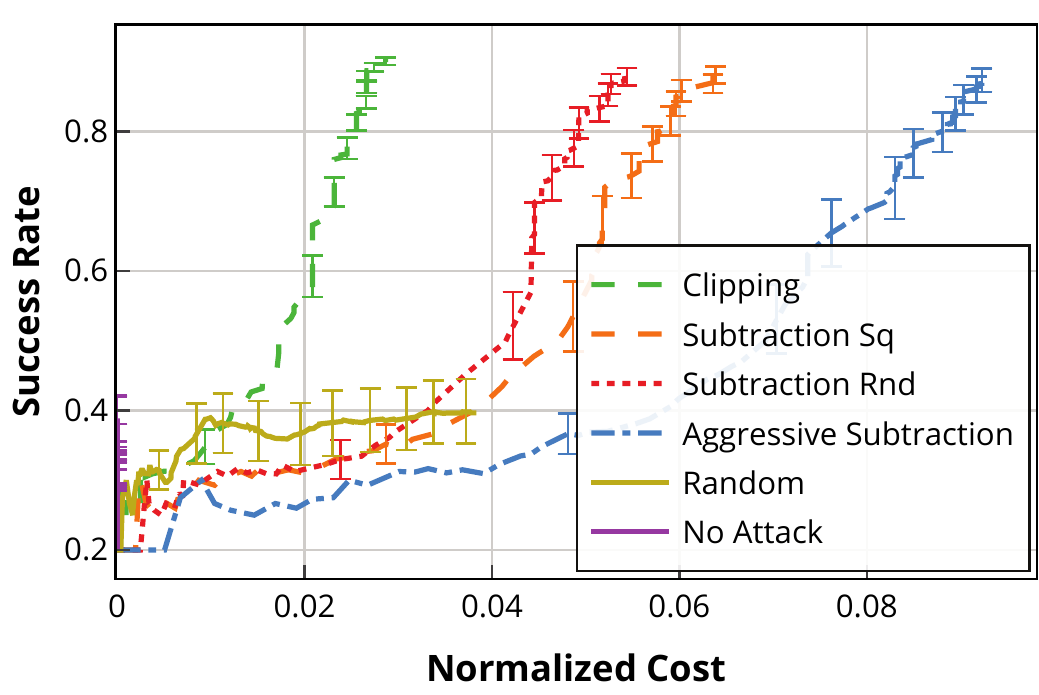} \label{fig:branin-bestparam}}
    \subfloat[\robotsmall]{
        \includegraphics[width=0.32\textwidth]{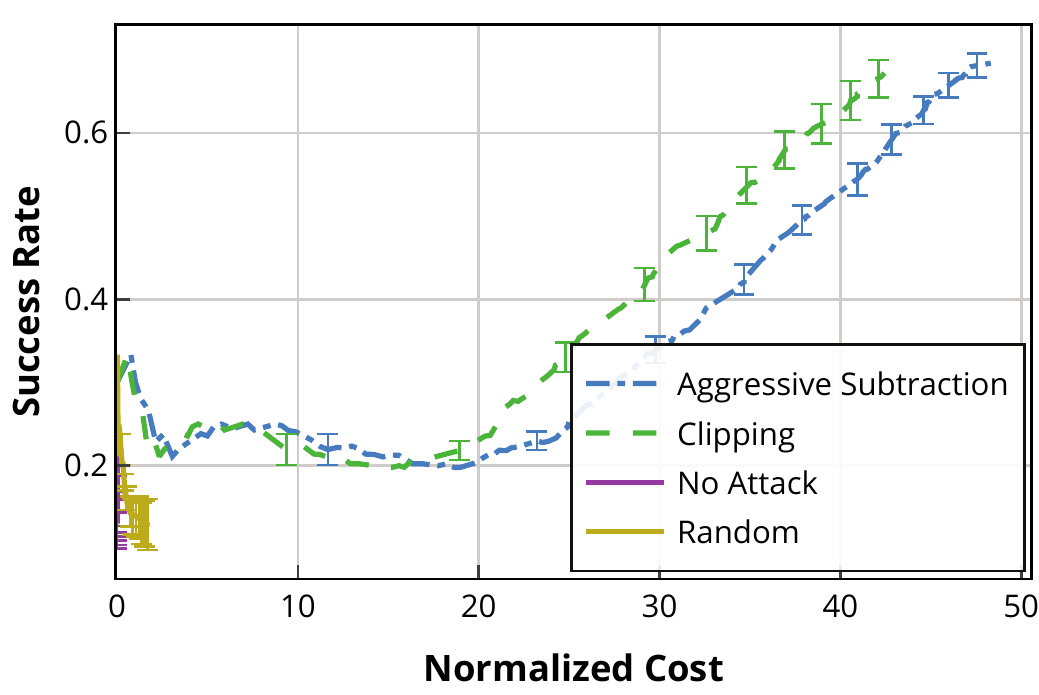} \label{fig:robot_pushing3d-bestparam}}\\
    \caption{Plots of each attack with the most efficient (best success rate over normalized cost) hyperparameter.\label{fig:bestparam}}
\end{figure*}

The synthetic functions are described and plotted in \refsec{sec:synth} (appendix). 
The performance of each attack method depends on the associated hyperparameter (e.g., $h_{\max}$ or $\Delta$).
Ideally, we want the best hyperparameter for the attack where it spends the least cost for the highest success rate.  One might expect that a higher cost is always associated with a higher success rate, 
but this is not quite the case, as we observe in \refig{fig:hyperparameter}.  The subtlety is that a poorly-chosen hyperparameter can be poor in both regards. 
Typically, what happens is that when the hyperparameter is at its `boundary value' (e.g., $\Delta=0$ in \refig{fig:att_clip}) the attack is not very successful, 
because GP-UCB explores both the clipped regions and $\parens{\tilde{\xv}^*,f\parens{\tilde{\xv}^*}}$. 
However, going carefully beyond the boundary value can quickly lead to a near-maximal success rate with a fairly modest cost. 
Perhaps less intuitively, even attacking more aggressively (e.g., increasing $\Delta$) sometimes maintains a similar cost-success trade-off, 
suggesting that it may be preferable to be `over-aggressive' than `under-aggressive'.

From \refig{fig:overall}, we observe that \clipping performs consistently well across the various functions, providing competitive success rates and costs.
Both \subtractionrnd and \subtractionsq tend to have a performance `in between' that of \clipping and \aggressivesubtraction.
Due to \subtractionrnd's smooth $h(\xv)$, \subtractionrnd tends to narrowly beat \subtractionsq in most experiments, e.g., see \refig{fig:forrester-overall}.
The trend is also observed in \levy, the easier version of \levyhard, as seen in \refig{fig:appendix-levy_Scatter_overall_avg_log_log_rev-Scatter-success_rate-norm_cum_c_t} (appendix).
However, these aforementioned methods require prior knowledge of $f$, whereas \aggressivesubtraction works well in practice without knowledge of $f$.

We additionally note in \refig{fig:overall} that across all of the functions considered, 
there typically exists a successful attack with a relatively low normalized cost, e.g., on the order of magnitude $1$ or less, 
though more difficult cases also exist (e.g., Levy-Hard1D).  This highlights the general lack of robustness of standard solutions to the GP bandit problem.

In \refig{fig:bestparam}, we plot the performance of each attack method with its most efficient hyperparameter.
The hyperparameter with the largest success rate over cost is the most efficient,
\begin{equation}
    \theta_\text{efficient} = \argmax_{\theta}{\frac{\successrate\parens{T}_\theta}{\normalizedcost\parens{T}_\theta}}\label{eq:efficient}.
\end{equation}
Considering the most efficient hyperparameter, 
\clipping continues to be successful and cost effective compared to others for most functions.
\subtractionrnd and \subtractionsq tend to perform in between \clipping and \aggressivesubtraction, e.g., see \refig{fig:branin-bestparam}. 

There are also some interesting cases where this trend is not observed;
\subtractionrnd and \subtractionsq struggle to achieve a high success rate in \camelback, 
while \aggressivesubtraction and \clipping achieve a higher success rate after spending more cost.
\camelback has a symmetric surface with multiple global and local optima in different regions,
which increases the difficulty in constructing $h(\xv)$ for the subtraction attacks.
Similarly, for the attacks with the most efficient hyperparameter in \bohachevsky, 
\aggressivesubtraction is the most efficient, spending the least but having the highest success rate.
Similar findings are observed in \refig{fig:bohachevsky-overall}.
\bohachevsky has a bowl shape with a single maximum, which is easy to optimize but hard to attack.

As a result, \bohachevsky increases throughout $\targetr$ and when leaving $\targetr$,
causing the attack to be especially difficult as explained in \refsec{sec:difficult} with \refig{fig:difficult}.
Hence, there is no obvious way to construct $h(\xv)$ for both \subtractionrnd and \subtractionsq (also for \camelback).
We believe that \clipping is less effective because it creates $\tilde{f}$ which is largely flat (equal to the clipped value), the algorithm may spend too long exploring this flat region.  In contrast, \aggressivesubtraction maintains the bowl shape (shifted downwards) outside $\targetr$, which the algorithm can more readily stop exploring due to believing it to be suboptimal.  Hence, \aggressivesubtraction performs best here in terms of success rate and cost.

\subsection{Robot Pushing Experiments}
We consider the deterministic robot pushing objective from \cite{Wan17}, to find the best pre-image to push an object towards a target location.
We test both the 4-dimensional variant $f_4\parens{r_x,r_y,r_t, r_\theta}$
and 3-dimensional variant $f_3$, where $r_\theta = \arctan\parens{\sfrac{r_y}{r_x}}$.
Both functions return the distance from the pushed object to the target location.
Here, $r_x,r_y \in \sqparens{-5,5}$ is the robot location, pushing duration $r_t\in\sqparens{1,30}$ and pushing angle $r_\theta\in\sqparens{0,2\pi}$.
From \refig{fig:robot3d-overall} and \refig{fig:robot_pushing3d-bestparam}, the results here exhibit similar findings to most results on synthetic experiments; in particular,
\clipping performs better than \aggressivesubtraction.

\subsection{Further Experiments} 
In \refsec{sec:addexp-details} (appendix), we present more details on the above findings, as well as exploring 
(i) different kernel choices, 
(ii) attacking the elimination algorithm, 
(iii) online vs.~offline kernel learning, 
(iv) a more robust variant of GP-UCB from \cite{Bog20}, and 
(v) a proof-of-concept attack that automatically adapts its hyperparameters.

\section{Conclusion}
We have studied the problem of adversarial attacks on GP bandits, providing both a theoretical understanding of when certain attacks are guaranteed to succeed, as well as a detailed experimental study of the cost and success rate of our attacks.  Possible future research directions include (i) further investigating more ``automated'' attacks that can adapt their hyperparameters online, and (ii) theoretical lower bounds on the attack budget, which are currently unknown even in the simpler linear bandit setting.

\section*{Acknowledgments.}
This work was supported by the Singapore National Research Foundation (NRF) under grant number R-252-000-A74-281.

\newpage
\clearpage

\fontsize{9.5pt}{10.5pt}
\selectfont

\clearpage
\appendix
\section{Additional Experimental Details/Results\label{sec:addexp-details}}

\subsection{Implementation Details\label{sec:implementation}}

Our implementation is based on Python 3.8, 
using Conda \cite{anaconda} and MLflow \cite{zaharia2018accelerating} to manage environments and experiments across multiple machines.
Our implementation uses several standard machine learning and scientific packages, such as GPy \cite{gpy2014}, NumPy \cite{harris2020array} and others.
The various libraries and the specific versions used can be found in our supplied code, in the Conda environment file attack{\textunderscore}bo.yml.

Implementations of the synthetic functions are from HPOlib2 \cite{eggensperger2013}.
We use the original authors' implementation \cite{Wan17} for the robot pushing objective functions (\robotsmall and \robotlarge).

\subsection{Synthetic Functions} \label{sec:synth}

For convenience, we explicitly append the dimensionality to the name of the function. Firstly, we consider the 1D function 
\begin{equation}
    f\parens{x} = \parens{6x-2}^2 \sin\parens{12x-4},\label{eq:1d}
\end{equation}
which we refer to as \synthetic.
Then, we compare the attacks on commonly used optimization synthetic function benchmarks  \cite{eggensperger2013}: 
\forrester, \levy, \levyhard, \bohachevsky, \bohachevskyhard, \branin, \camelback, \hartmann. 
We have chosen these benchmarks so that we attack under a diverse range of conditions. 
The 1D and 2D functions are illustrated in \refig{fig:appendix-Illustration}.
\levyhard and \bohachevskyhard are variants of \levy and \bohachevsky, but with different $\targetr$, chosen to increase the difficultly of the attack.

\subsection{Time Horizon, Kernel, and Optimization Algorithm} \label{sec:kernel}

For objective functions with $1$ or $2$ dimensions, we run the experiments with $10$ initial points and $100$ iterations. Due to the added difficulty in higher dimensions, we use $50$ initial points and $250$ iterations for the other experiments.

We adopt the \matern-$\sfrac{5}{2}$ kernel, and initialize the kernel parameters with values learned from data sampled from $f$ prior to the optimization:
\begin{itemize}
    \item For 1D and 2D functions, we randomly select $100$ points from a multi-dimensional grid where each dimension is evenly spaced;
    \item For higher-dimensional functions, we use a total of $1000$ points using a mix of strategies.
    We sample $500$ points using the same grid strategy as for the 1D and 2D functions, 
    as well as an additional $500$ random points sampled from a uniform distribution and are not confined to a grid.
\end{itemize}
The kernel parameters are then fixed, and not updated throughout the attack. % The model's noise variance is set $\eta^2=0.005$ to account for noisy observations.
This amounts to an ``idealized'' setting for the player in which the kernel is well-known even in advance. 
In \refsec{sec:online_kernel}, we additionally present results for the case that the kernel parameters are learned online.  

As mentioned in the main body, we focus primarily on attacking a standard form of \gpucb with $\beta_t = 0.5\log{\parens{2t}}$.  However, in \refsec{sec:defense_exp}, we additionally consider a variant of \gpucb with enlarged confidence bounds designed for improved robustness \cite{Bog20}.

\begin{table}
    \centering
    \begin{tabular}{lrcc} 
    \toprule
    & \textbf{Name} & \textbf{Dim.} & \textbf{\#Exps.} \\
    \midrule
    \parbox[t]{2mm}{\multirow{11}{*}{\rotatebox[origin=c]{90}{\gpucb Attack}}} 
    & \synthetic & 1 & 1510 \\
    & \forrester & 1 & 1510 \\
    & \levy & 1 & 1510 \\
    & \levyhard & 1 & 1510 \\
    & \bohachevsky & 2 & 1510 \\
    & \bohachevskyhard & 2 & 1510 \\
    & \branin & 2 & 1510 \\
    & \camelback & 2 & 1510 \\
    & \hartmann & 6 & 910 \\
    \cmidrule{2-4}
    & \robotsmall & 3 & 910 \\
    & \robotlarge & 4 & 910 \\
    \midrule
    \parbox[t]{2mm}{\multirow{4}{*}{\rotatebox[origin=c]{90}{Kernels}}} 
    & \branin-\rbf & 2 & 1510 \\
    & \branin-\maternx{3}{2} & 2 & 1510 \\
    & \camelback-\rbf & 2 & 1510 \\
    & \camelback-\maternx{3}{2} & 2 & 1510 \\
    \midrule
    \parbox[t]{2mm}{\multirow{3}{*}{\rotatebox[origin=c]{90}{\maxvar}}} 
    & \varforrester & 1 & 1510 \\
    & \varcamelback & 2 & 1510 \\
    \cmidrule{2-4}
    & \varrobotsmall & 3 & 910 \\
    \midrule
    \parbox[t]{2mm}{\multirow{3}{*}{\rotatebox[origin=c]{90}{Online}}}  % \centering\noprefit
    & \npsynthetic & 1 & 1510 \\
    & \npforrester & 1 & 1510 \\
    & \npbohachevsky & 2 & 1510 \\
    \midrule
    \parbox[t]{2mm}{\multirow{9}{*}{\rotatebox[origin=c]{90}{\defense}}}
    & \levydef{0.5}{0.01} & 1 & 1510 \\
    & \levydef{0.5}{0.1} & 1 & 1510 \\
    & \levydef{0.5}{1} & 1 & 1510 \\
    & \levydef{2}{0.01} & 1 & 1510 \\
    & \levydef{2}{0.1} & 1 & 1510 \\
    & \levydef{2}{1} & 1 & 1510 \\
    & \levydef{8}{0.01} & 1 & 1510 \\
    & \levydef{8}{0.1} & 1 & 1510 \\
    & \levydef{8}{1} & 1 & 1510 \\
    \midrule
    \parbox[t]{2mm}{\multirow{3}{*}{\rotatebox[origin=c]{90}{Dynamic}}} 
    & \synthetic-\dynamic & 1 & 500 \\
    & \forrester-\dynamic & 1 & 500 \\
    & \levy-\dynamic & 1 & 500 \\
    \bottomrule
    \end{tabular}
    \caption{Summary of various experiments and functions used, sorted by type and dimensionality.} %A total of $41,990$ experiments was performed.
    \label{tab:appendix-dataset-summary}
\end{table}

\subsection{\gpucb Attack Experiments} \label{sec:gpucb}

\begin{table}
    \centering
    \begin{tabular}{rcc} 
    \toprule
    \textbf{Name} & \textbf{Centroid} & \textbf{Length} \\
    \midrule
    \synthetic & $(0)$ & $1$ \\
    \forrester & $(0.25)$ & $0.5$ \\
    \levy & $(-2.915)$ & $4$ \\
    \levyhard & $(-11.56)$ & $6.881$ \\
    \bohachevsky & $(55, 55)$ & $90$ \\
    \bohachevskyhard & $(75, 75)$ & $50$ \\
    \branin & $(-1, 11)$ & $8$ \\
    \camelback & $(0.090, -0.713)$ & $1.425$ \\
    \robotsmall & $(2.5, 2.5, 20)$ & $(5, 5, 20)$ \\
    \robotlarge & $(2.5, 2.5, 20, \sfrac{\pi}{2})$ & $(5, 5, 20, \pi)$ \\
    \hartmann & $(0.6, \cdots, 0.6)$ & $0.8$ \\
    \bottomrule
    \end{tabular}
    \caption{Summary of $\targetr$ parameters.}
    \label{tab:appendix-targetr}
\end{table}

\begin{figure*}[!ht]
    \centering
    \subfloat[\synthetic]{\includegraphics[width=0.32\textwidth]{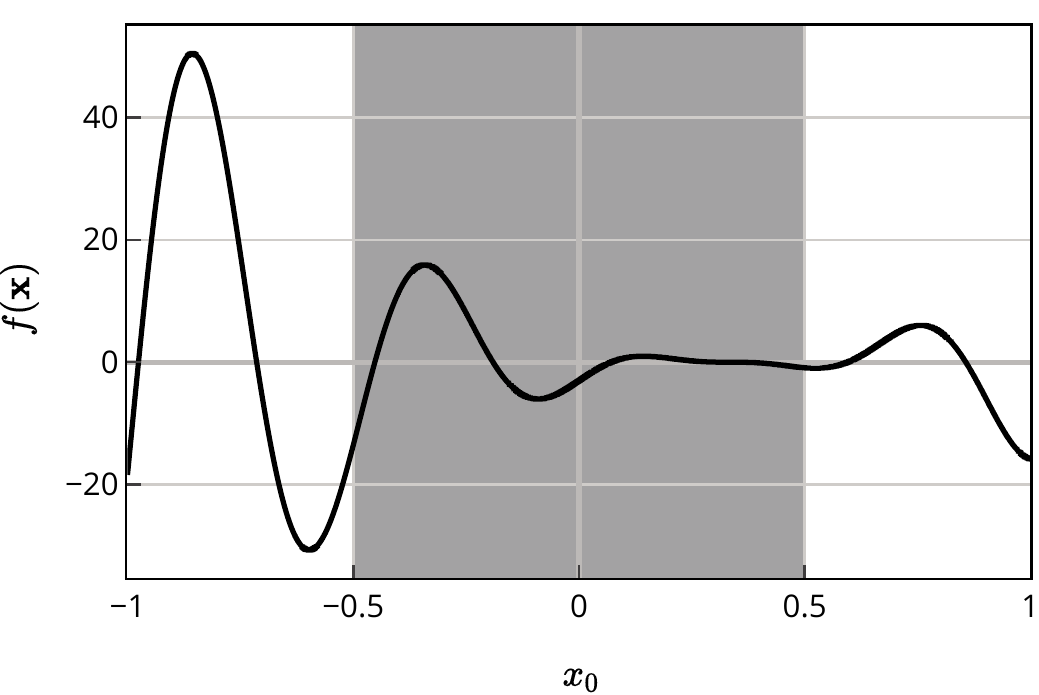}\label{fig:appendix-1d_Illustration_example-Illustration}}
    \subfloat[\forrester]{\includegraphics[width=0.32\textwidth]{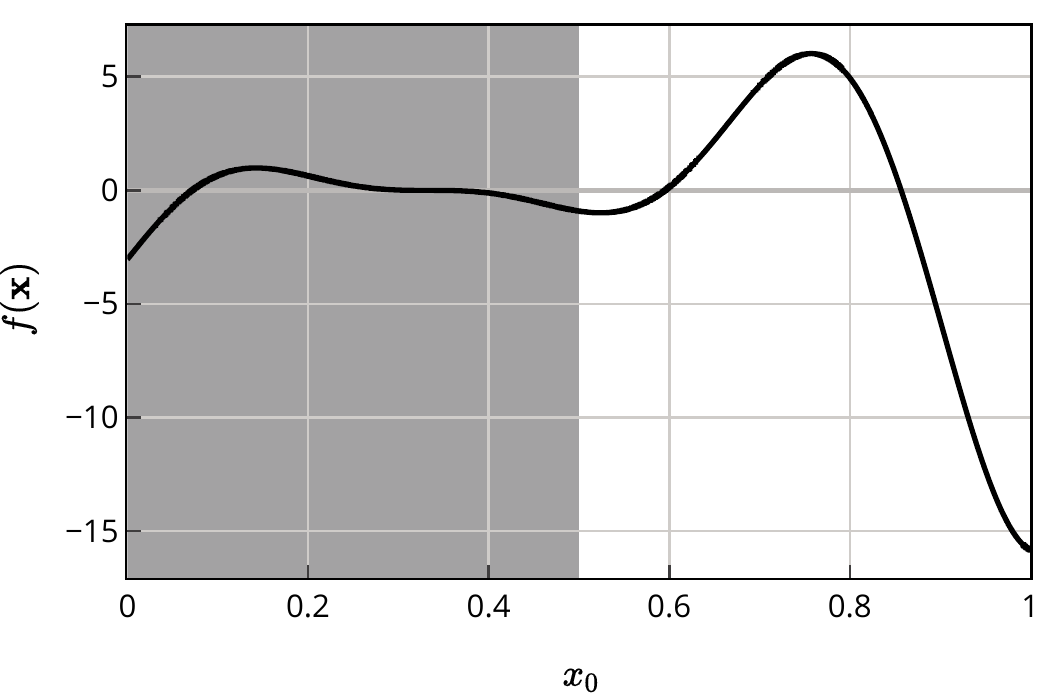}\label{fig:appendix-forrester_Illustration_example-Illustration}}
    \subfloat[\levy]{\includegraphics[width=0.32\textwidth]{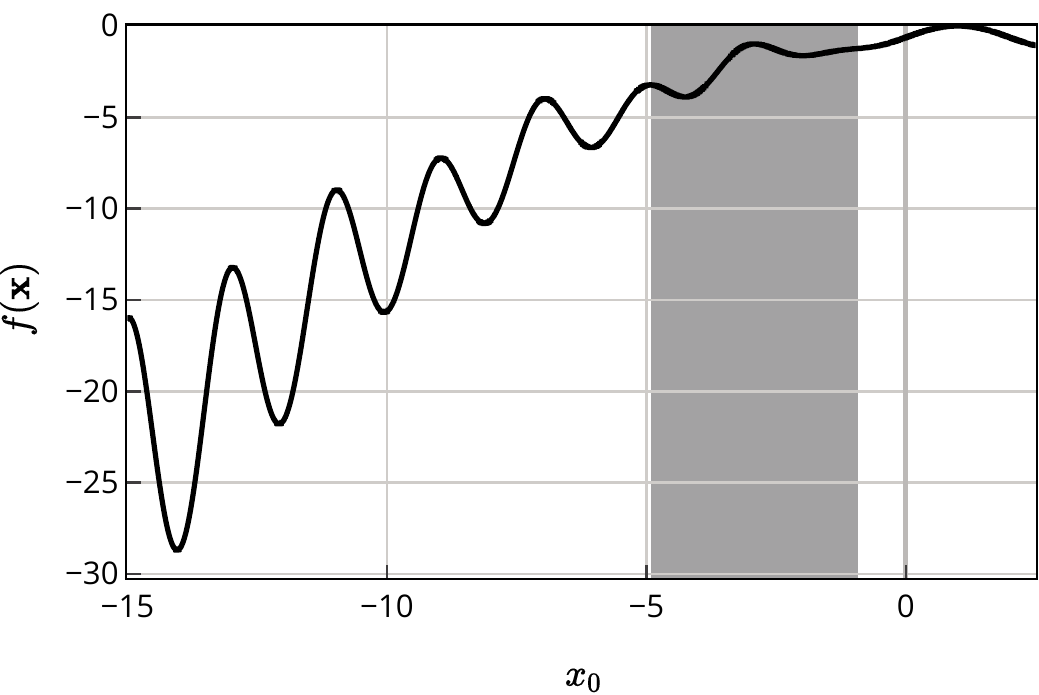}\label{fig:appendix-levy_Illustration_example-Illustration}}\\
    \subfloat[\levyhard]{\includegraphics[width=0.32\textwidth]{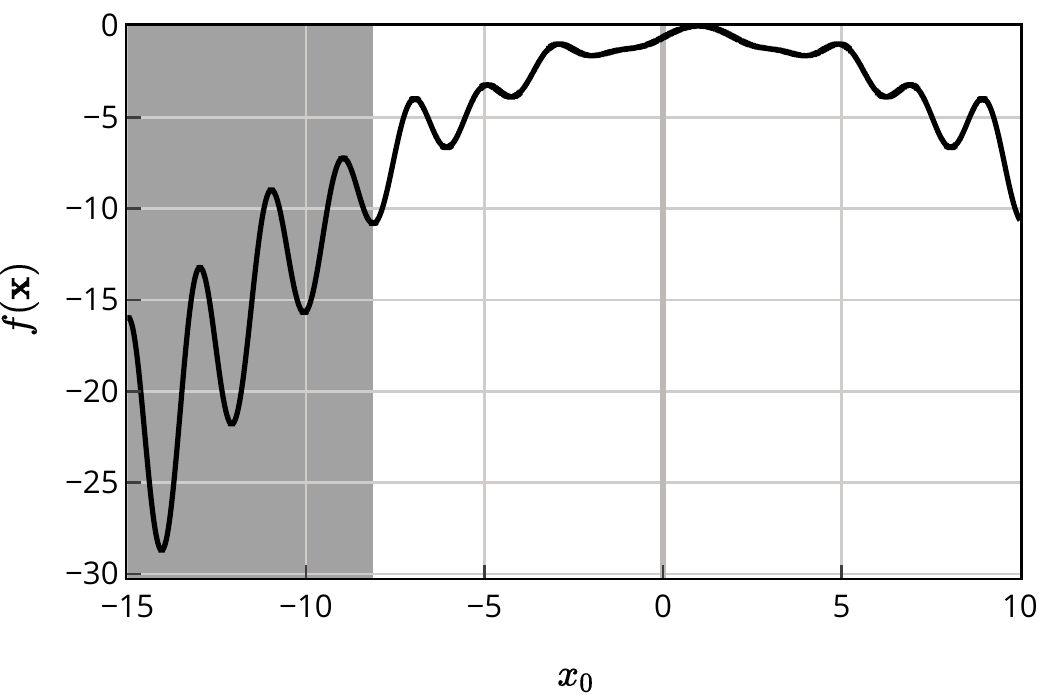}\label{fig:appendix-levy_hard_Illustration_example-Illustration}}
    \subfloat[\bohachevsky]{\includegraphics[width=0.32\textwidth]{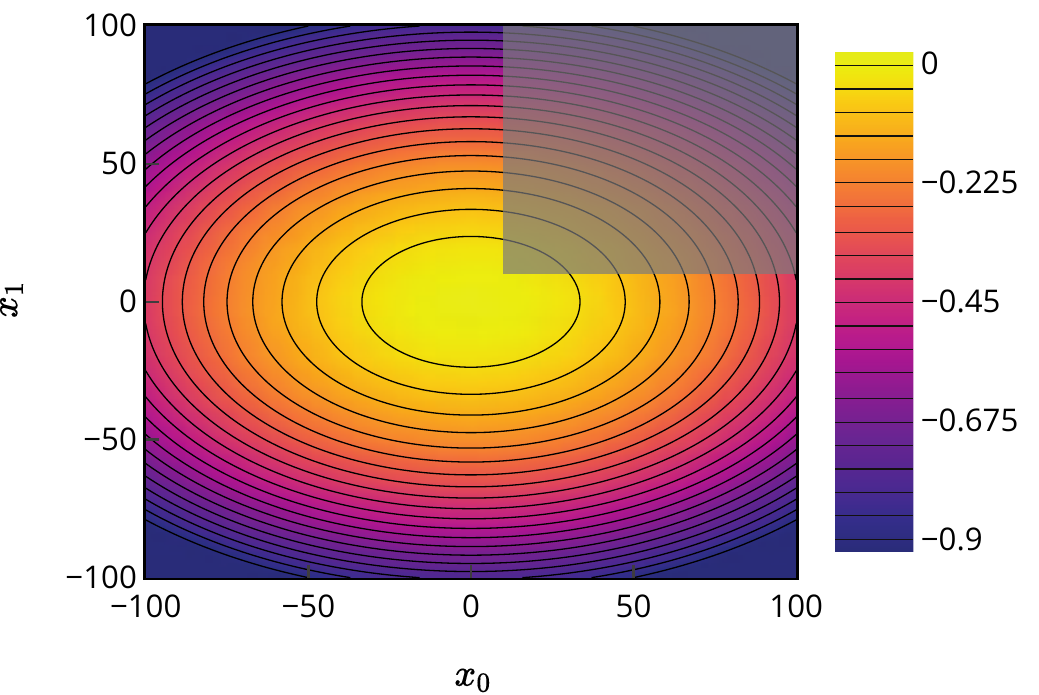}\label{fig:appendix-bohachevsky_Illustration_example-Illustration}}
    \subfloat[\bohachevskyhard]{\includegraphics[width=0.32\textwidth]{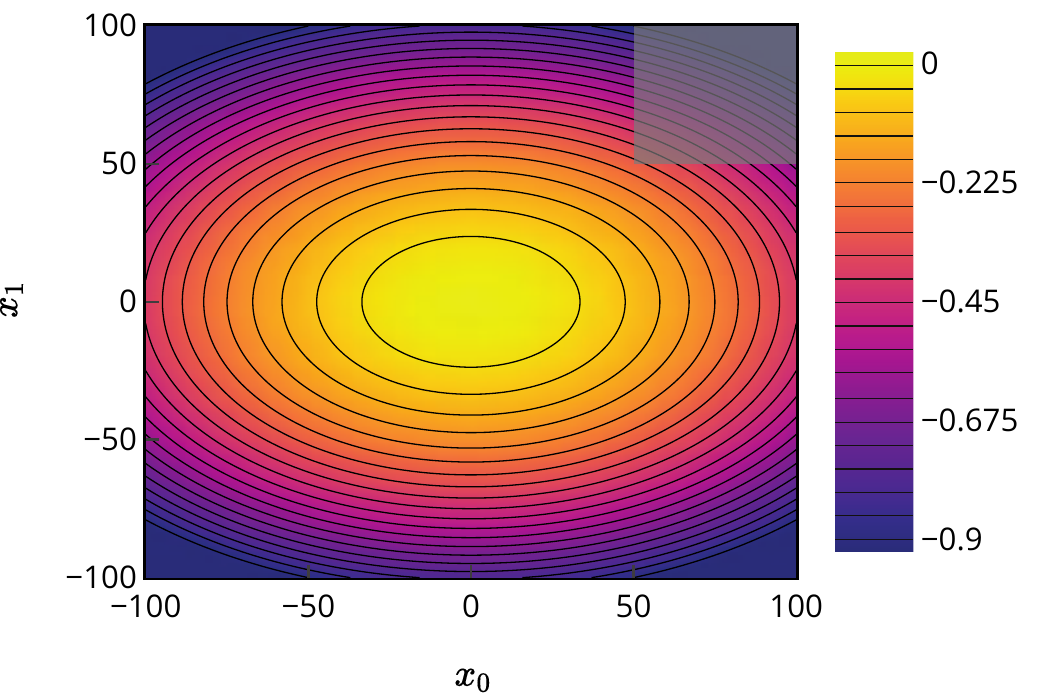}\label{fig:appendix-bohachevsky_hard_Illustration_example-Illustration}}\\
    \subfloat[\branin]{\includegraphics[width=0.32\textwidth]{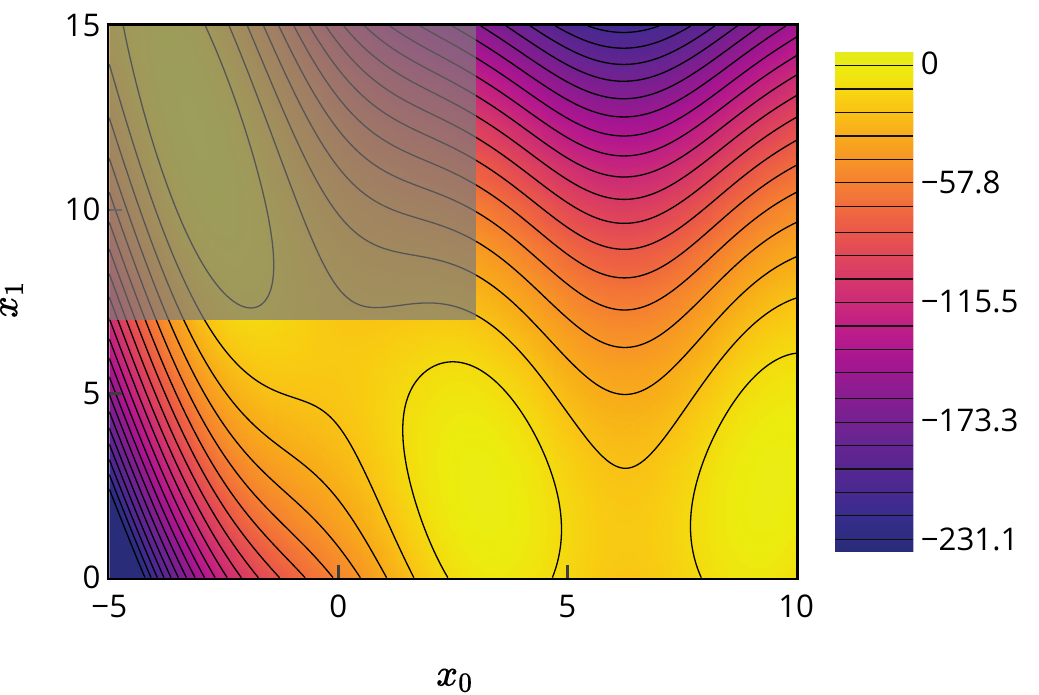}\label{fig:appendix-branin_Illustration_example-Illustration}}
    \subfloat[\camelback]{\includegraphics[width=0.32\textwidth]{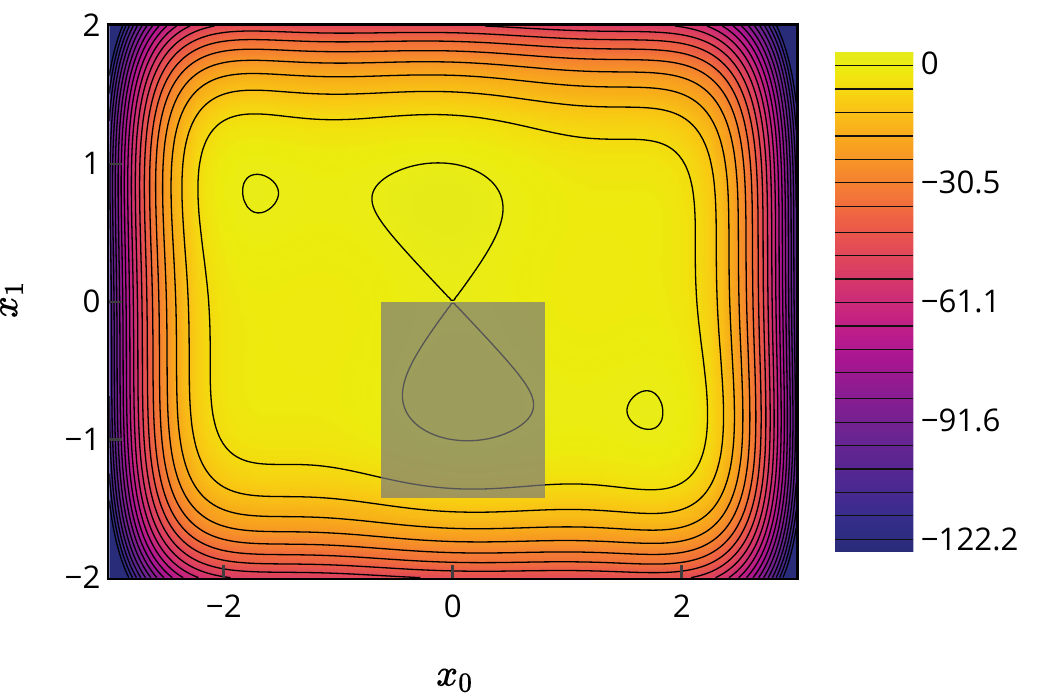}\label{fig:appendix-camelback_Illustration_example-Illustration}}
    \caption{Illustrations for 1D and 2D experiments; $\targetr$ is shaded in each plot.}
    \label{fig:appendix-Illustration}
\end{figure*}

{\bf Discussion of functions used.} % Our benchmarks are chosen to demonstrate a diverse range of conditions for testing our attacks.
As seen in \refig{fig:appendix-Illustration}, each target region $\targetr$ is rectangular, and defined by its centroid and length;
the choices for each experiment are given in \reftbl{tab:appendix-targetr}.

\synthetic and \forrester are relatively simple examples where the second-best peak falls inside of the $\targetr$ and there are few local maxima.
\levy is an experiment with a periodic function, which increases the difficulty by the introduction of multiple local maxima.
\levyhard uses the same function, but with a different domain and $\targetr$
such that the second-best peak falls outside of $\targetr$.
\bohachevsky has a characteristic bowl shape with a single maximum, which is easy to optimize but difficult to attack.
Here, we use a variant of the \bohachevsky function that is scaled by a constant factor to better manage the function range.
\bohachevskyhard uses the same \bohachevsky function but with a smaller $\targetr$ to further increase the difficulty.  Specifically, the function increases throughout $\targetr$ and when leaving $\targetr$,
causing the attack to be especially difficult as explained in \refsec{sec:difficult}.

\branin has $3$ global maxima, and $\targetr$ contains one of the global maxima.
\camelback has $2$ global maxima and several local maxima, and $\targetr$ contains one of the global maxima.
As for an example with higher dimensionality,
\hartmann is a $6$-dimensional function with a global maximum lying outside $\targetr$.

Finally, we attack the $3$-dimensional and $4$-dimensional functions -- \robotsmall and \robotlarge respectively \cite{Wan17}. 
In both of these experiments, $\targetr$ is again rectangular, but with different lengths in each dimension.

{\bf Discussion of hyperparameters used.} We consider $30$ hyperparameter configurations (e.g., choices of $\Delta$ or $h_{\max}$) for each attack,
choosing the relevant minimum and maximum values so that we sufficiently cover the entire spectrum of the behavior, 
e.g., configurations with low to high success rates.
From \refigs{fig:appendix-AggressiveSubtractionAttack}{fig:appendix-SubtractionAttackSq}, 
we can observe that higher costs may not associate with a higher success rate; this is also discussed in the main body.

\begin{figure}
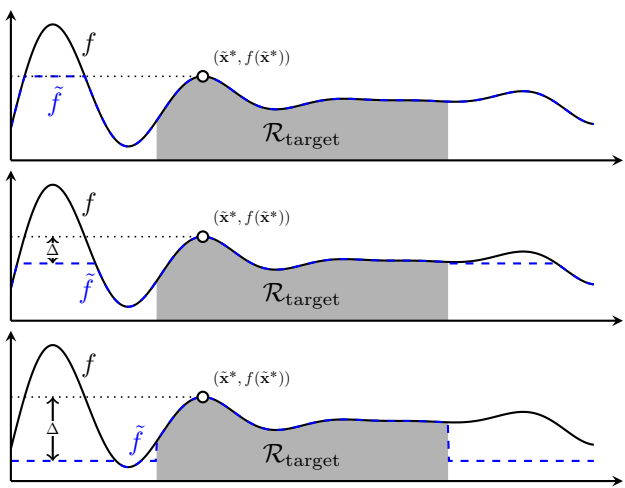

    \centering
    \illustration{
        \addplot [blue,dashed,name path=SUB_RND] {clipping(x,0)} node[below,pos=0.1835] {$\tilde{f}$};
      }{
        \draw[] (-0.3413, {f(-0.3413)}) node[dot,label={north east,scale=\illscale:$\parens{\tilde{\xv}^*,f\parens{\tilde{\xv}^*}}$}]{};
      }{-40}{100}
      \illustration{
        \addplot [blue,dashed,name path=SUB_RND] {clipping(x,17.8)} node[below,pos=0.1211] {$\tilde{f}$};
      }{
        \draw[] (-0.3413, {f(-0.3413)}) node[dot,label={north east,scale=\illscale:$\parens{\tilde{\xv}^*,f\parens{\tilde{\xv}^*}}$}]{};
        \deltaline{-0.85636275}{$\Delta$}{17.8};
      }{-40}{100}
      \illustration{
        \addplot [blue,dashed,name path=SUB_RND] {clipping(x,42.5)} node[above,pos=0.056] {$\tilde{f}$};
      }{
        \draw[] (-0.3413, {f(-0.3413)}) node[dot,label={north east,scale=\illscale:$\parens{\tilde{\xv}^*,f\parens{\tilde{\xv}^*}}$}]{};
        \deltaline{-0.85636275}{$\Delta$}{42.5};
      }{-40}{100}
    \caption{\clipping attack on \synthetic with `under-aggressive' $\Delta=0$ (top), 
    most efficient $\Delta_\text{efficient}=17.8$ (middle),
    and `over-aggressive' $\Delta=42.5$ (bottom). 
    See \refeq{eq:efficient} for the definition of ``most efficient''.
    \label{fig:appendix-aggressive}}
\end{figure}

In general, increasing a hyperparameter value is representative of increasing the aggressiveness of the attack.
Here, we see examples of poorly chosen parameters (\eg, smallest hyperparameter value) having high cost yet attacking unsuccessfully.
For example, from \refig{fig:appendix-1d_TripleAxisParam_ClippingAttack-TripleAxisParam-norm_cum_c_t-success_rate} and illustrated in \refig{fig:appendix-aggressive} (top), 
we see that the poorly chosen hyperparameter $\Delta=0$ for \clipping incurs a high cost with a poor success rate.
In this case, the attack is {\em not aggressive enough} (\ie `under-aggressive'), and the \gpucb algorithm is not incentivized to explore $\targetr$.
Over many iterations, \gpucb repeatedly explores $\targetr$,
and continues to determine that $\targetr$ is no better than $\parens{\tilde{\xv}^*,f\parens{\tilde{\xv}^*}}$.
Hence, the attack can incur a high cost due to \gpucb continuing to choose points outside $\targetr$ across the entire time horizon.

At the other end of the spectrum, 
increasing the aggressiveness of the attack can sometimes incur significantly more cost without any insignificant increase in success rate.
For instance, we observe in \refig{fig:appendix-1d_TripleAxisParam_ClippingAttack-TripleAxisParam-norm_cum_c_t-success_rate} (see also \refig{fig:appendix-aggressive}) that the hyperparameter $\Delta=42.5$ (bottom) incurs more cost then the most efficient $\Delta_\text{efficient}=17.8$ (middle) but without a significant increase in success rate.
Thus, the attack is overly aggressive; more cost than necessary is being used to incentivize the algorithm to choose actions in $\targetr$.  
In this particular case (\refig{fig:appendix-aggressive}), increasing $\Delta$ more than $\Delta_\text{efficient}$ increases the cost, as it unnecessarily, 
``swallows'' more of the the local maxima on the right of $\targetr$.
On the other hand, sometimes increasing the aggressiveness of the attack can have a minimal impact on the cost-success trade-off. 
For instance, we see this behavior in \refig{fig:appendix-robot_pushing3d_TripleAxisParam_AggressiveSubtractionAttack-TripleAxisParam-norm_cum_c_t-success_rate},
and \refig{fig:appendix-camelback_TripleAxisParam_ClippingAttack-TripleAxisParam-norm_cum_c_t-success_rate}, where increasing aggression improves the success rate.

Considering that successfully forcing actions in $\targetr$ is the attacker's primary goal and that $\theta_\text{efficient}$ is difficult to find, 
in general, it may be preferable to be `over-aggressive' than `under-aggressive' when considering hyperparameters.

{\bf Comparison of attacks.} In \refigs{fig:appendix-SuccessRate_Cost_Avg}{fig:appendix-SuccessRate_Cost}, 
we again see \clipping attacking efficiently and successfully, as discussed in \refsec{sec:experiments}.
Similarly, we continue to observe that \subtractionrnd and \subtractionsq are often the next best attacks following \clipping.
In \refigs{fig:appendix-SubtractionAttackRnd}{fig:appendix-SubtractionAttackSq}, 
\subtractionrnd has a slight advantage over \subtractionsq, though the two are similar.
Due to \subtractionrnd's smooth $h(x)$, the attack behavior varies less with respect to its hyperparameter $h_{\max}$, and the attacks generally succeed with slightly less cost compared with \subtractionsq.

The unique behavior of the various attacks in \bohachevsky is similar to the behavior observed for \bohachevskyhard.
The efficiency of the \aggressivesubtraction attack in \bohachevskyhard continues to be similar to \bohachevsky,
despite having a smaller $\targetr$ in \bohachevskyhard.
Both \bohachevsky and \bohachevskyhard serve as illustrations of the difficult case discussed in \refsec{sec:difficult}.

\subsection{Kernel Experiments} \label{sec:diffkernels}

We have focused on the \maternx{5}{2} kernel, as it is widely-adopted and provides useful properties that we can exploit in our theoretical analysis.
Here we additionally demonstrate our attacks on two other popular kernel functions: Radial Basis Function (\rbf) and \maternx{3}{2}.
Apart from the kernel choice, the setup is the same as that of \refsec{sec:gpucb}.

From \refig{fig:appendix-diffkernels}, we observe that the RBF kernel is easier to attack as it is smoother than \maternx{5}{2}; the \rbf kernel corresponds to a \matern kernel where $\nu\rightarrow\infty$, leading to infinitely many derivatives.
\refthm{thm:main} indicates that smoother functions (smaller $\gamma_t$) are easier to attack.
Similarly, \maternx{3}{2} is harder to attack, which is consistent with it having a larger $\gamma_t$.  Intuitively, using a less smooth kernel encourages more exploration, which makes it more difficult to force the algorithm into the target region.

\subsection{\maxvar+~Elimination Attack Experiments} \label{sec:maxvar}

Here we move from attacking \gpucb to attacking the MaxVar+Elimination algorithm, described in \refsec{sec:opt_algs}.
With the exception of the selection rule itself, 
other settings in each of the \maxvar attack experiments are identical to the corresponding \gpucb attack experiments.

We observe from \refig{fig:appendix-maxvar} that the behavior of both \varcamelback and  \varrobotsmall are similar to the \gpucb counterparts.
Notwithstanding the non-standard `gap' as observed in the figure, \varforrester behaved similarly when compared to \forrester.
One exception is the `gap' in \refig{fig:appendix-var_forrester_Scatter_overall_avg_log_log_rev-Scatter-success_rate-norm_cum_c_t}, which may be due to the `non-smooth' selection criteria of the $t$-th point, explicitly
filtering candidate points that are not as high as the highest LCB.

\subsection{\noprefit Experiments} \label{sec:online_kernel}

In \gpucb Attack, we initialized the kernel parameters with values learned from data sampled from $f$ prior to the optimization.
Here, we consider an alternative setting in which the kernel parameters are learned online.
The kernel parameters here are updated using the widely-used technique -- 
maximum likelihood optimized with the Limited-memory Broyden-Fletcher-Goldfarb-Shanno (L-BFGS) algorithm.
Our experiments use the same settings as the corresponding \gpucb attack experiments, with the exception of the kernel parameters.

Comparing the online vs.~offline approaches in \refig{fig:appendix-noprefit},
we see that there is generally no major difference in the behavior of these two approaches; being online vs.~offline does not appear to significantly impact robustness. 
A possible exception is that some {\em specific runs} (i.e., random seeds) appear to have an unusually low success rate in \npsynthetic when online kernel learning is used, particularly with \subtractionrnd.  A possible explanation for this is that the attack can sometimes perturb the function to the extent that the maximum-likelihood kernel parameters are highly misleading, causing the algorithm to behave in an unpredictable manner, and choose fewer points from $\targetr$.

\subsection{\defense Experiments} \label{sec:defense_exp}

We consider defending against our attacks using a combination of two techniques. 
Firstly, we use a technique proposed in \cite{Bog20}, namely, adding a constant $C$ to the exploration parameter:
\begin{equation}
    \beta_t = 0.5\log{\parens{2t}} + C. \label{eq:beta_expanded}
\end{equation}
Secondly, we increase the model's robustness to adversarial noise via the model's noise variance $\eta^2$ in the posterior update equations: 
\begin{equation}
    \begin{aligned}
        \mu_{t+1}\parens{\xv} &= \kv_t\parens{\xv}^\top\parens{\Kv_t+\eta^2\Iv_t}^{-1}\yv_t, \\
        \sigma_{t+1}\parens{\xv} &= k\parens{\xv, \xv} - \kv_t\parens{\xv}^\top\parens{\Kv_t+\eta^2\Iv_t}^{-1}\kv_t\parens{\xv}.
    \end{aligned}
    \label{eq:predictive_gpr}
\end{equation}
We vary $C \in \{0.5,2,8\}$ with $\eta^2 \in \{0.01,0.1,1\}$, resulting in nine different parameter combinations, allowing us to study their combined effect.

We observe from \refig{fig:appendix-defend} that increasing $C$ indeed decreases the attack success rate and increases the cost for each run, 
though seemingly not in a very major way. 
Essentially, increasing the confidence width causes \gpucb to become more explorative and less exploitative, 
making it harder for the attacker to quickly steer the algorithm towards $\targetr$, and hence increasing the attack cost.

Similarly, we see that increasing $\eta^2$ tends to increase the cost for each run.
By increasing $\eta^2$, we make the model expect more noise, and thus be more cautious against `unusual' observations.
Thus, the adversary needs to spend more cost to perform the attack.
This defense works significantly better on the Subtraction Attacks (\subtractionrnd, \subtractionsq) than on \clipping and \aggressivesubtraction.
We believe that this is because \clipping and \aggressivesubtraction perturb the function across the majority of the domain rather than a few localized regions.  The latter can potentially be written off as random noise more easily than the former.

The higher the value of $C$ used in \refeq{eq:beta_expanded} and of $\eta^2$ used in \refeq{eq:predictive_gpr},
the more similar the attacks tend to behave, and the less sensitive the algorithm tends to be to varying the attacker's hyperparameters.
We observe from \refig{fig:appendix-levy_def_C8_S1_Scatter_overall_avg_log_log_rev-Scatter-success_rate-norm_cum_c_t} 
that the combination of both simple defense techniques on \subtractionrnd and \subtractionsq increase robustness.

Overall, we believe that this experiment, and our work in general, motivates the study of further defenses beyond the simple defenses proposed,
particularly when targeting applications in which robustness is critical.

\subsection{\dynamic Experiments} \label{sec:dynamic_exp}

We have focused on the fixed-hyperparameter setting, where the hyperparameters are fixed prior to the attack (e.g., $\Delta$ is pre-determined and fixed in the \aggressivesubtraction Attack).  This setting serves as an important stepping stone to more general settings, and also aligns with our theory.

On the other hand, from our experiments and discussion, it is evident that the choice of the attack's hyperparameters can be a very important factor in the success of the attack.  If $f$ is known to the attacker, then in principle various configurations could be simulated to find the best one.  However, pre-specifying the hyperparameters may be much more difficult when limited prior knowledge is available.

Here, we provide an initial investigation into whether the hyperparameters can be chosen automatically, \ie, dynamically adjusted online during the attack.
We consider a simple strategy to adjust the hyperparameter:
The attack gets less aggressive whenever $\targetr$ is sampled consecutively $K$ times, and more aggressive otherwise.  Thus, the hyperparameter (e.g., representing $\Delta$) is updated as follows:
\begin{equation}
    \theta_{t} = 
    \begin{cases}       
        \theta_{t-1} - F \cdot \theta_{t-1} & \text{sampled consecutively,}\\
        \theta_{t-1} + F \cdot \theta_{t-1} & \text{otherwise},
    \end{cases}
\end{equation}
where $F$ is the size of the hyperparameter update, as a fraction of the current value.

Apart from this change, we consider the same setup as that of \refsec{sec:gpucb}.
We apply the strategy to \aggressivesubtraction on \synthetic, \forrester and \levy; based on minimal manual tuning, we set
$F=0.1$ and $K=3$ across the three functions without further tweaking.   We compare the fixed vs.~dynamic attacks for various hyperparameters, where setting a hyperparameter in the dynamic setting means fixing its {\em initial} value only.

From \refig{fig:appendix-dynamic}, this strategy works well on these functions, significantly improving the performance (i.e., trade-off of success rate and cost) for the configurations that perform poorly in the case of being fixed.  This experiment serves as a useful proof of concept, but further research is needed towards establishing fully automated attacks method for general scenarios.  
We believe that our work provides a good starting point towards this goal.

\begin{figure*}[!hb]
    \centering
    \subfloat[\synthetic]{\includegraphics[width=0.32\textwidth]{img___1d___TripleAxisParam___AggressiveSubtractionAttack-TripleAxisParam-norm_cum_c_t-success_rate.pdf}\label{fig:appendix-1d_TripleAxisParam_AggressiveSubtractionAttack-TripleAxisParam-norm_cum_c_t-success_rate}}
    \subfloat[\forrester]{\includegraphics[width=0.32\textwidth]{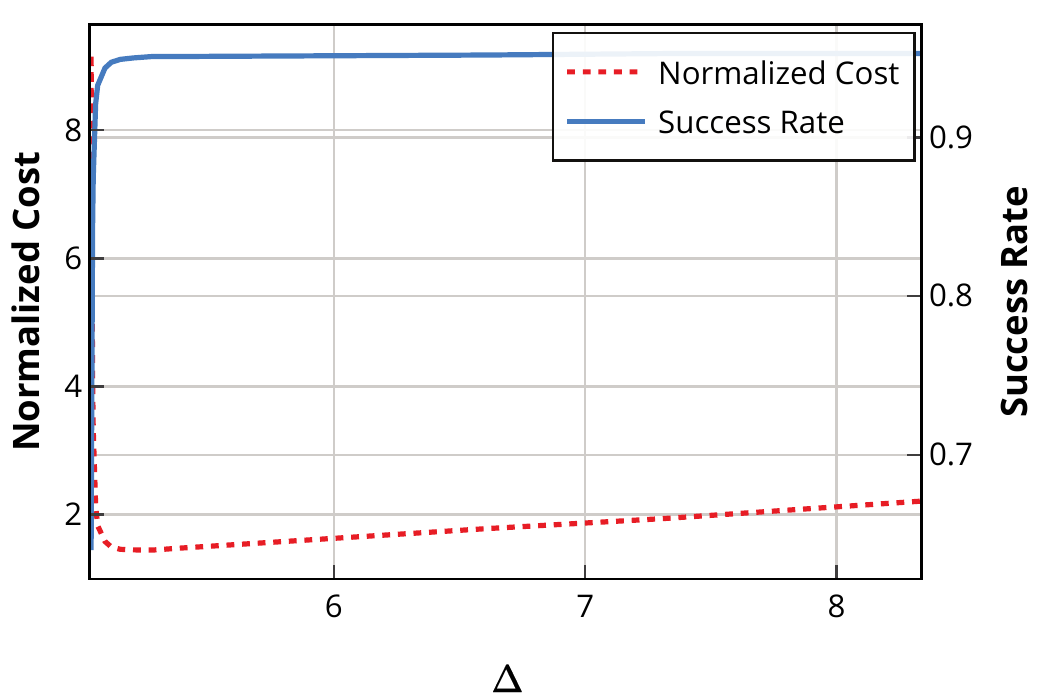}\label{fig:appendix-forrester_TripleAxisParam_AggressiveSubtractionAttack-TripleAxisParam-norm_cum_c_t-success_rate}}
    \subfloat[\levy]{\includegraphics[width=0.32\textwidth]{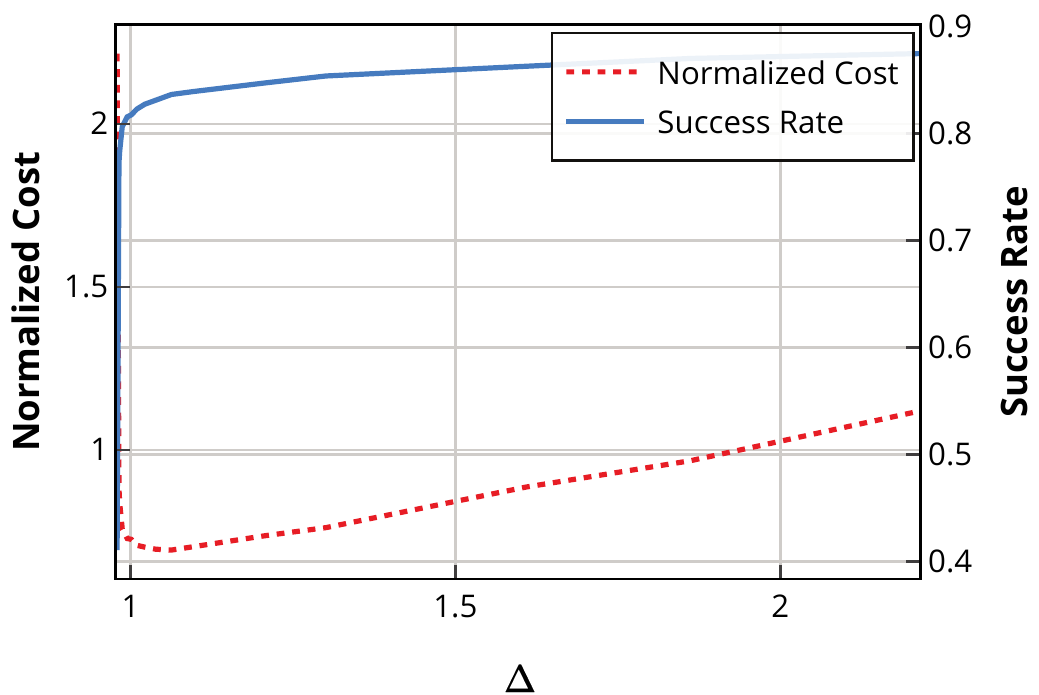}\label{fig:appendix-levy_TripleAxisParam_AggressiveSubtractionAttack-TripleAxisParam-norm_cum_c_t-success_rate}}\\
    \subfloat[\levyhard]{\includegraphics[width=0.32\textwidth]{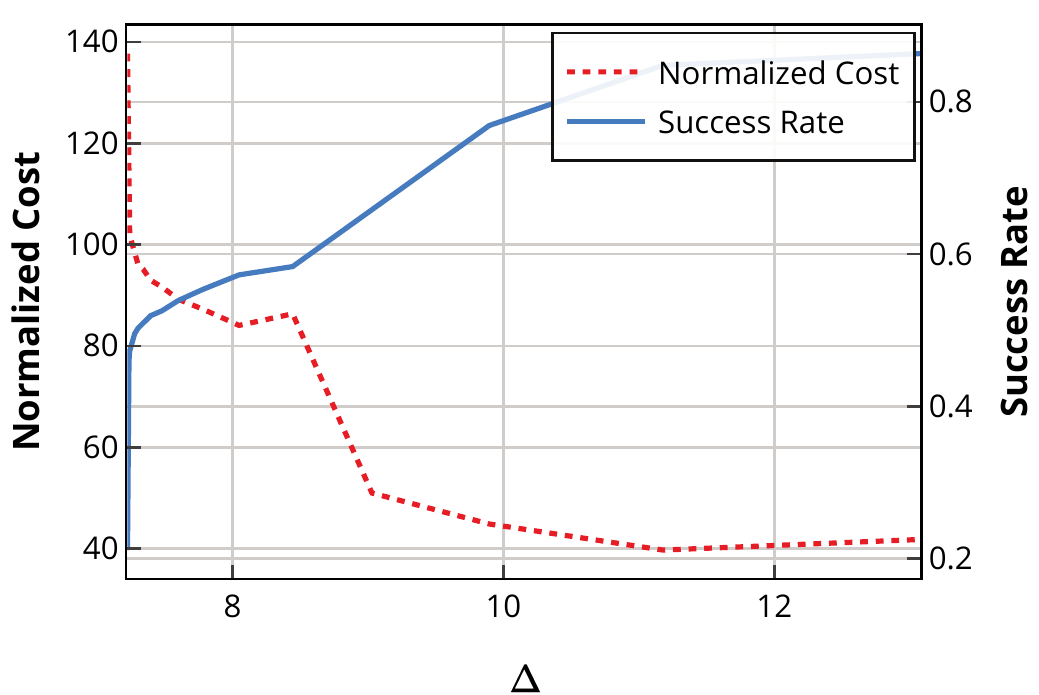}\label{fig:appendix-levy_hard_TripleAxisParam_AggressiveSubtractionAttack-TripleAxisParam-norm_cum_c_t-success_rate}}
    \subfloat[\bohachevsky]{\includegraphics[width=0.32\textwidth]{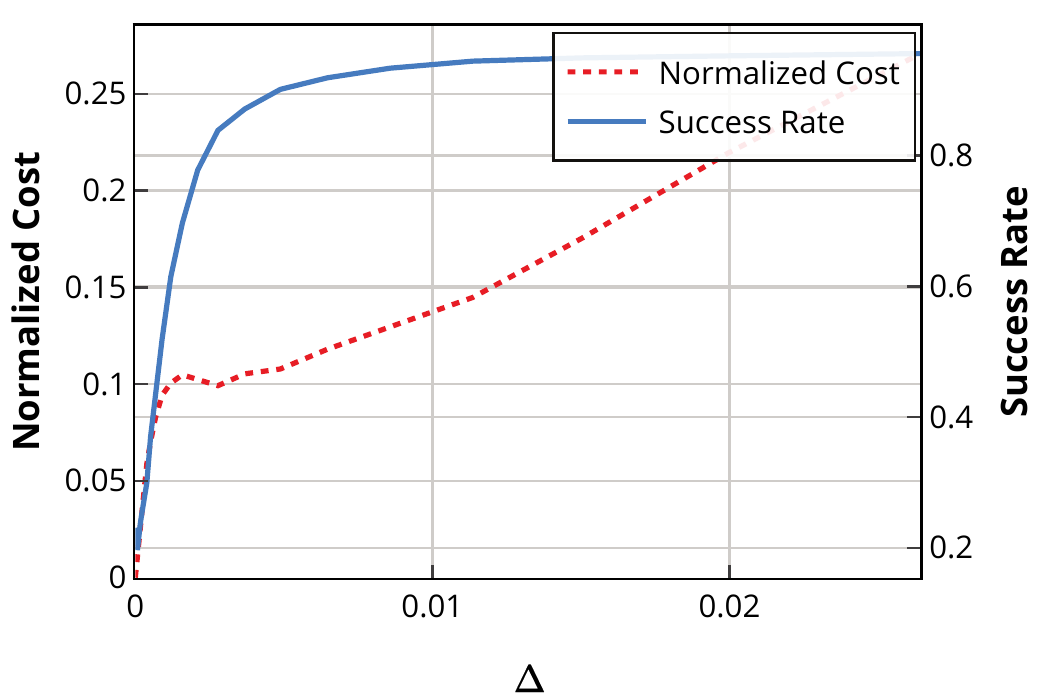}\label{fig:appendix-bohachevsky_TripleAxisParam_AggressiveSubtractionAttack-TripleAxisParam-norm_cum_c_t-success_rate}}
    \subfloat[\bohachevskyhard]{\includegraphics[width=0.32\textwidth]{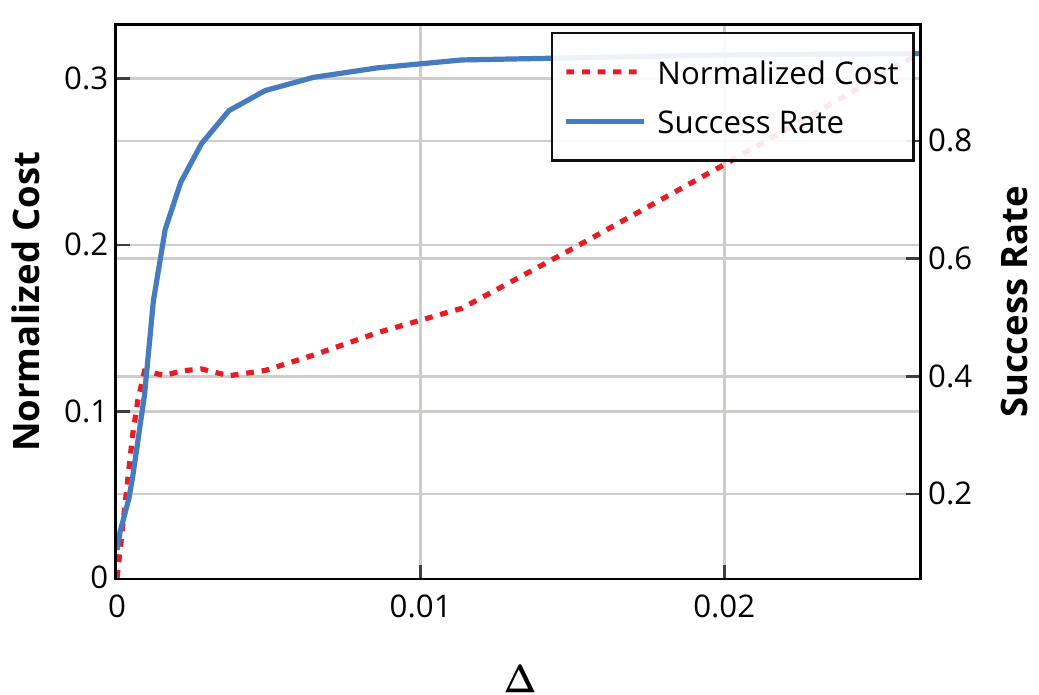}\label{fig:appendix-bohachevsky_hard_TripleAxisParam_AggressiveSubtractionAttack-TripleAxisParam-norm_cum_c_t-success_rate}}\\
    \subfloat[\branin]{\includegraphics[width=0.32\textwidth]{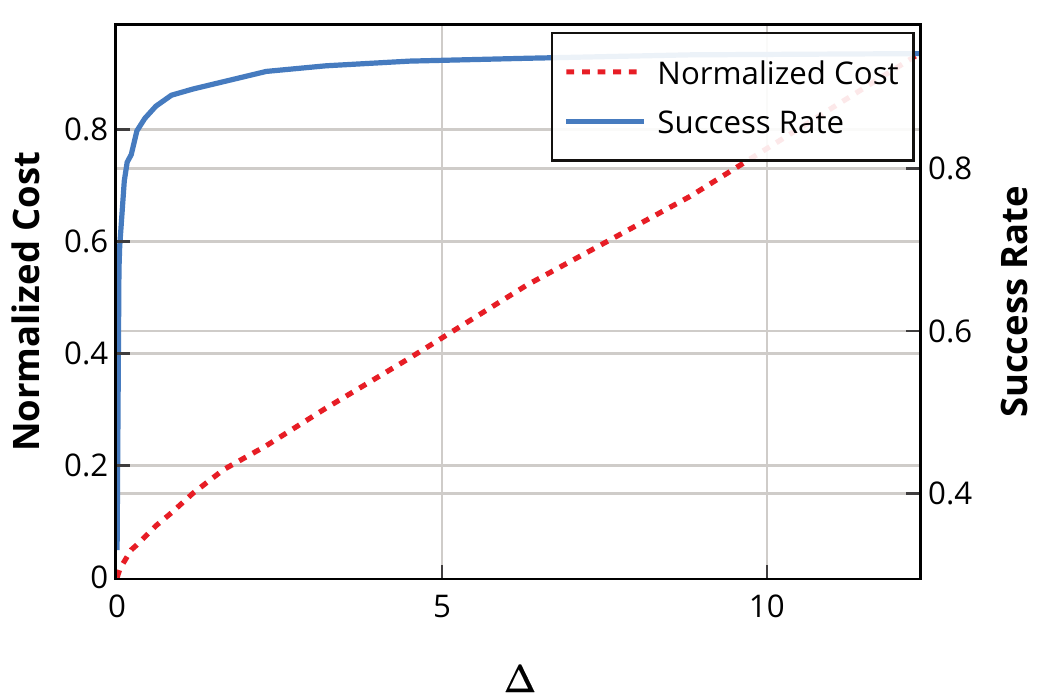}\label{fig:appendix-branin_TripleAxisParam_AggressiveSubtractionAttack-TripleAxisParam-norm_cum_c_t-success_rate}}
    \subfloat[\camelback]{\includegraphics[width=0.32\textwidth]{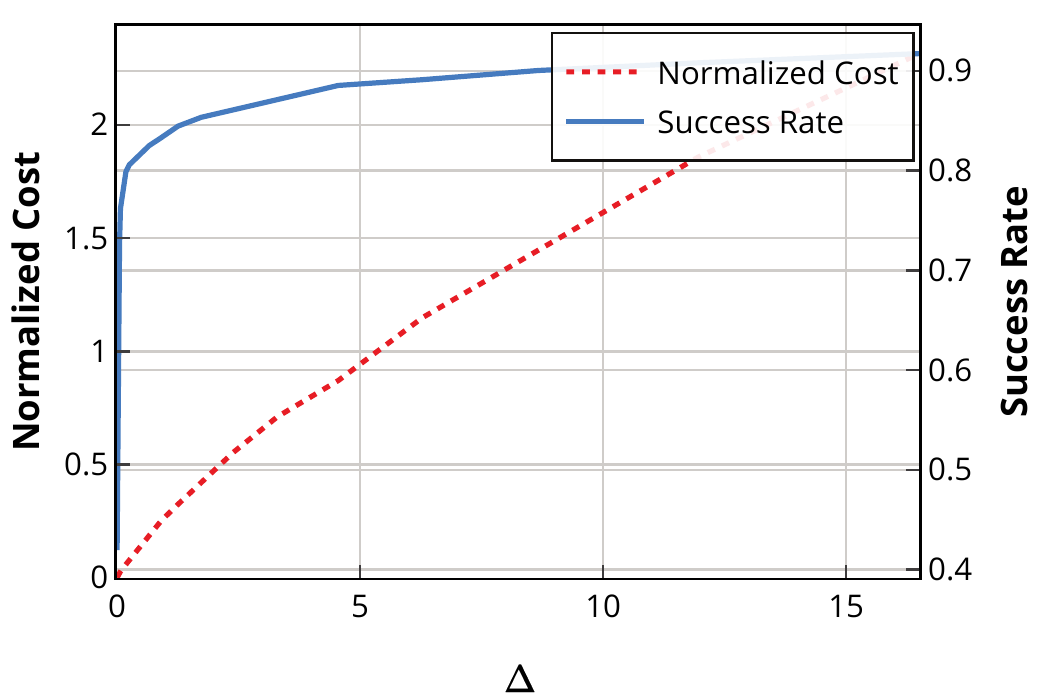}\label{fig:appendix-camelback_TripleAxisParam_AggressiveSubtractionAttack-TripleAxisParam-norm_cum_c_t-success_rate}}
    \subfloat[\hartmann]{\includegraphics[width=0.32\textwidth]{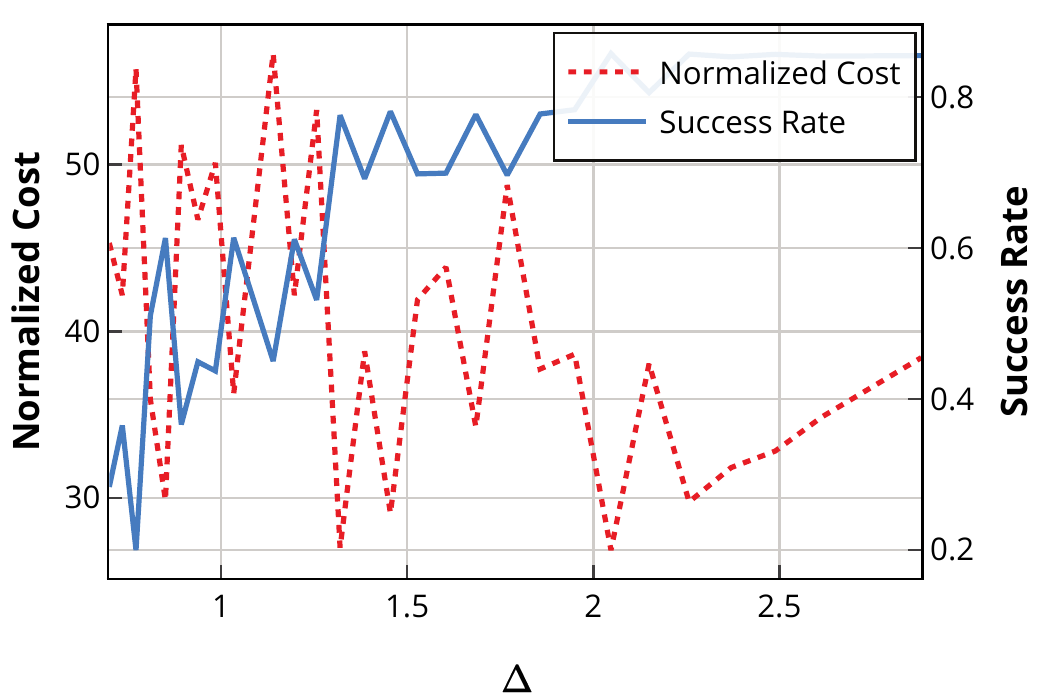}\label{fig:appendix-hartmann6_TripleAxisParam_AggressiveSubtractionAttack-TripleAxisParam-norm_cum_c_t-success_rate}}\\
    \subfloat[\robotsmall]{\includegraphics[width=0.32\textwidth]{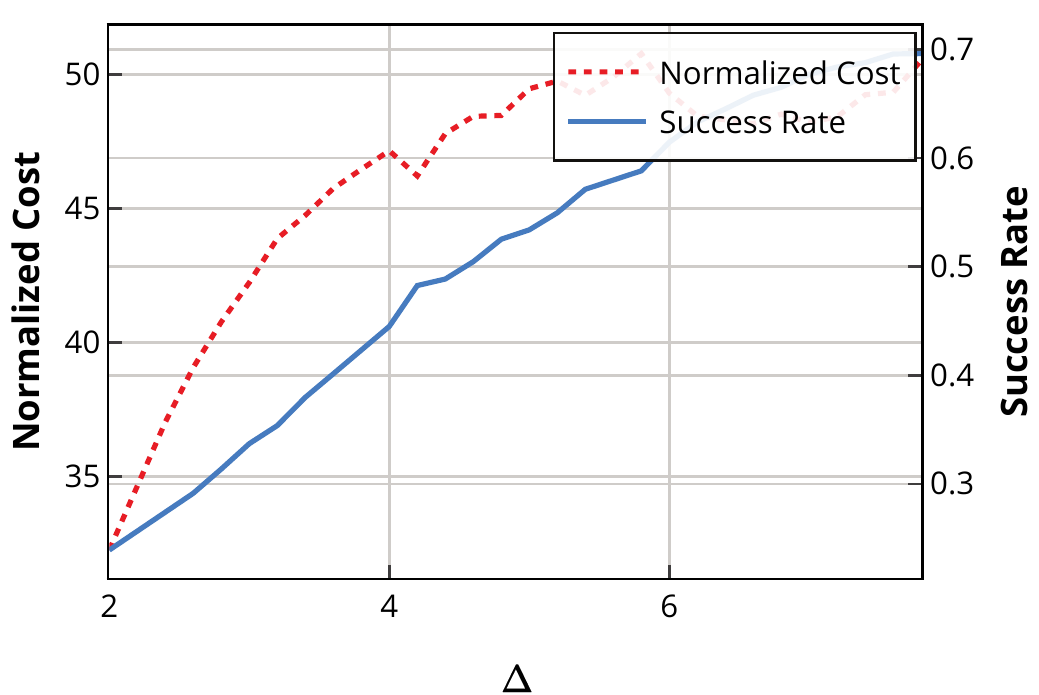}\label{fig:appendix-robot_pushing3d_TripleAxisParam_AggressiveSubtractionAttack-TripleAxisParam-norm_cum_c_t-success_rate}}
    \subfloat[\robotlarge]{\includegraphics[width=0.32\textwidth]{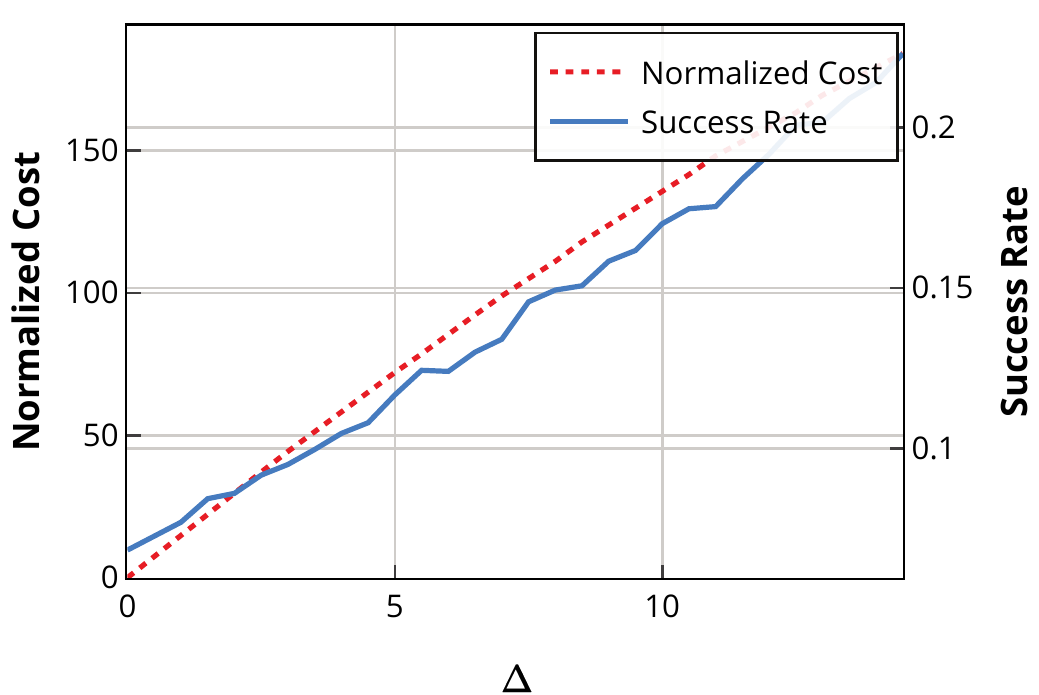}\label{fig:appendix-robot_pushing4d_TripleAxisParam_AggressiveSubtractionAttack-TripleAxisParam-norm_cum_c_t-success_rate}}
   
    \caption{Effect of \aggressivesubtraction's hyperparameter $\Delta$ on the success rate and cost incurred.}
    \label{fig:appendix-AggressiveSubtractionAttack}
\end{figure*}

\begin{figure*}[!ht]
    \centering
    \subfloat[\synthetic]{\includegraphics[width=0.32\textwidth]{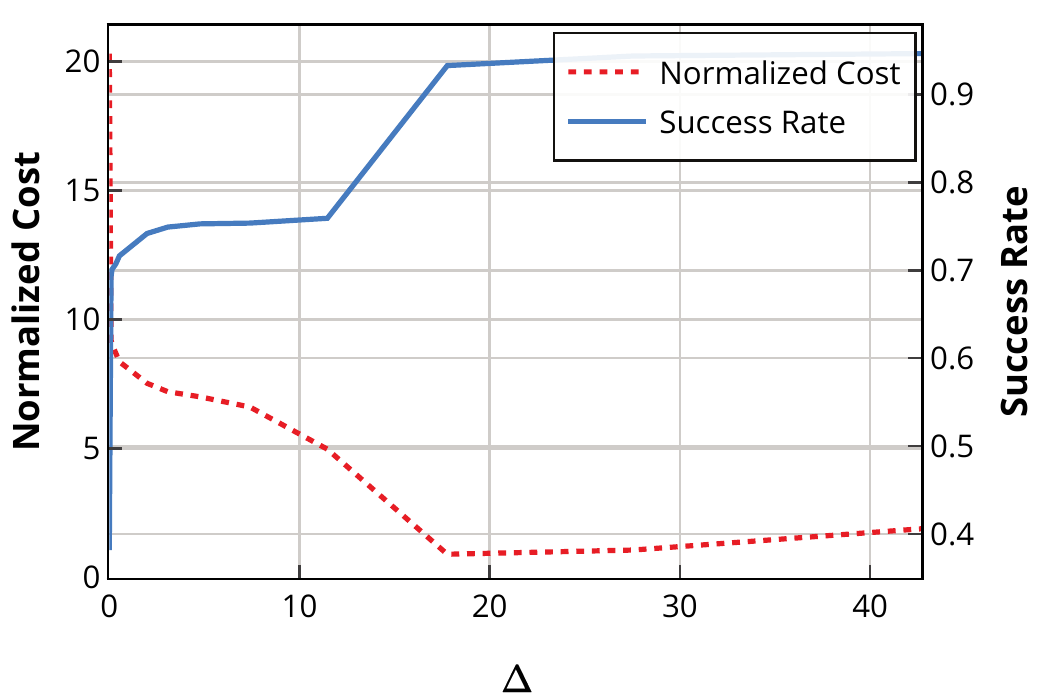}\label{fig:appendix-1d_TripleAxisParam_ClippingAttack-TripleAxisParam-norm_cum_c_t-success_rate}}
    \subfloat[\forrester]{\includegraphics[width=0.32\textwidth]{img___forrester___TripleAxisParam___ClippingAttack-TripleAxisParam-norm_cum_c_t-success_rate.pdf}\label{fig:appendix-forrester_TripleAxisParam_ClippingAttack-TripleAxisParam-norm_cum_c_t-success_rate}}
    \subfloat[\levy]{\includegraphics[width=0.32\textwidth]{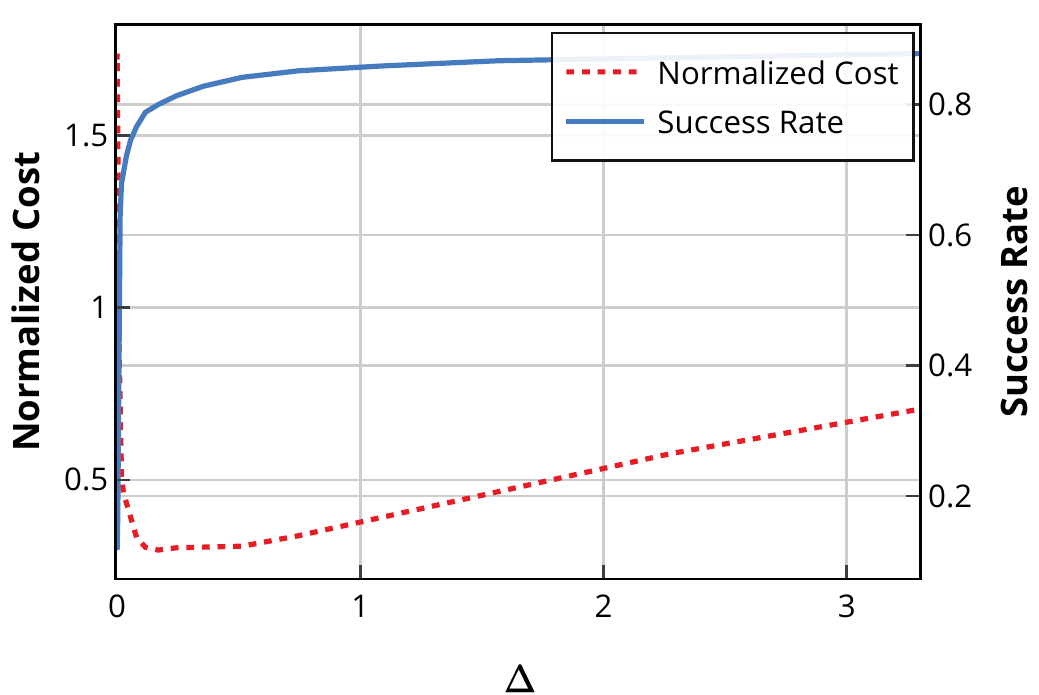}\label{fig:appendix-levy_TripleAxisParam_ClippingAttack-TripleAxisParam-norm_cum_c_t-success_rate}}\\
    \subfloat[\levyhard]{\includegraphics[width=0.32\textwidth]{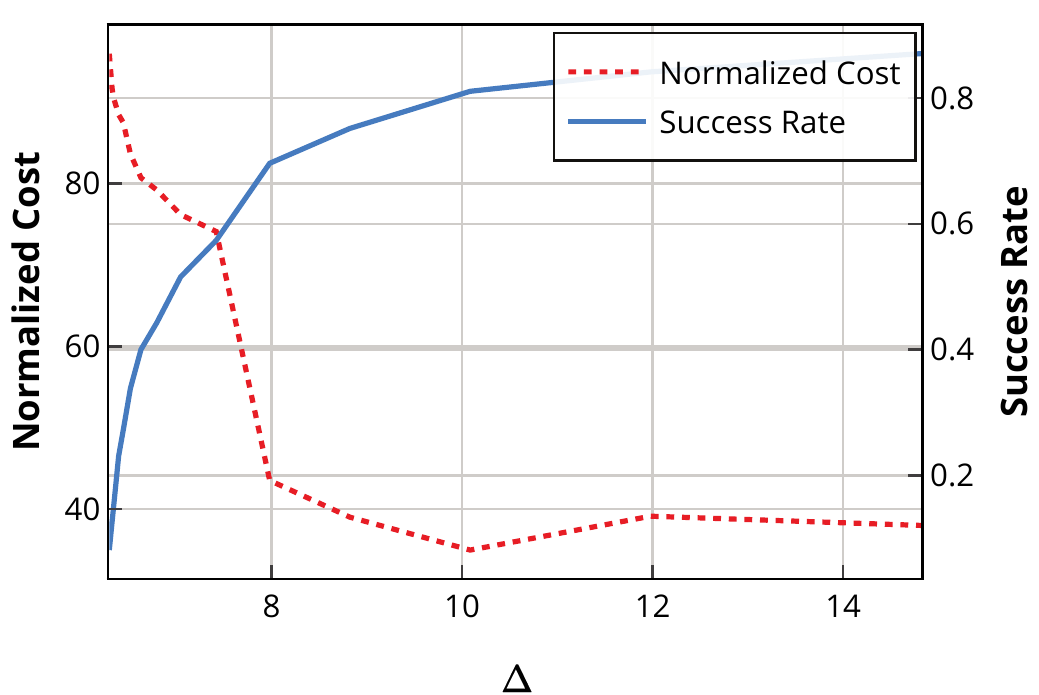}\label{fig:appendix-levy_hard_TripleAxisParam_ClippingAttack-TripleAxisParam-norm_cum_c_t-success_rate}}
    \subfloat[\bohachevsky]{\includegraphics[width=0.32\textwidth]{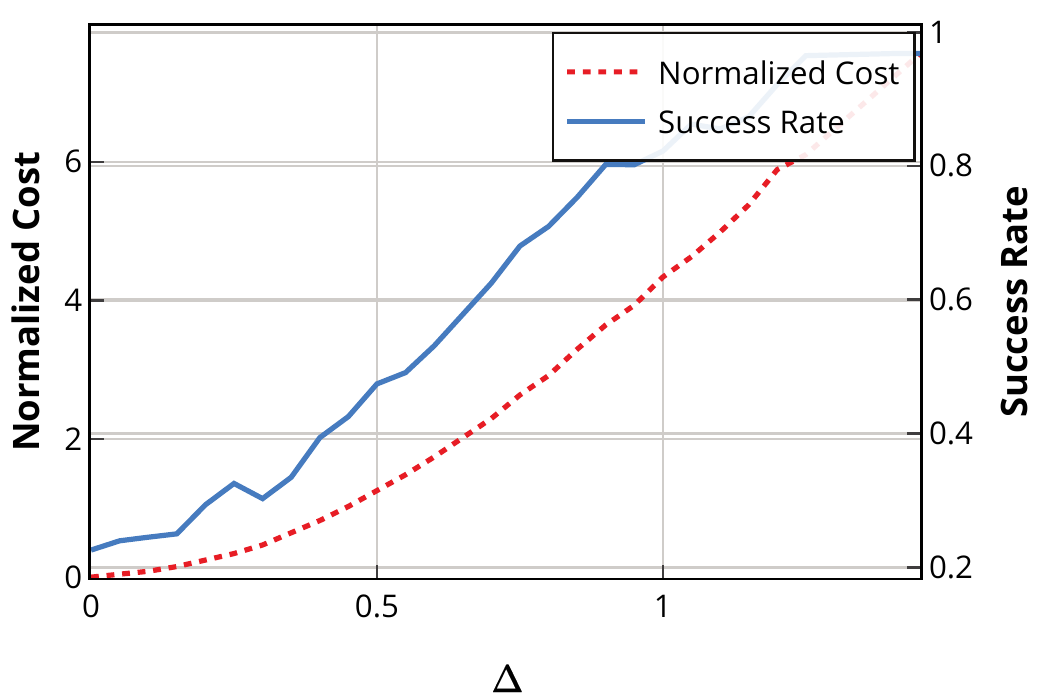}\label{fig:appendix-bohachevsky_TripleAxisParam_ClippingAttack-TripleAxisParam-norm_cum_c_t-success_rate}}
    \subfloat[\bohachevskyhard]{\includegraphics[width=0.32\textwidth]{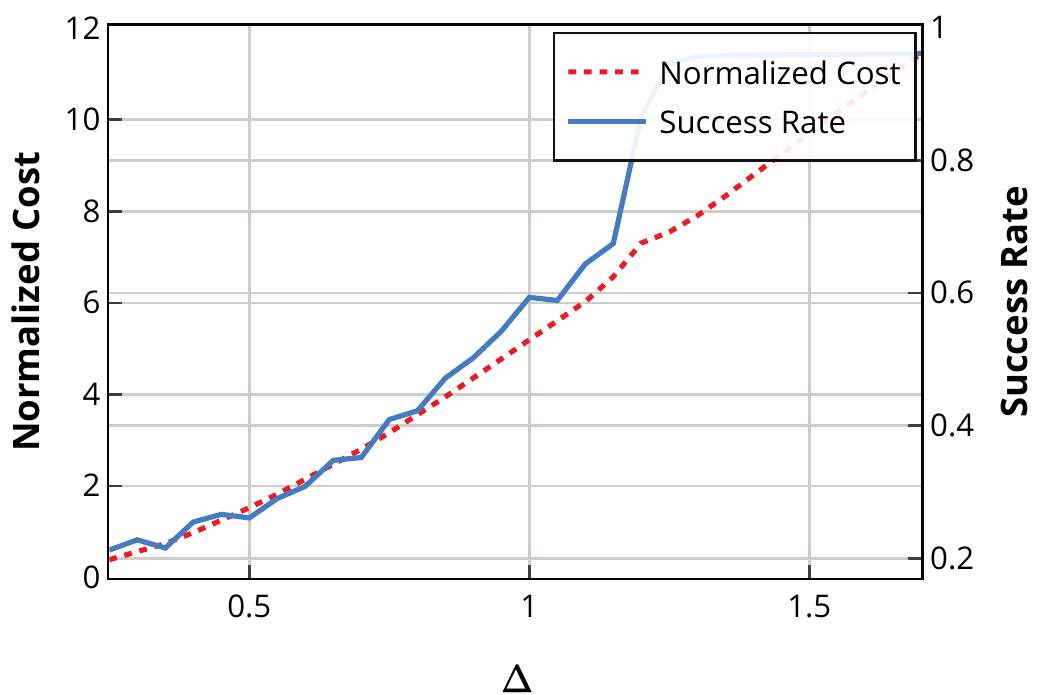}\label{fig:appendix-bohachevsky_hard_TripleAxisParam_ClippingAttack-TripleAxisParam-norm_cum_c_t-success_rate}}\\
    \subfloat[\branin]{\includegraphics[width=0.32\textwidth]{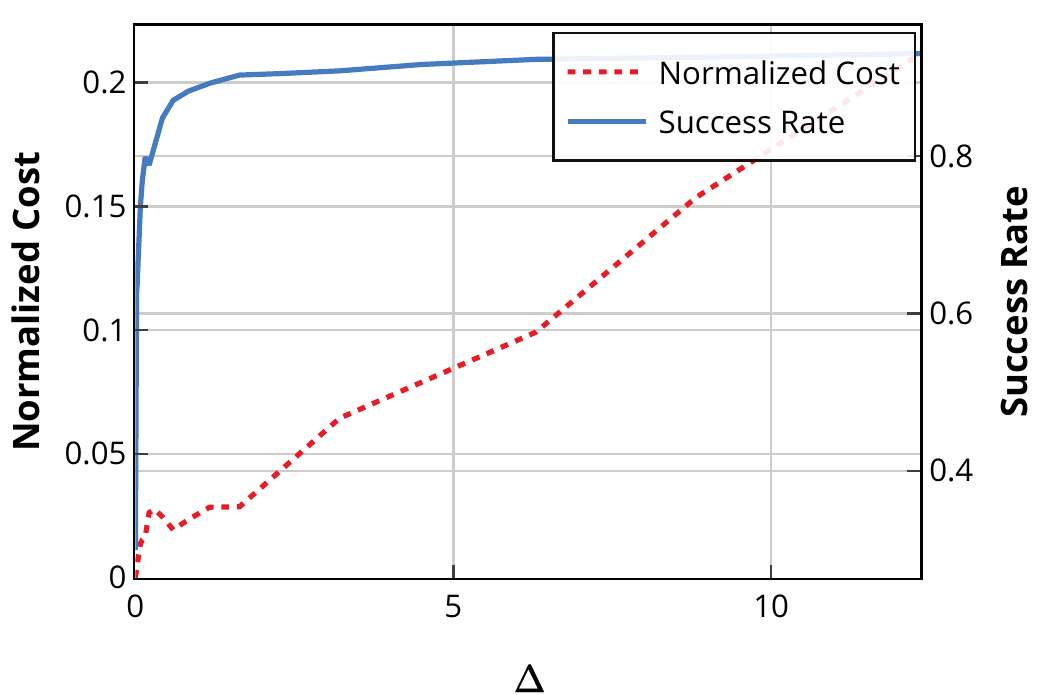}\label{fig:appendix-branin_TripleAxisParam_ClippingAttack-TripleAxisParam-norm_cum_c_t-success_rate}}
    \subfloat[\camelback]{\includegraphics[width=0.32\textwidth]{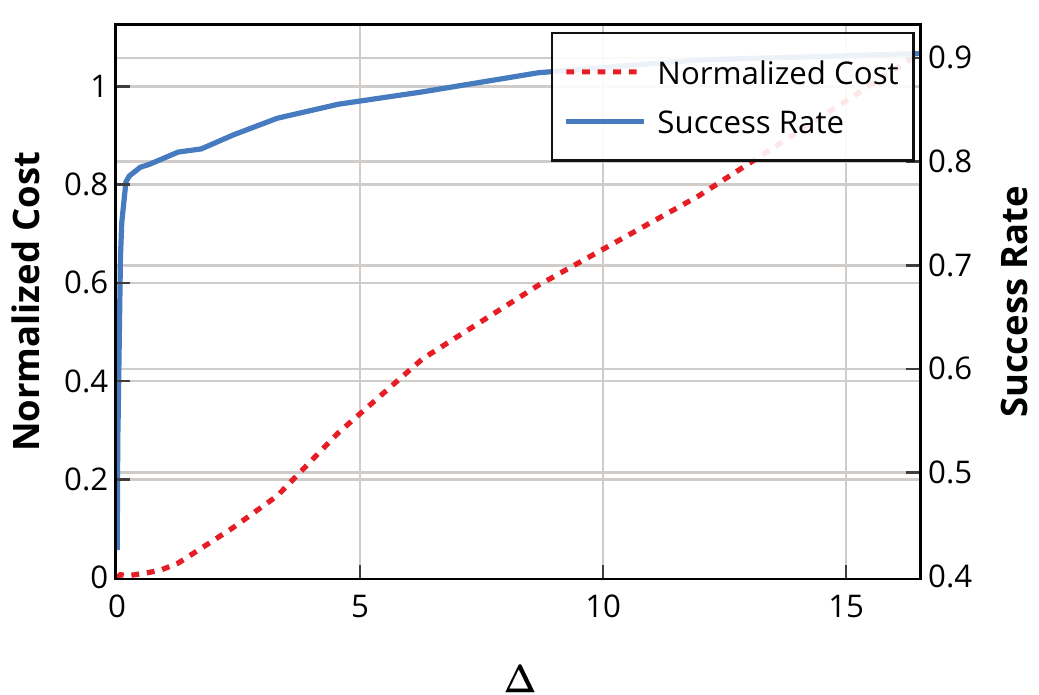}\label{fig:appendix-camelback_TripleAxisParam_ClippingAttack-TripleAxisParam-norm_cum_c_t-success_rate}}
    \subfloat[\hartmann]{\includegraphics[width=0.32\textwidth]{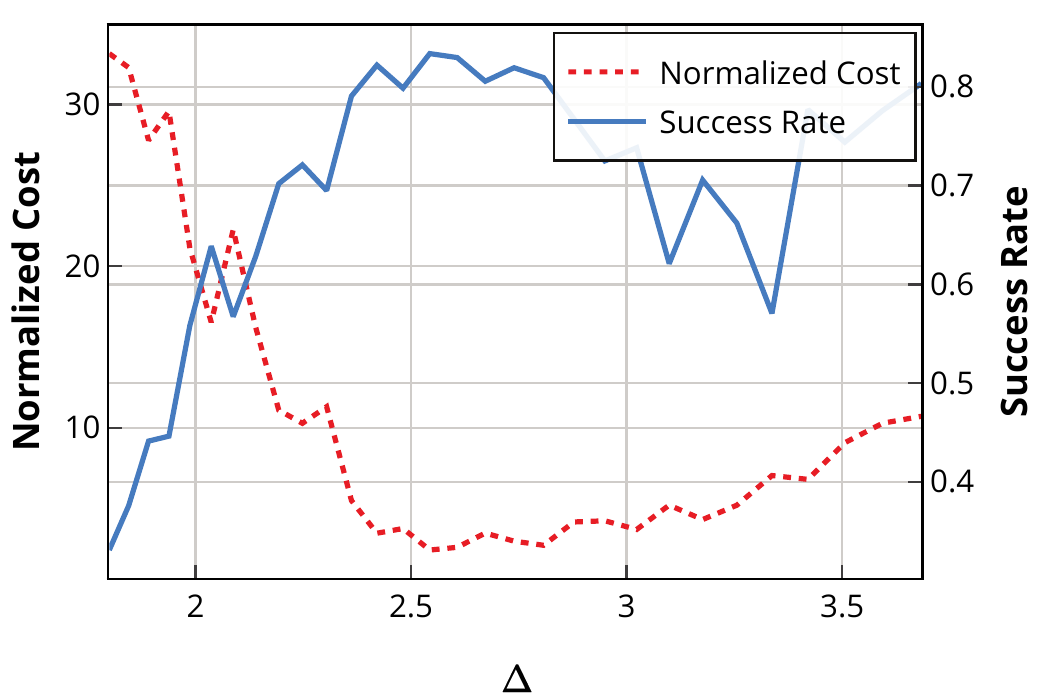}\label{fig:appendix-hartmann6_TripleAxisParam_ClippingAttack-TripleAxisParam-norm_cum_c_t-success_rate}}\\
    \subfloat[\robotsmall]{\includegraphics[width=0.32\textwidth]{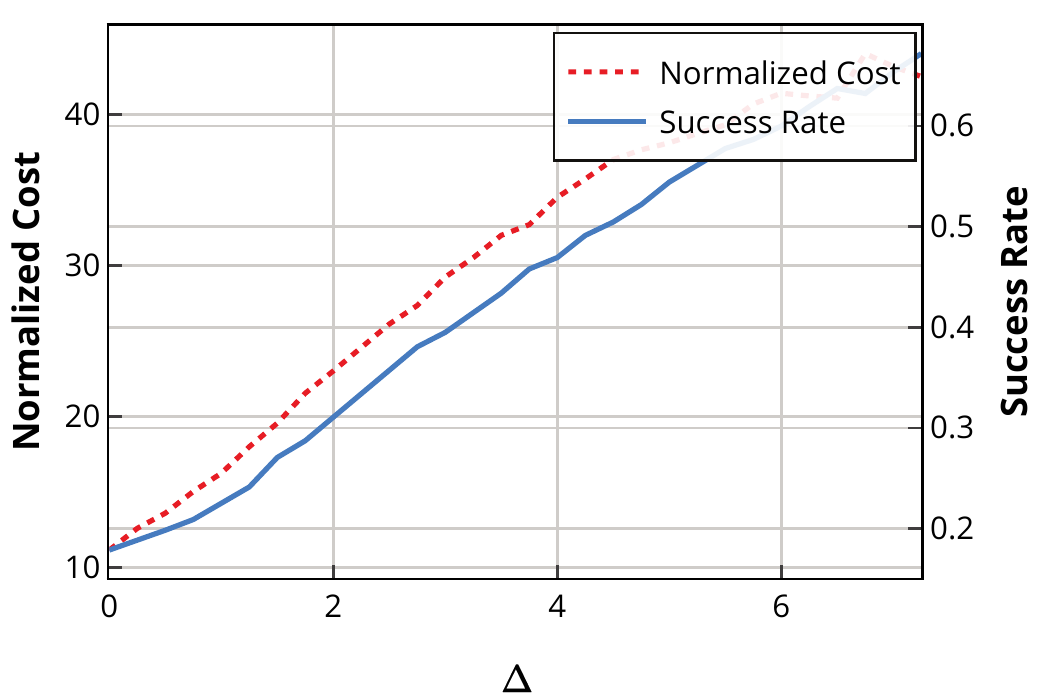}\label{fig:appendix-robot_pushing3d_TripleAxisParam_ClippingAttack-TripleAxisParam-norm_cum_c_t-success_rate}}
    \subfloat[\robotlarge]{\includegraphics[width=0.32\textwidth]{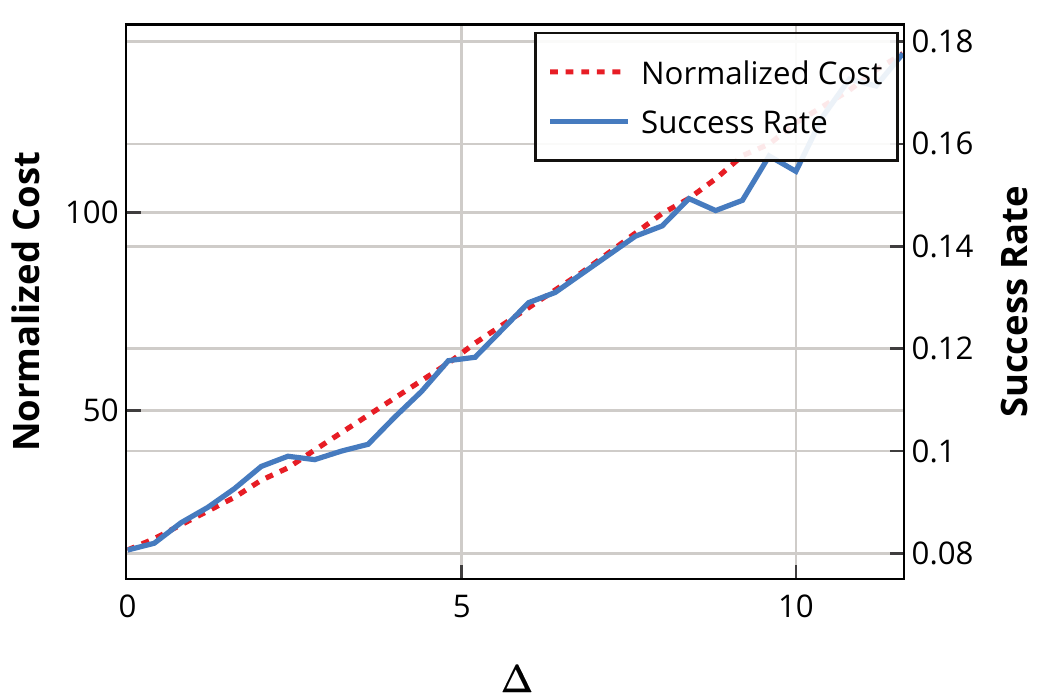}\label{fig:appendix-robot_pushing4d_TripleAxisParam_ClippingAttack-TripleAxisParam-norm_cum_c_t-success_rate}}
    \caption{Effect of \clipping's hyperparameter $\Delta$ on the success rate and cost incurred.}
    \label{fig:appendix-ClippingAttack}
\end{figure*}

\begin{figure*}[!ht]
    \centering
    \subfloat[\synthetic]{\includegraphics[width=0.32\textwidth]{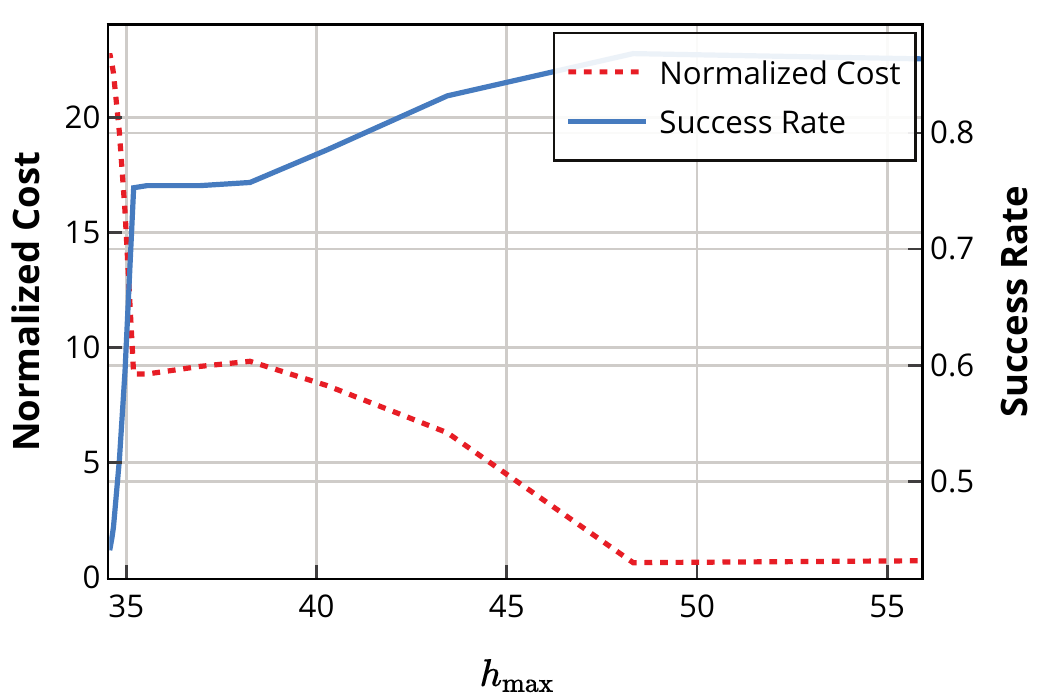}\label{fig:appendix-1d_TripleAxisParam_SubtractionAttackRnd-TripleAxisParam-norm_cum_c_t-success_rate}}
    \subfloat[\forrester]{\includegraphics[width=0.32\textwidth]{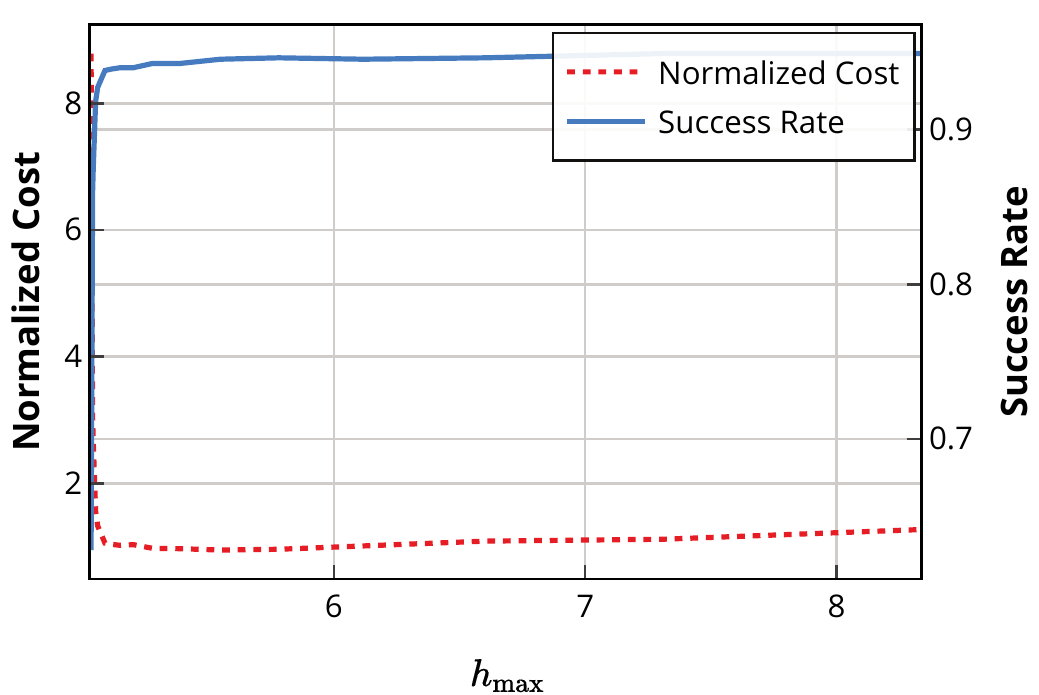}\label{fig:appendix-forrester_TripleAxisParam_SubtractionAttackRnd-TripleAxisParam-norm_cum_c_t-success_rate}}
    \subfloat[\levy]{\includegraphics[width=0.32\textwidth]{img___levy___TripleAxisParam___SubtractionAttackRnd-TripleAxisParam-norm_cum_c_t-success_rate.pdf}\label{fig:appendix-levy_TripleAxisParam_SubtractionAttackRnd-TripleAxisParam-norm_cum_c_t-success_rate}}\\
    \subfloat[\levyhard]{\includegraphics[width=0.32\textwidth]{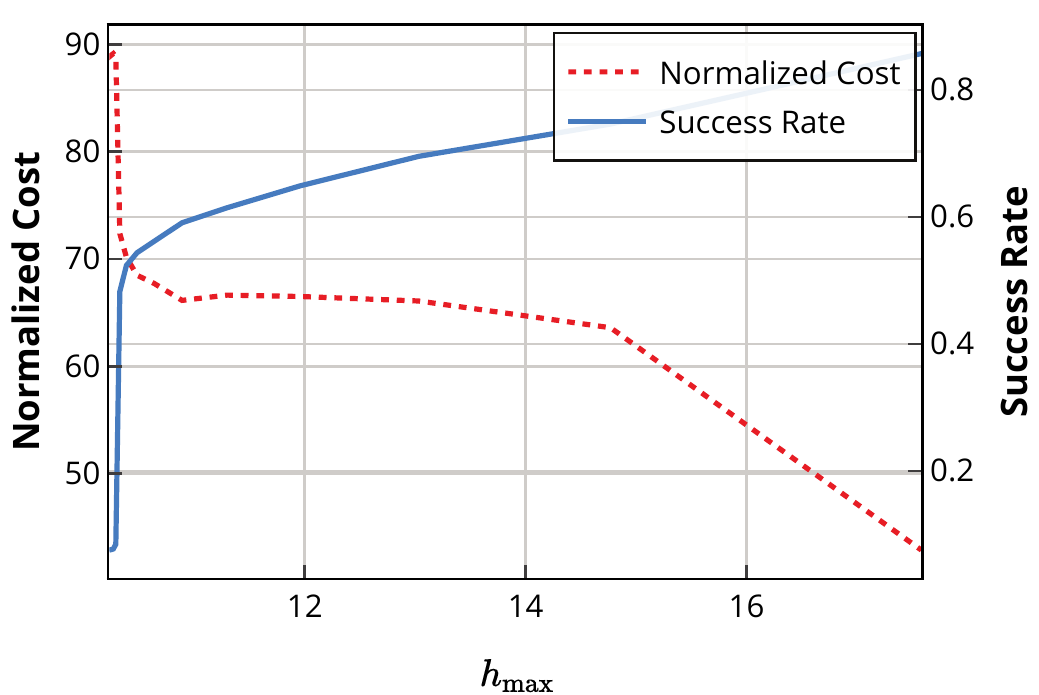}\label{fig:appendix-levy_hard_TripleAxisParam_SubtractionAttackRnd-TripleAxisParam-norm_cum_c_t-success_rate}}
    \subfloat[\bohachevsky]{\includegraphics[width=0.32\textwidth]{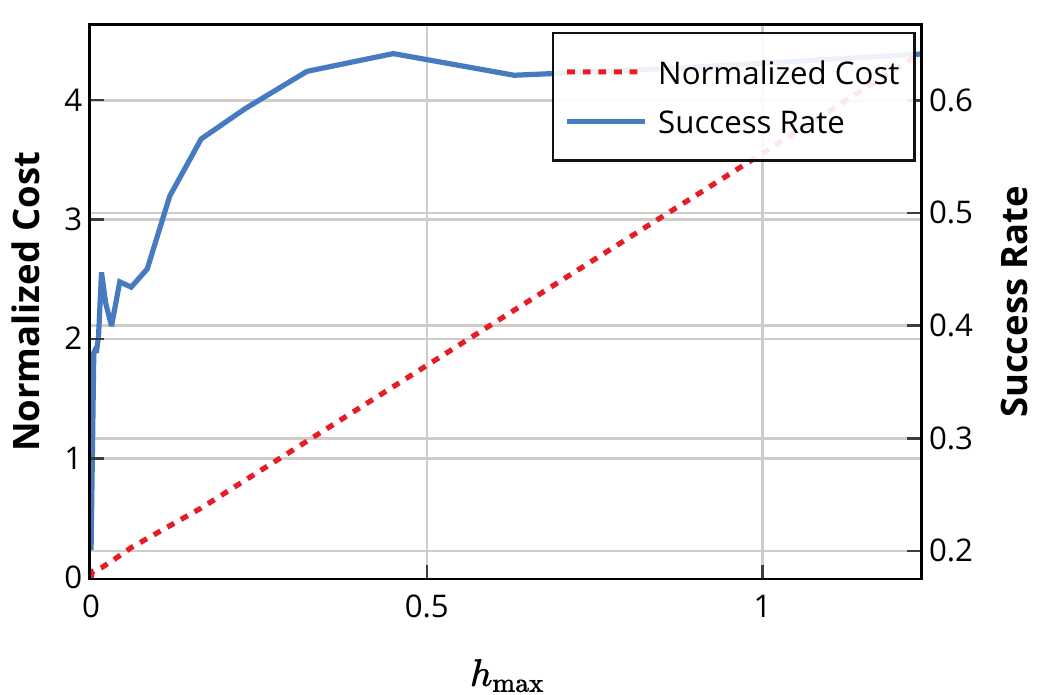}\label{fig:appendix-bohachevsky_TripleAxisParam_SubtractionAttackRnd-TripleAxisParam-norm_cum_c_t-success_rate}}
    \subfloat[\bohachevskyhard]{\includegraphics[width=0.32\textwidth]{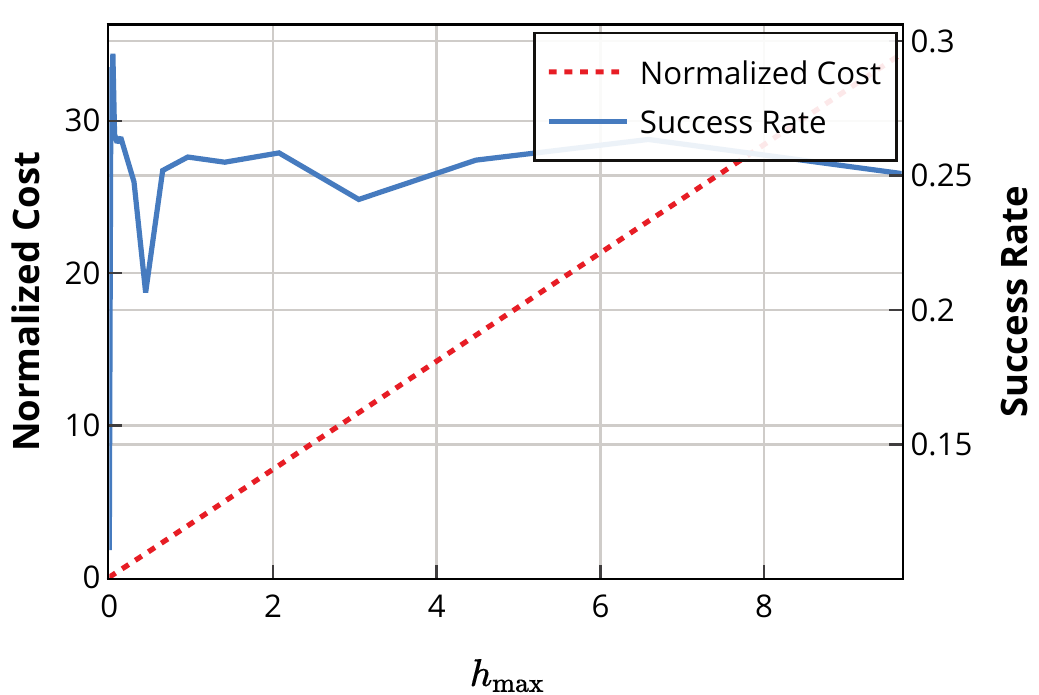}\label{fig:appendix-bohachevsky_hard_TripleAxisParam_SubtractionAttackRnd-TripleAxisParam-norm_cum_c_t-success_rate}}\\
    \subfloat[\branin]{\includegraphics[width=0.32\textwidth]{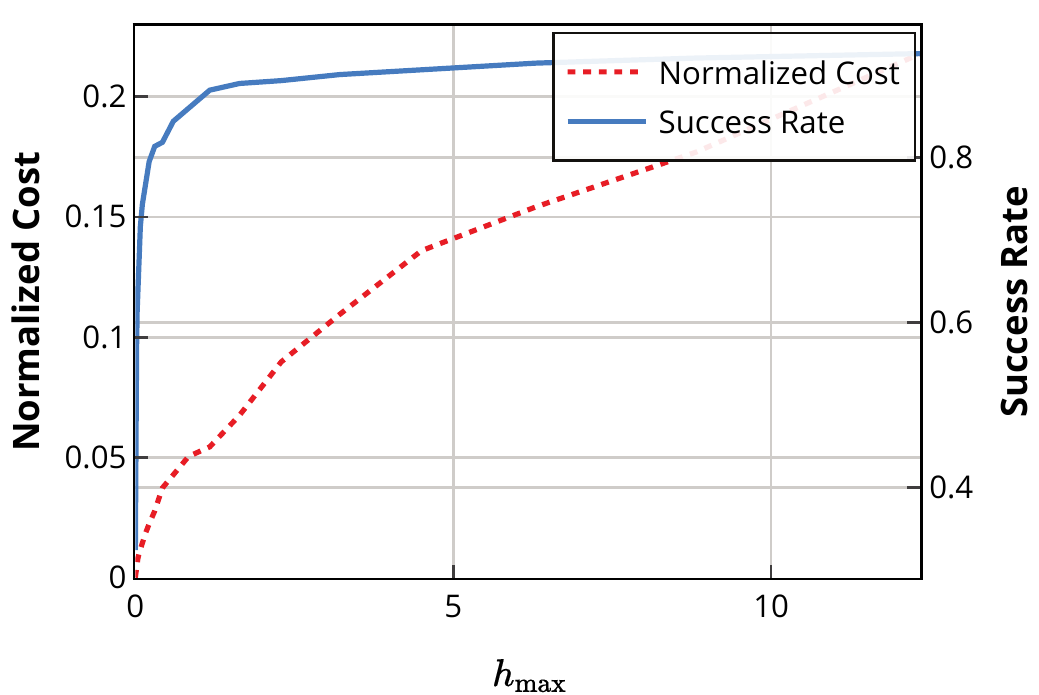}\label{fig:appendix-branin_TripleAxisParam_SubtractionAttackRnd-TripleAxisParam-norm_cum_c_t-success_rate}}
    \subfloat[\camelback]{\includegraphics[width=0.32\textwidth]{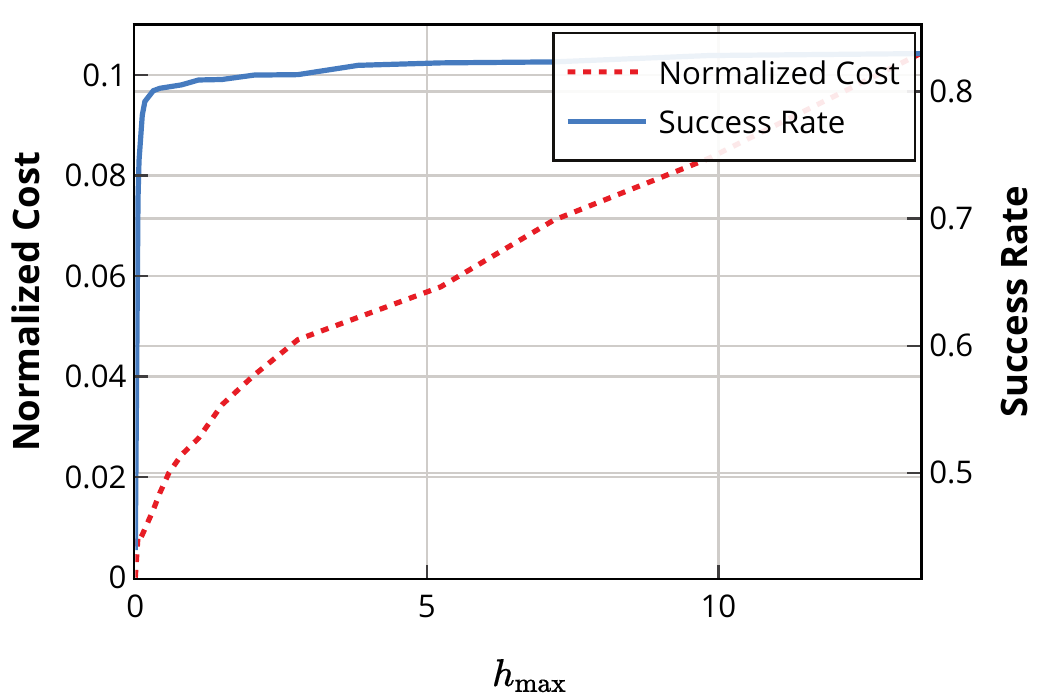}\label{fig:appendix-camelback_TripleAxisParam_SubtractionAttackRnd-TripleAxisParam-norm_cum_c_t-success_rate}}
    \caption{Effect of \subtractionrnd's hyperparameter $h_{\max}$ on the success rate and cost incurred.}
    \label{fig:appendix-SubtractionAttackRnd}
\end{figure*}

\begin{figure*}[!ht]
    \centering
    \subfloat[\synthetic]{\includegraphics[width=0.32\textwidth]{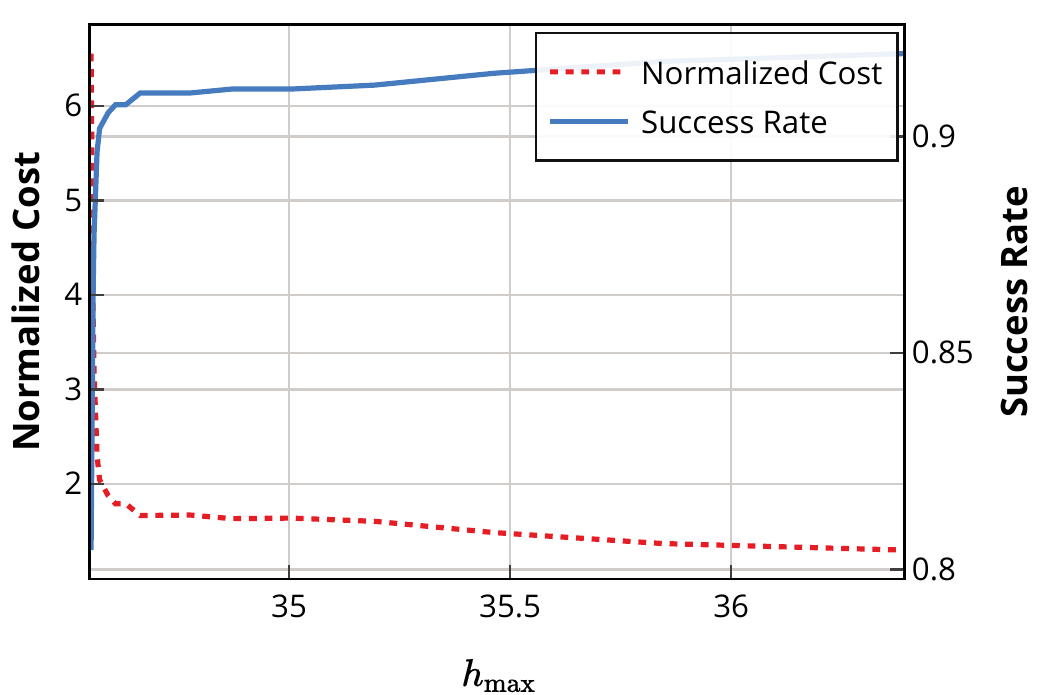}\label{fig:appendix-1d_TripleAxisParam_SubtractionAttackSq-TripleAxisParam-norm_cum_c_t-success_rate}}
    \subfloat[\forrester]{\includegraphics[width=0.32\textwidth]{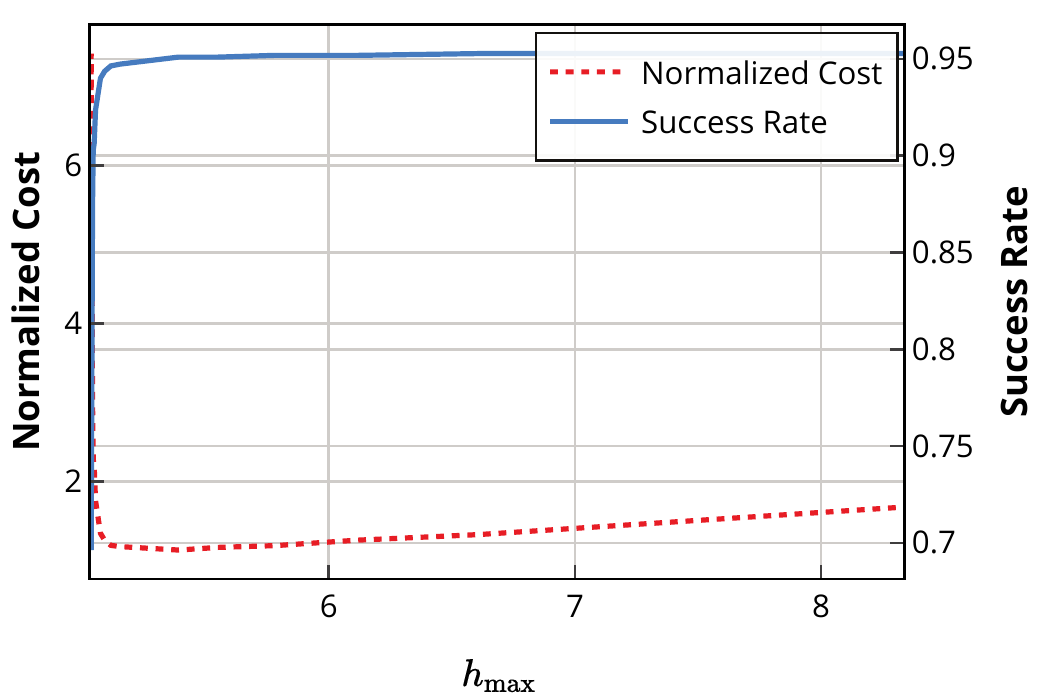}\label{fig:appendix-forrester_TripleAxisParam_SubtractionAttackSq-TripleAxisParam-norm_cum_c_t-success_rate}}
    \subfloat[\levy]{\includegraphics[width=0.32\textwidth]{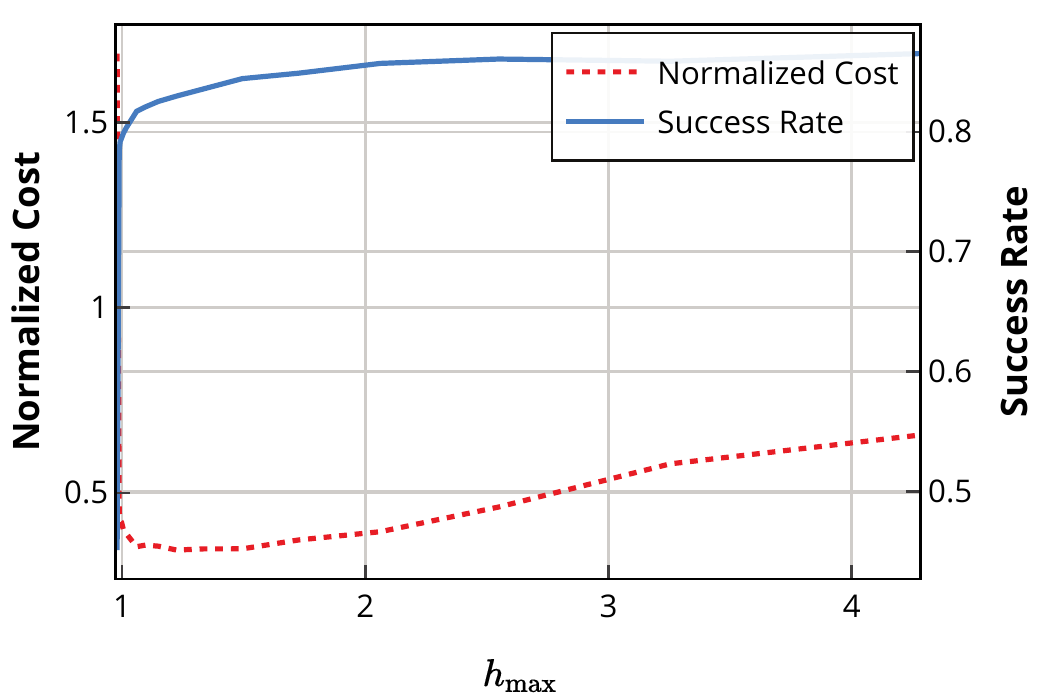}\label{fig:appendix-levy_TripleAxisParam_SubtractionAttackSq-TripleAxisParam-norm_cum_c_t-success_rate}}\\
    \subfloat[\levyhard]{\includegraphics[width=0.32\textwidth]{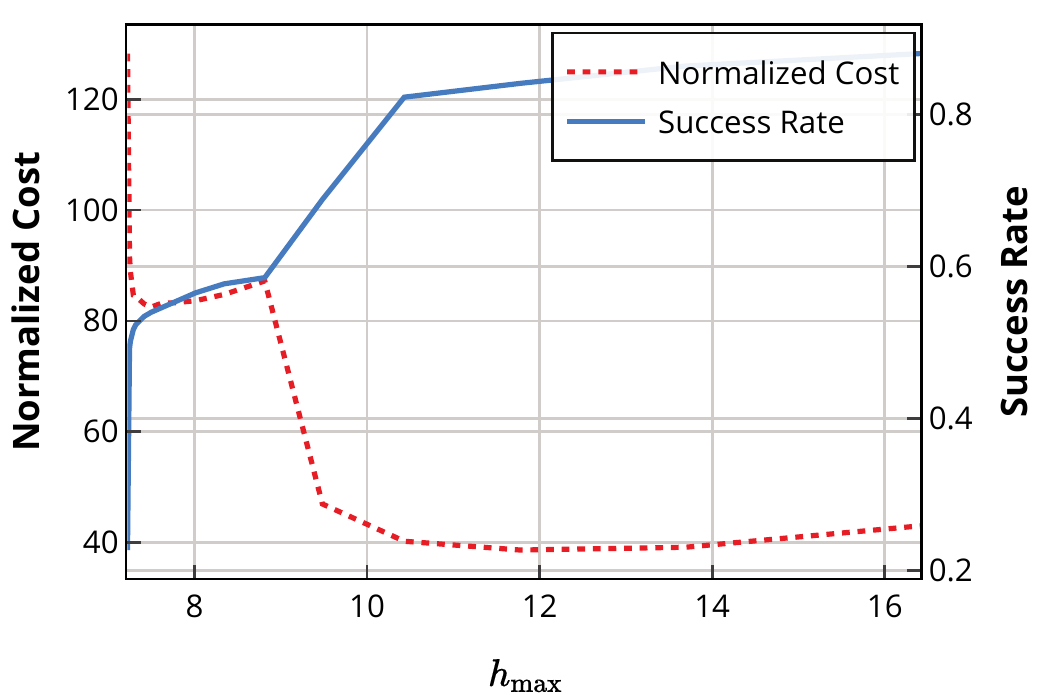}\label{fig:appendix-levy_hard_TripleAxisParam_SubtractionAttackSq-TripleAxisParam-norm_cum_c_t-success_rate}}
    \subfloat[\bohachevsky]{\includegraphics[width=0.32\textwidth]{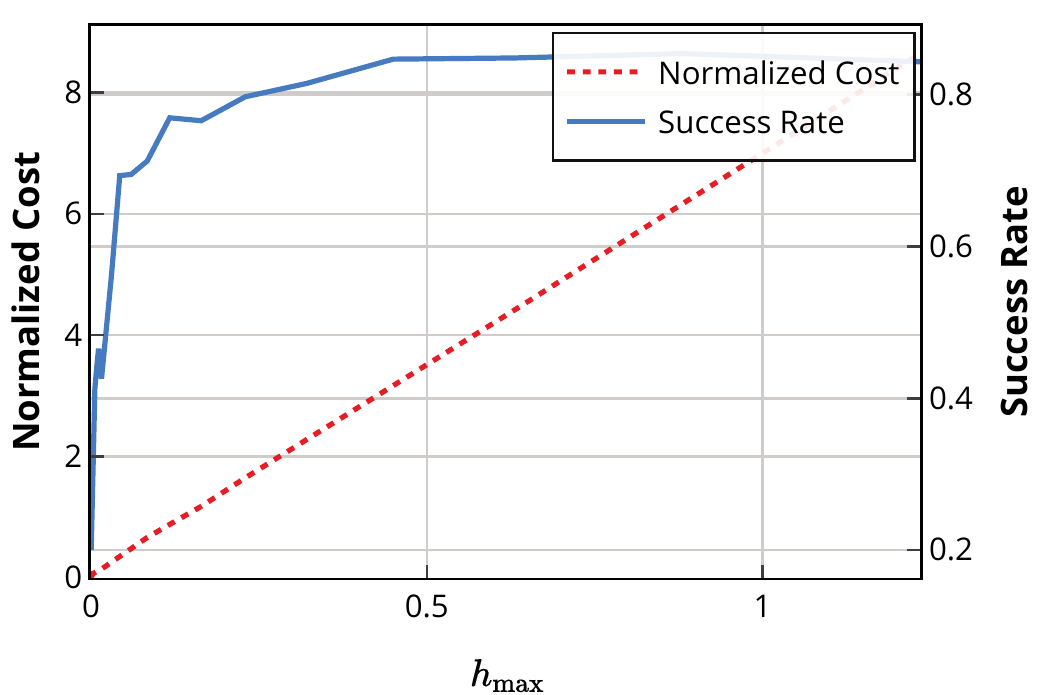}\label{fig:appendix-bohachevsky_TripleAxisParam_SubtractionAttackSq-TripleAxisParam-norm_cum_c_t-success_rate}}
    \subfloat[\bohachevskyhard]{\includegraphics[width=0.32\textwidth]{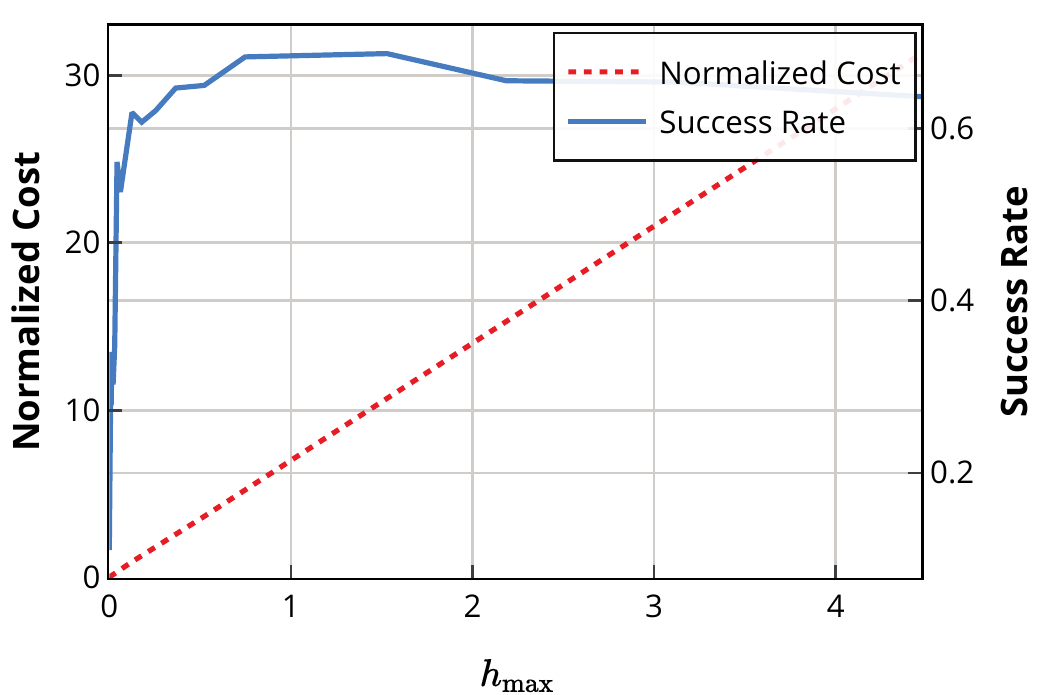}\label{fig:appendix-bohachevsky_hard_TripleAxisParam_SubtractionAttackSq-TripleAxisParam-norm_cum_c_t-success_rate}}\\
    \subfloat[\branin]{\includegraphics[width=0.32\textwidth]{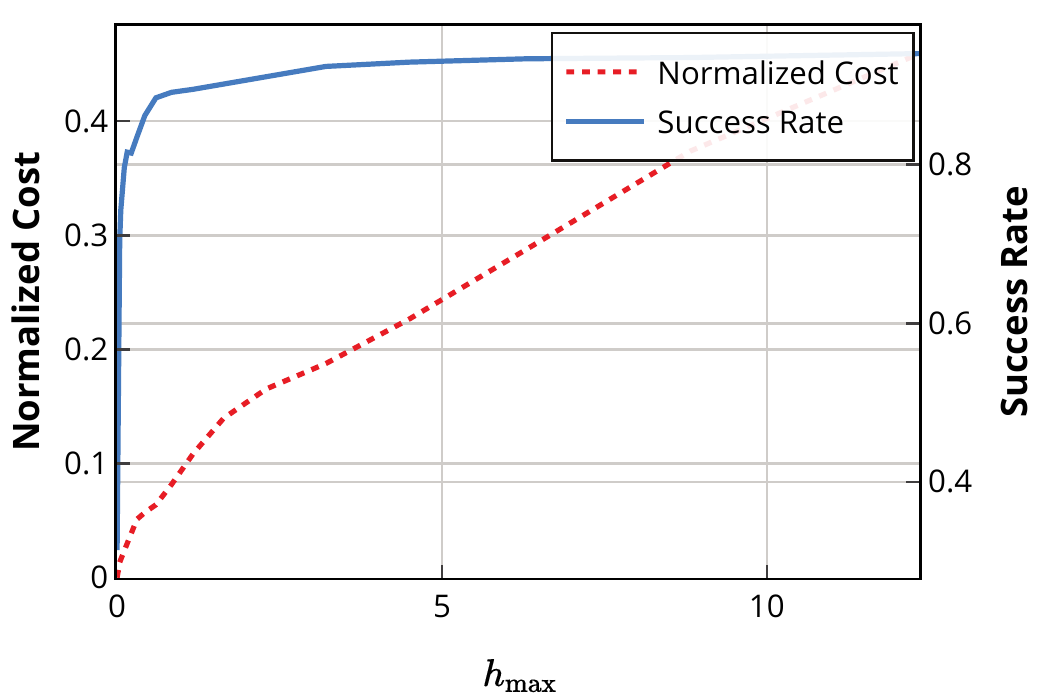}\label{fig:appendix-branin_TripleAxisParam_SubtractionAttackSq-TripleAxisParam-norm_cum_c_t-success_rate}}
    \subfloat[\camelback]{\includegraphics[width=0.32\textwidth]{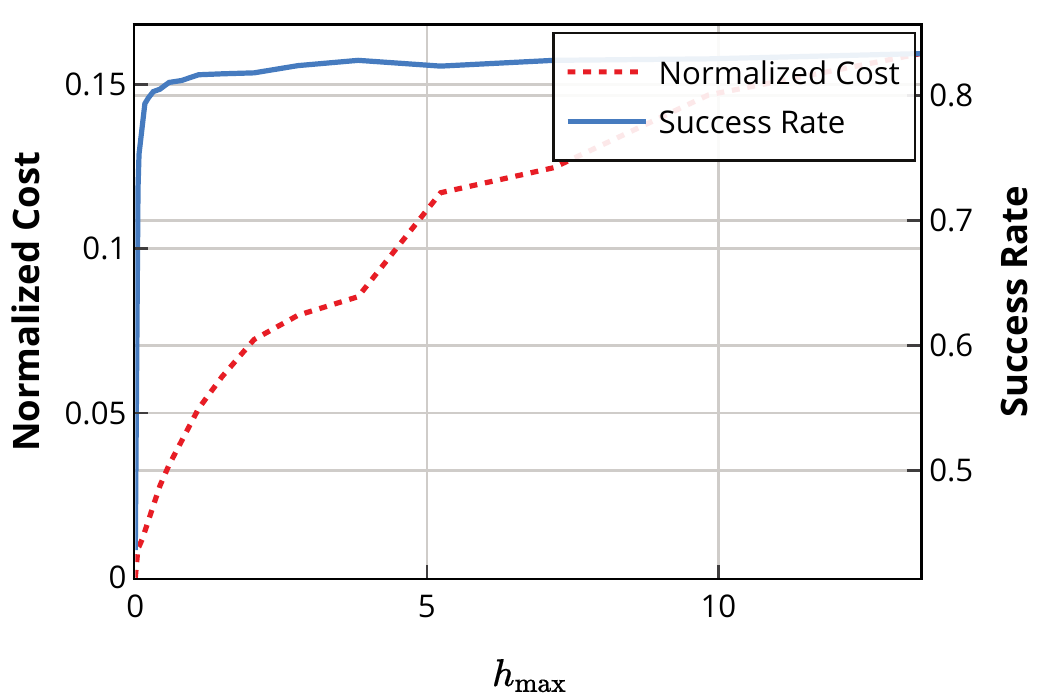}\label{fig:appendix-camelback_TripleAxisParam_SubtractionAttackSq-TripleAxisParam-norm_cum_c_t-success_rate}}
    \caption{\subtractionsq}
    \caption{Effect of \subtractionsq's hyperparameter $h_{\max}$ on the success rate and cost incurred.}
    \label{fig:appendix-SubtractionAttackSq}
\end{figure*}

\begin{figure*}[!ht]
    \centering
    \includegraphics[height=0.25cm]{img___1d___Scatter___legend-Scatter-success_rate-norm_cum_c_t.pdf} \\
    \subfloat[\synthetic]{\includegraphics[width=0.32\textwidth]{img___1d___Scatter___overall_avg_log_log_rev-Scatter-success_rate-norm_cum_c_t.pdf}\label{fig:appendix-1d_Scatter_overall_avg_log_log_rev-Scatter-success_rate-norm_cum_c_t}}
    \subfloat[\forrester]{\includegraphics[width=0.32\textwidth]{img___forrester___Scatter___overall_avg_log_log_rev-Scatter-success_rate-norm_cum_c_t.pdf}\label{fig:appendix-forrester_Scatter_overall_avg_log_log_rev-Scatter-success_rate-norm_cum_c_t}}
    \subfloat[\levy]{\includegraphics[width=0.32\textwidth]{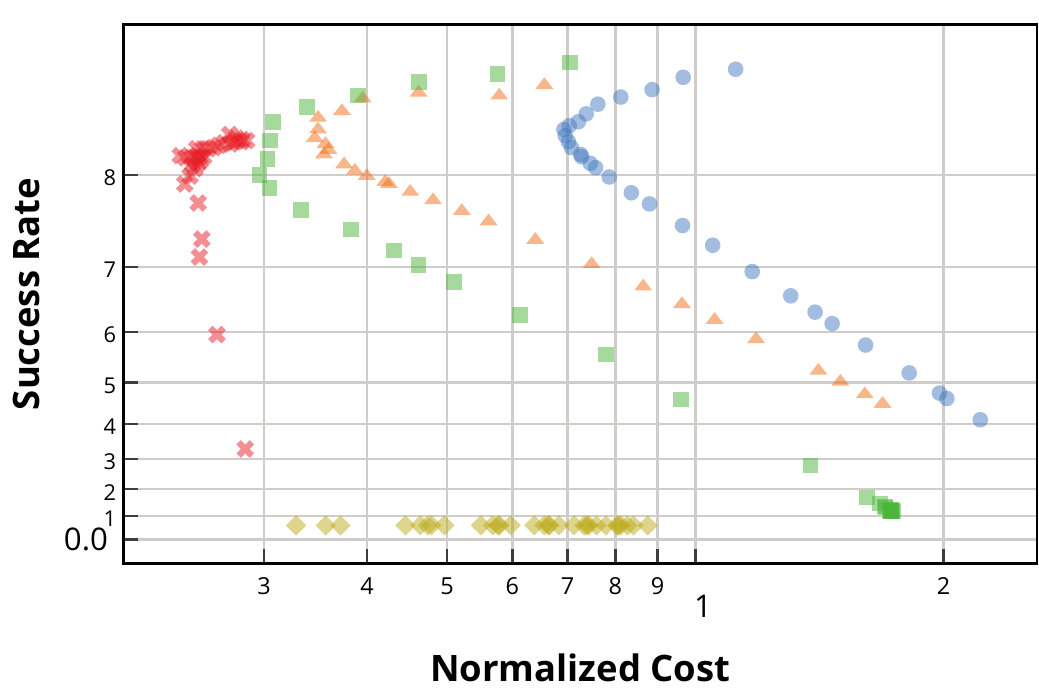}\label{fig:appendix-levy_Scatter_overall_avg_log_log_rev-Scatter-success_rate-norm_cum_c_t}}\\
    \subfloat[\levyhard]{\includegraphics[width=0.32\textwidth]{img___levy_hard___Scatter___overall_avg_log_log_rev-Scatter-success_rate-norm_cum_c_t.pdf}\label{fig:appendix-levy_hard_Scatter_overall_avg_log_log_rev-Scatter-success_rate-norm_cum_c_t}}
    \subfloat[\bohachevsky]{\includegraphics[width=0.32\textwidth]{img___bohachevsky___Scatter___overall_avg_log_log_rev-Scatter-success_rate-norm_cum_c_t.pdf}\label{fig:appendix-bohachevsky_Scatter_overall_avg_log_log_rev-Scatter-success_rate-norm_cum_c_t}}
    \subfloat[\bohachevskyhard]{\includegraphics[width=0.32\textwidth]{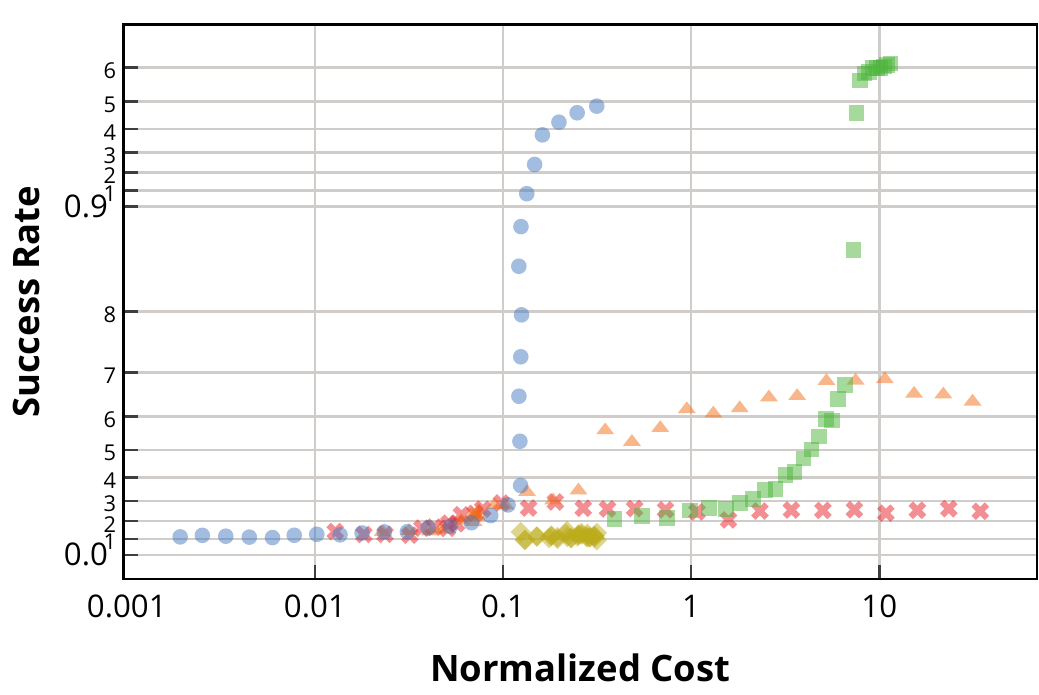}\label{fig:appendix-bohachevsky_hard_Scatter_overall_avg_log_log_rev-Scatter-success_rate-norm_cum_c_t}}\\
    \subfloat[\branin]{\includegraphics[width=0.32\textwidth]{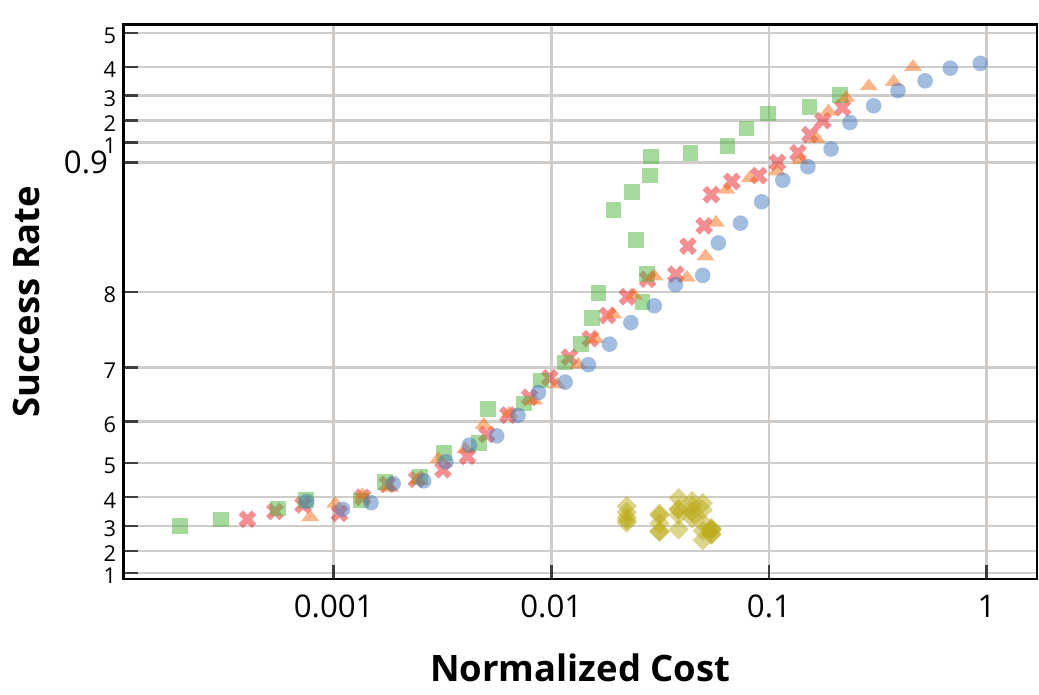}\label{fig:appendix-branin_Scatter_overall_avg_log_log_rev-Scatter-success_rate-norm_cum_c_t}}
    \subfloat[\camelback]{\includegraphics[width=0.32\textwidth]{img___camelback___Scatter___overall_avg_log_log_rev-Scatter-success_rate-norm_cum_c_t.pdf}\label{fig:appendix-camelback_Scatter_overall_avg_log_log_rev-Scatter-success_rate-norm_cum_c_t}}
    \subfloat[\hartmann]{\includegraphics[width=0.32\textwidth]{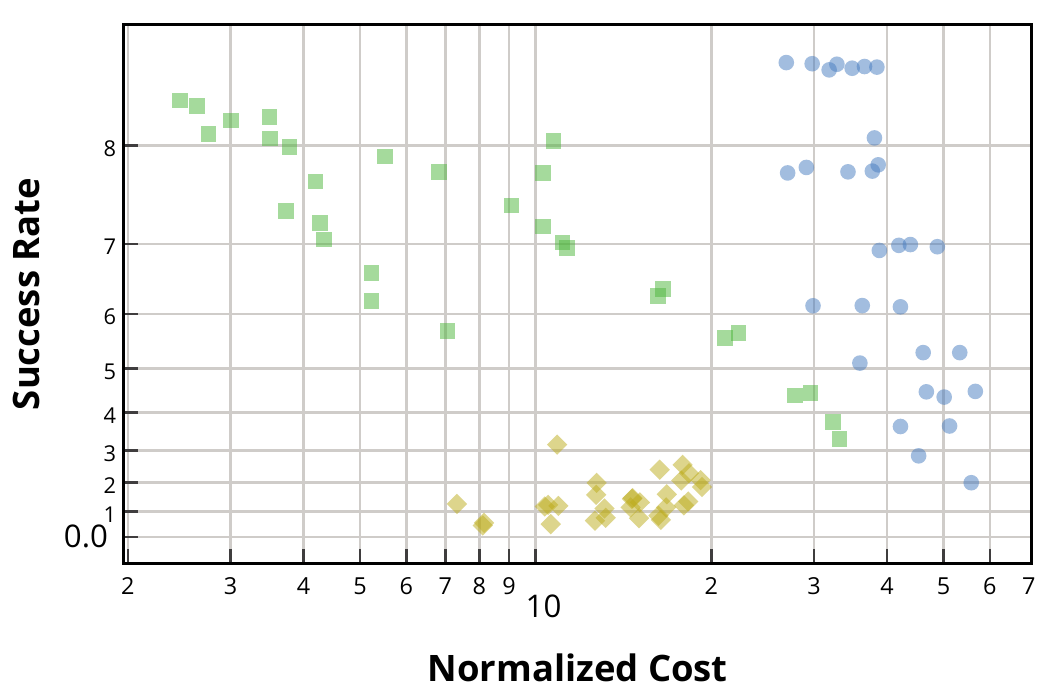}\label{fig:appendix-hartmann6_Scatter_overall_avg_log_log_rev-Scatter-success_rate-norm_cum_c_t}}\\
    \subfloat[\robotsmall]{\includegraphics[width=0.32\textwidth]{img___robot_pushing3d___Scatter___overall_avg_log_log_rev-Scatter-success_rate-norm_cum_c_t.pdf}\label{fig:appendix-robot_pushing3d_Scatter_overall_avg_log_log_rev-Scatter-success_rate-norm_cum_c_t}}
    \subfloat[\robotlarge]{\includegraphics[width=0.32\textwidth]{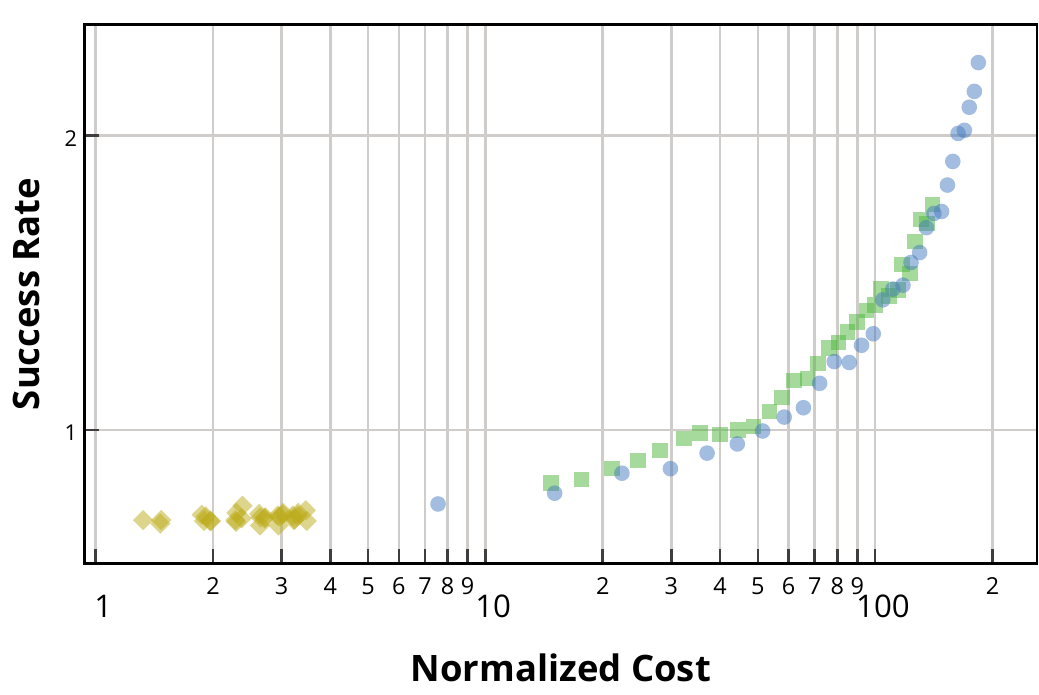}\label{fig:appendix-robot_pushing4d_Scatter_overall_avg_log_log_rev-Scatter-success_rate-norm_cum_c_t}}
    \caption{Success rate vs.~cost averaged over random seeds, each point is the performance of a particular hyperparameter.}
    \label{fig:appendix-SuccessRate_Cost_Avg}
\end{figure*}

\begin{figure*}[!ht]
    \centering
    \includegraphics[height=0.25cm]{img___1d___Scatter___legend-Scatter-success_rate-norm_cum_c_t.pdf} \\
    \subfloat[\synthetic]{\includegraphics[width=0.32\textwidth]{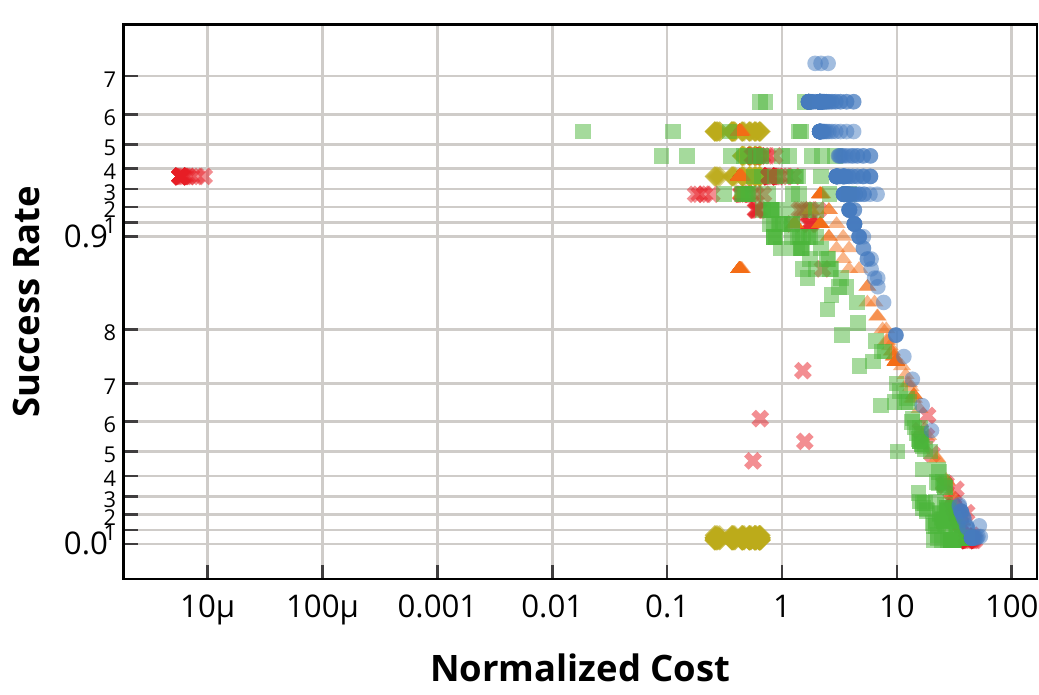}\label{fig:appendix-1d_Scatter_overall_log_log_rev-Scatter-success_rate-norm_cum_c_t}}
    \subfloat[\forrester]{\includegraphics[width=0.32\textwidth]{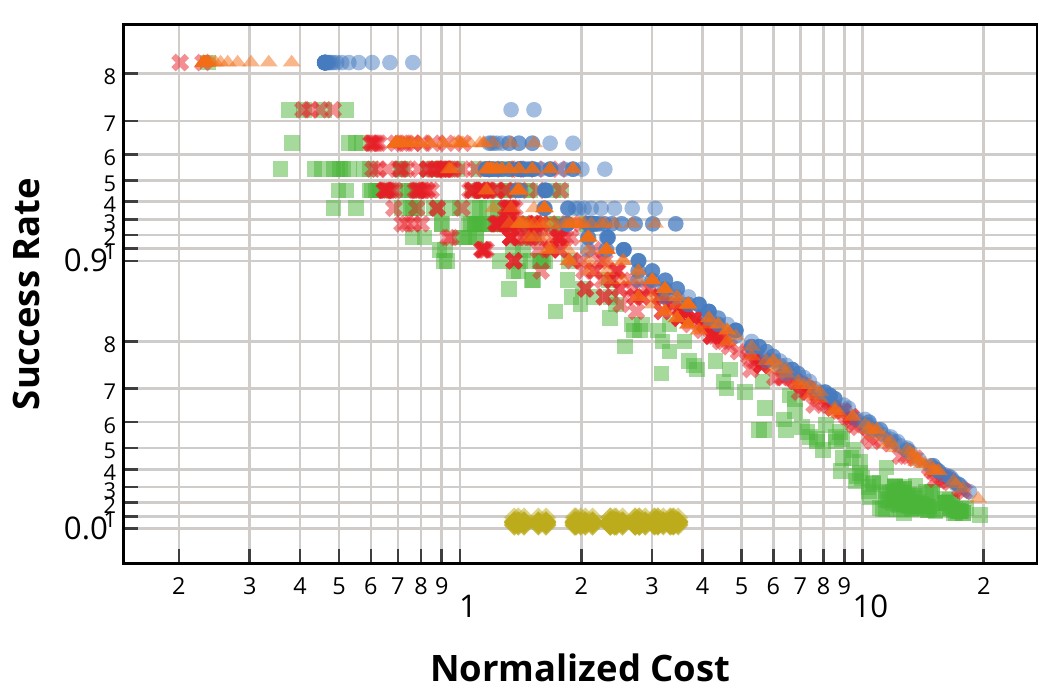}\label{fig:appendix-forrester_Scatter_overall_log_log_rev-Scatter-success_rate-norm_cum_c_t}}
    \subfloat[\levy]{\includegraphics[width=0.32\textwidth]{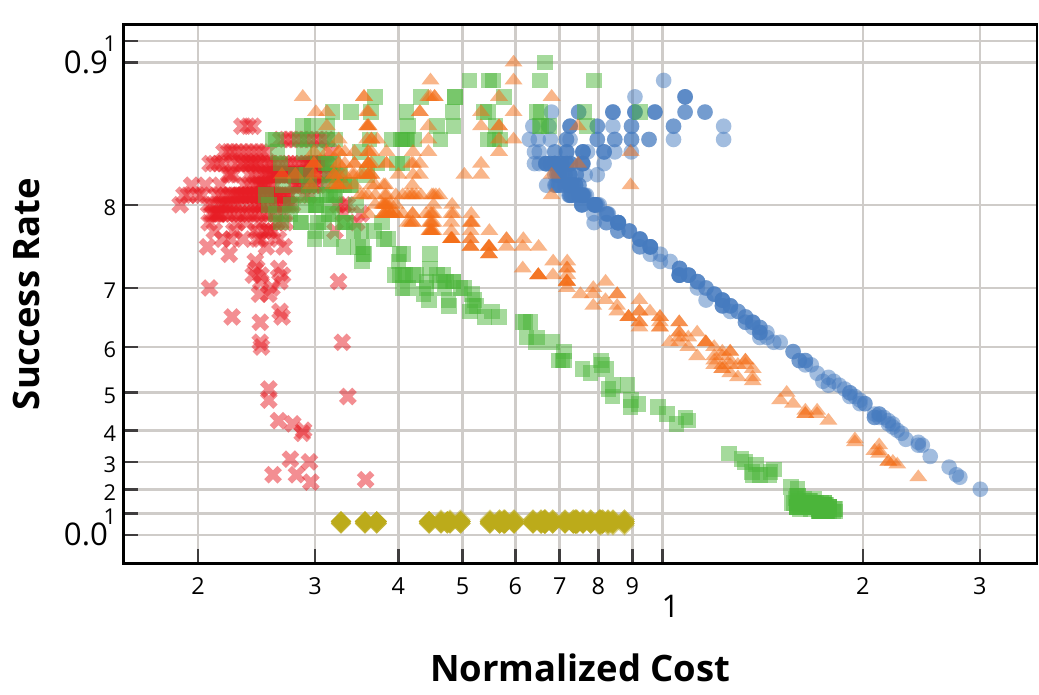}\label{fig:appendix-levy_Scatter_overall_log_log_rev-Scatter-success_rate-norm_cum_c_t}}\\
    \subfloat[\levyhard]{\includegraphics[width=0.32\textwidth]{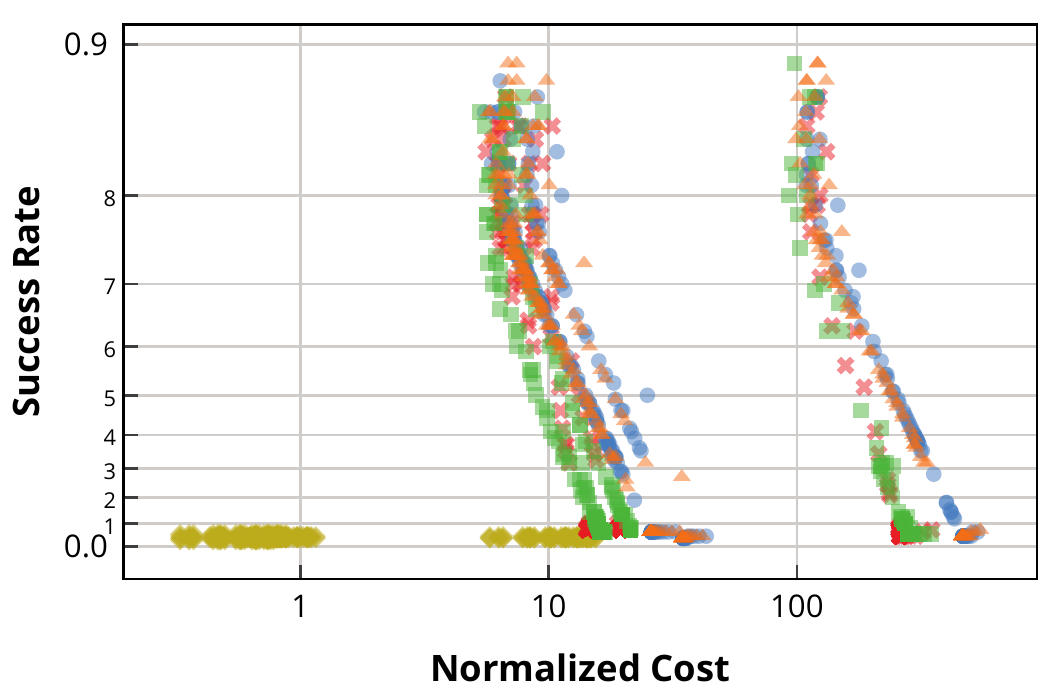}\label{fig:appendix-levy_hard_Scatter_overall_log_log_rev-Scatter-success_rate-norm_cum_c_t}}
    \subfloat[\bohachevsky]{\includegraphics[width=0.32\textwidth]{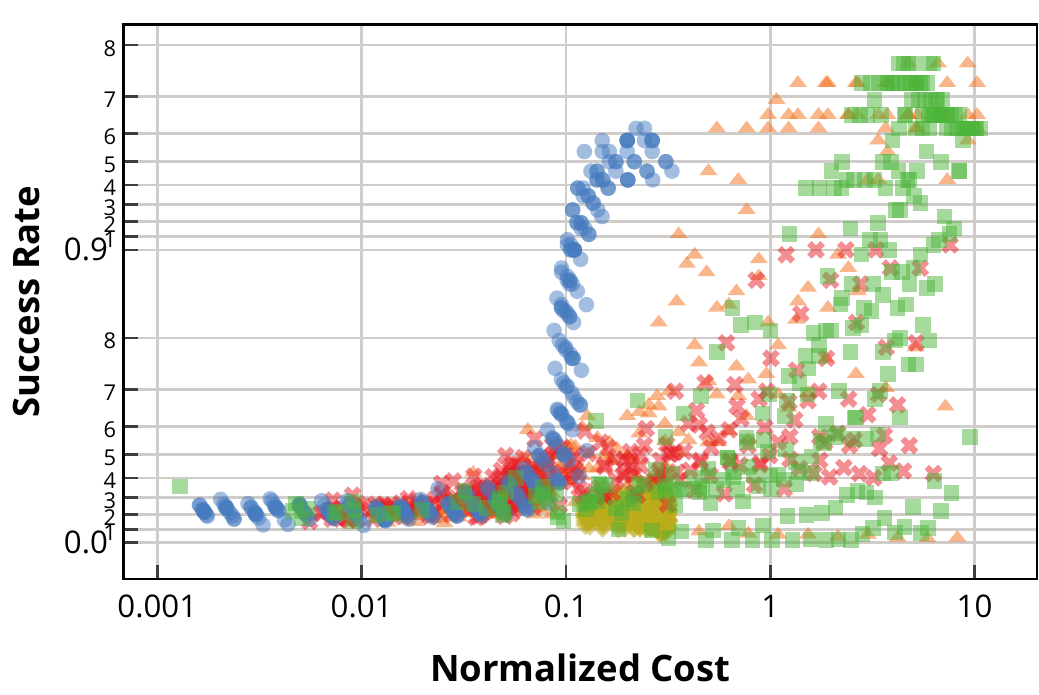}\label{fig:appendix-bohachevsky_Scatter_overall_log_log_rev-Scatter-success_rate-norm_cum_c_t}}
    \subfloat[\bohachevskyhard]{\includegraphics[width=0.32\textwidth]{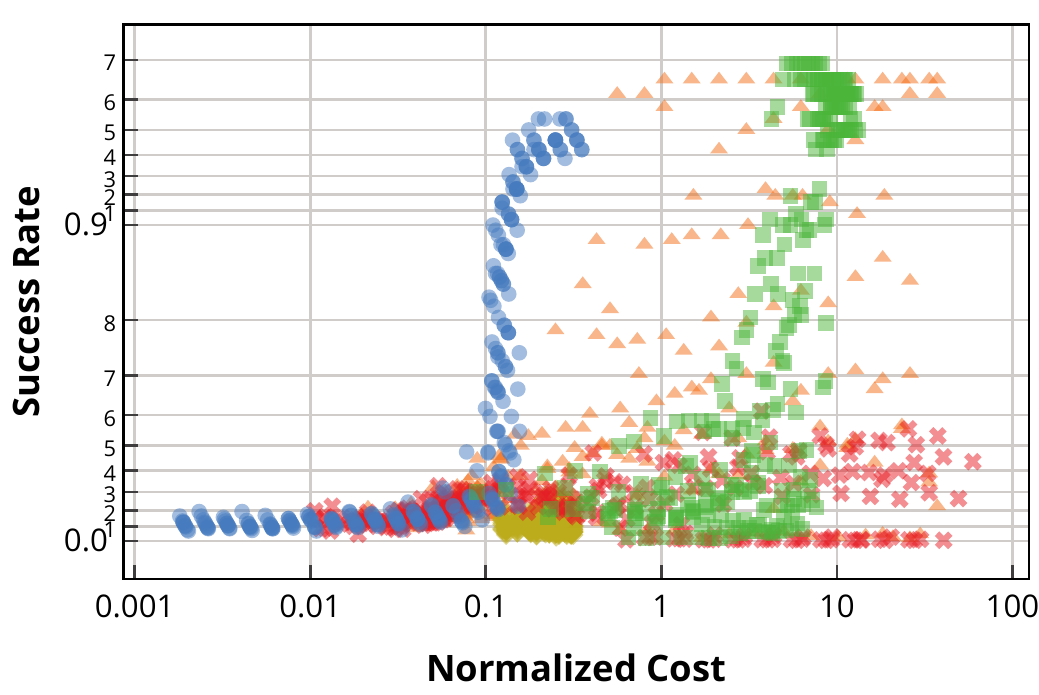}\label{fig:appendix-bohachevsky_hard_Scatter_overall_log_log_rev-Scatter-success_rate-norm_cum_c_t}}\\
    \subfloat[\branin]{\includegraphics[width=0.32\textwidth]{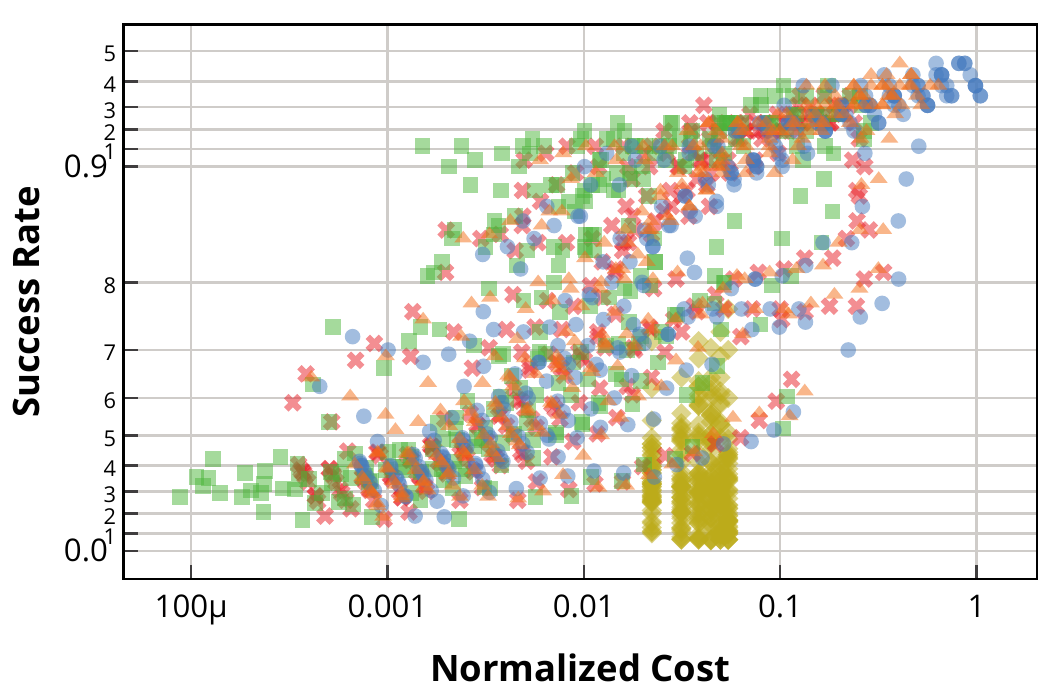}\label{fig:appendix-branin_Scatter_overall_log_log_rev-Scatter-success_rate-norm_cum_c_t}}
    \subfloat[\camelback]{\includegraphics[width=0.32\textwidth]{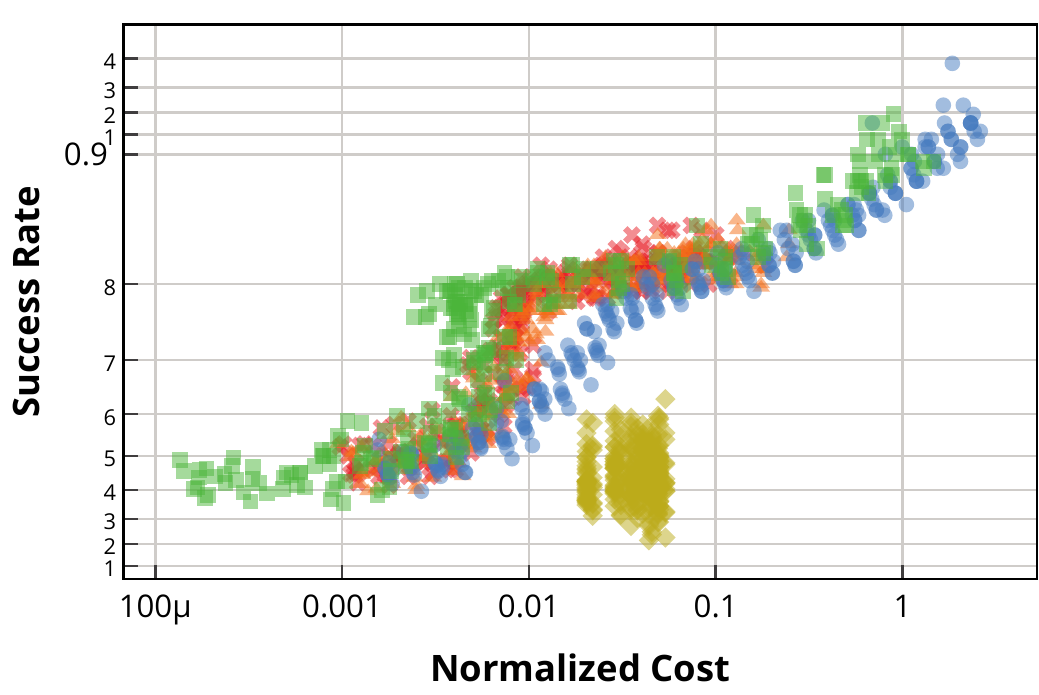}\label{fig:appendix-camelback_Scatter_overall_log_log_rev-Scatter-success_rate-norm_cum_c_t}}
    \subfloat[\hartmann]{\includegraphics[width=0.32\textwidth]{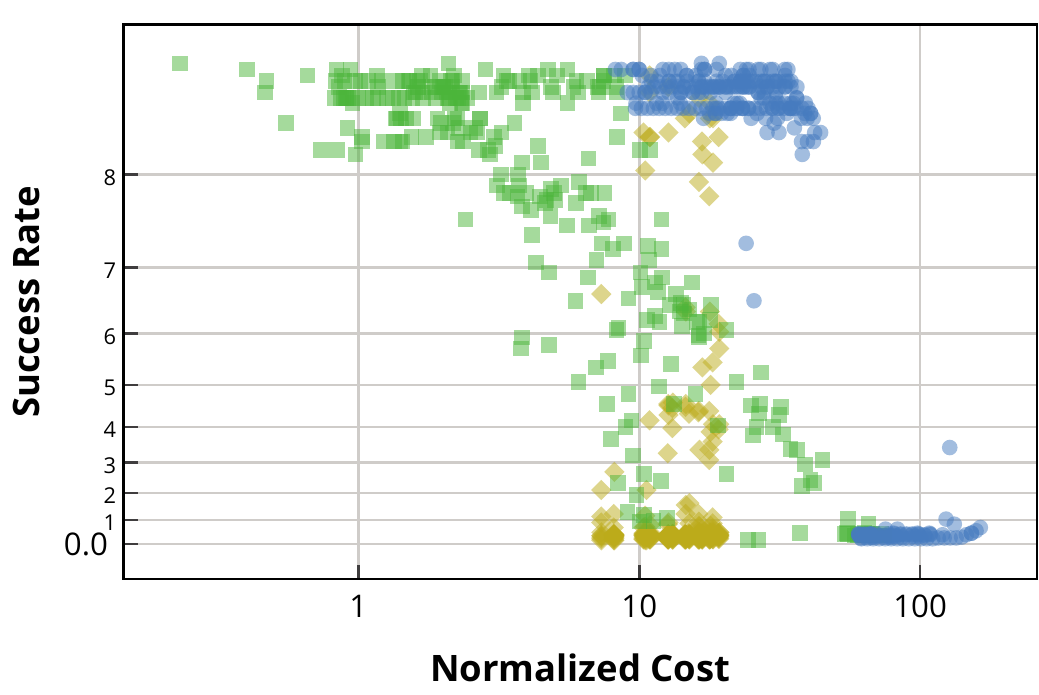}\label{fig:appendix-hartmann6_Scatter_overall_log_log_rev-Scatter-success_rate-norm_cum_c_t}}\\
    \subfloat[\robotsmall]{\includegraphics[width=0.32\textwidth]{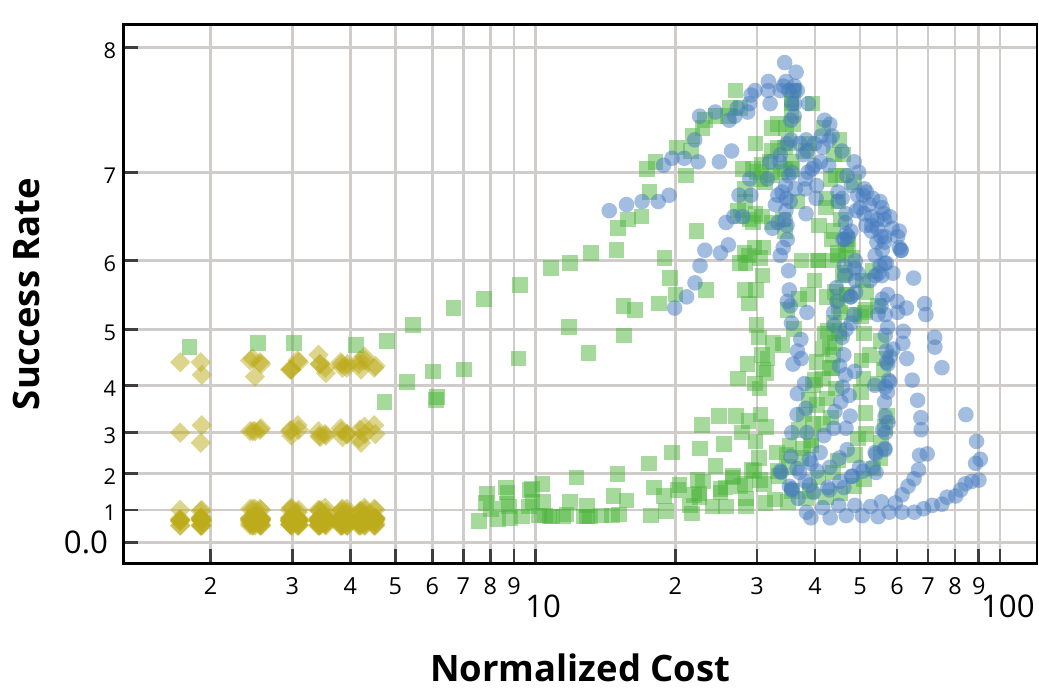}\label{fig:appendix-robot_pushing3d_Scatter_overall_log_log_rev-Scatter-success_rate-norm_cum_c_t}}
    \subfloat[\robotlarge]{\includegraphics[width=0.32\textwidth]{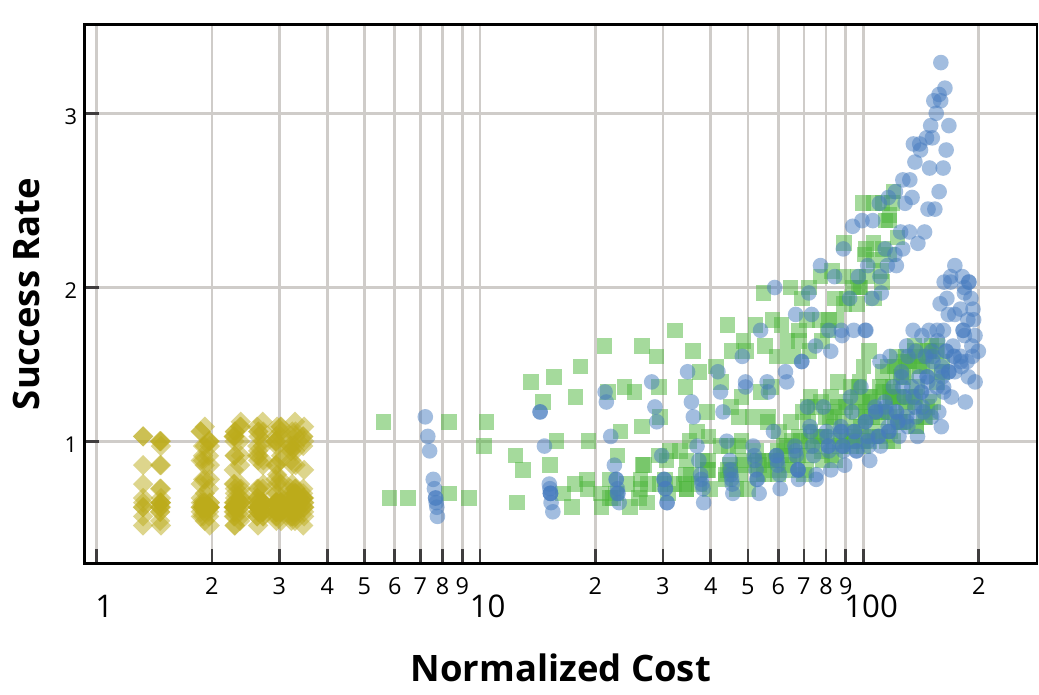}\label{fig:appendix-robot_pushing4d_Scatter_overall_log_log_rev-Scatter-success_rate-norm_cum_c_t}}
    \caption{Success rate vs.~cost with every random seed shown individually, i.e., each point is a single run.}
    \label{fig:appendix-SuccessRate_Cost}
\end{figure*}

\begin{figure*}[!ht]
    \centering
    \includegraphics[height=0.25cm]{img___1d___Scatter___legend-Scatter-success_rate-norm_cum_c_t.pdf} \\
    \subfloat[\branin-\rbf]{\includegraphics[width=0.32\textwidth]{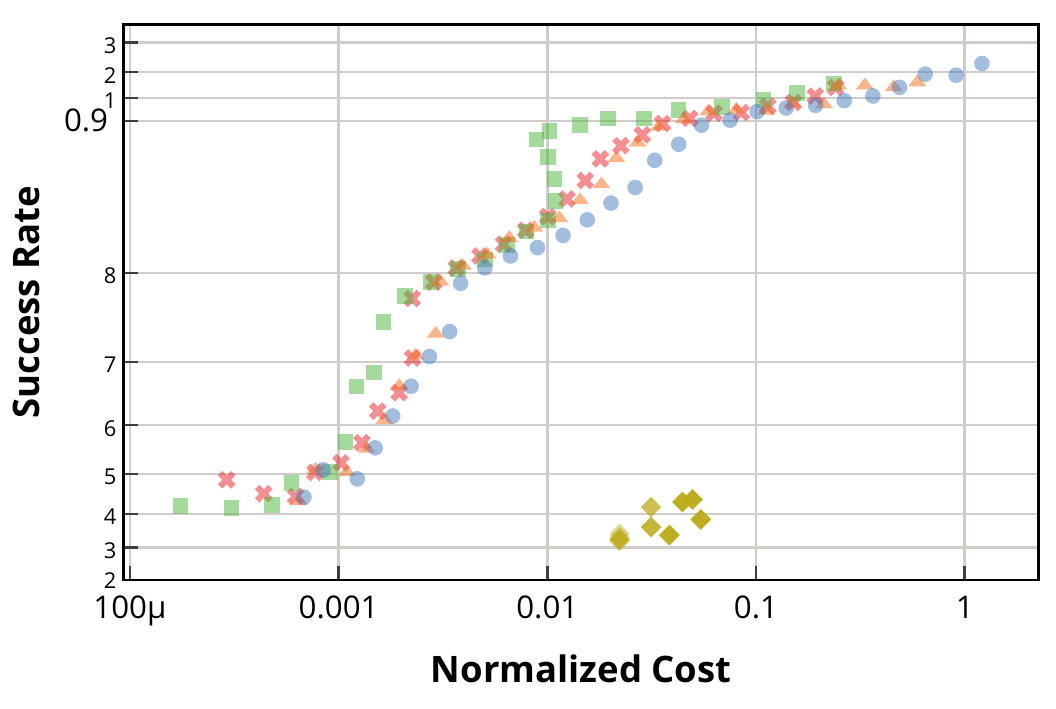}\label{fig:appendix-branin_rbf_Scatter_overall_avg_log_log_rev-Scatter-success_rate-norm_cum_c_t}}
    \subfloat[\branin-\maternx{5}{2}]{\includegraphics[width=0.32\textwidth]{img___branin___Scatter___overall_avg_log_log_rev-Scatter-success_rate-norm_cum_c_t.pdf}\label{fig:appendix-branin_Scatter_overall_avg_log_log_rev-Scatter-success_rate-norm_cum_c_t_1}}
    \subfloat[\branin-\maternx{3}{2}]{\includegraphics[width=0.32\textwidth]{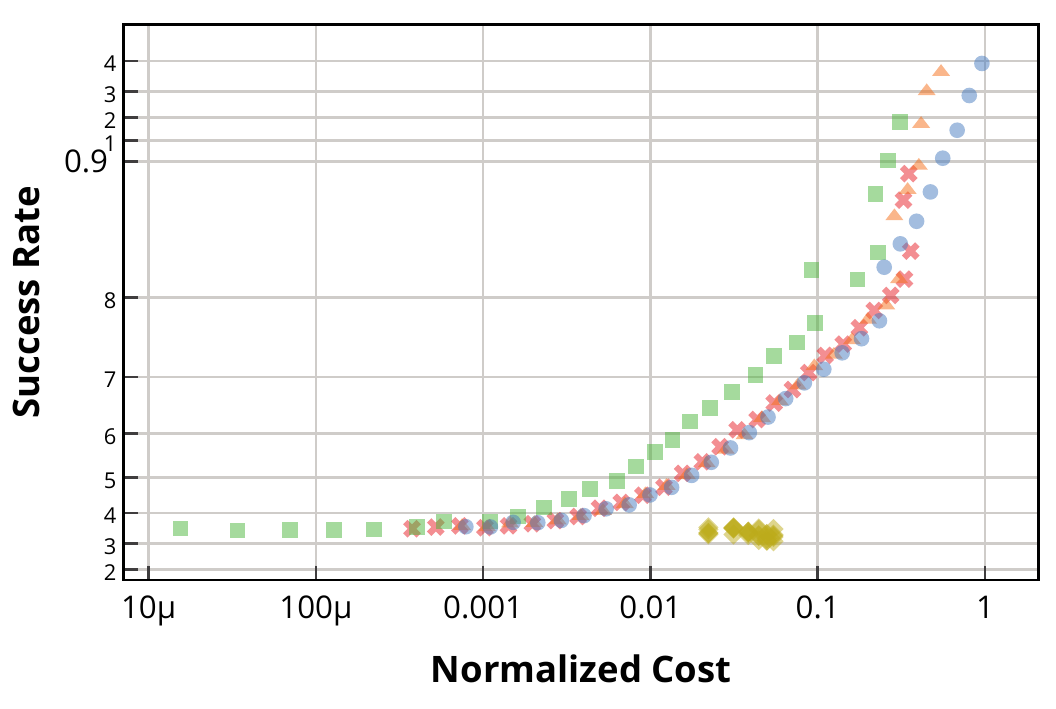}\label{fig:appendix-branin_matern32_Scatter_overall_avg_log_log_rev-Scatter-success_rate-norm_cum_c_t}}\\
    \subfloat[\camelback-\rbf]{\includegraphics[width=0.32\textwidth]{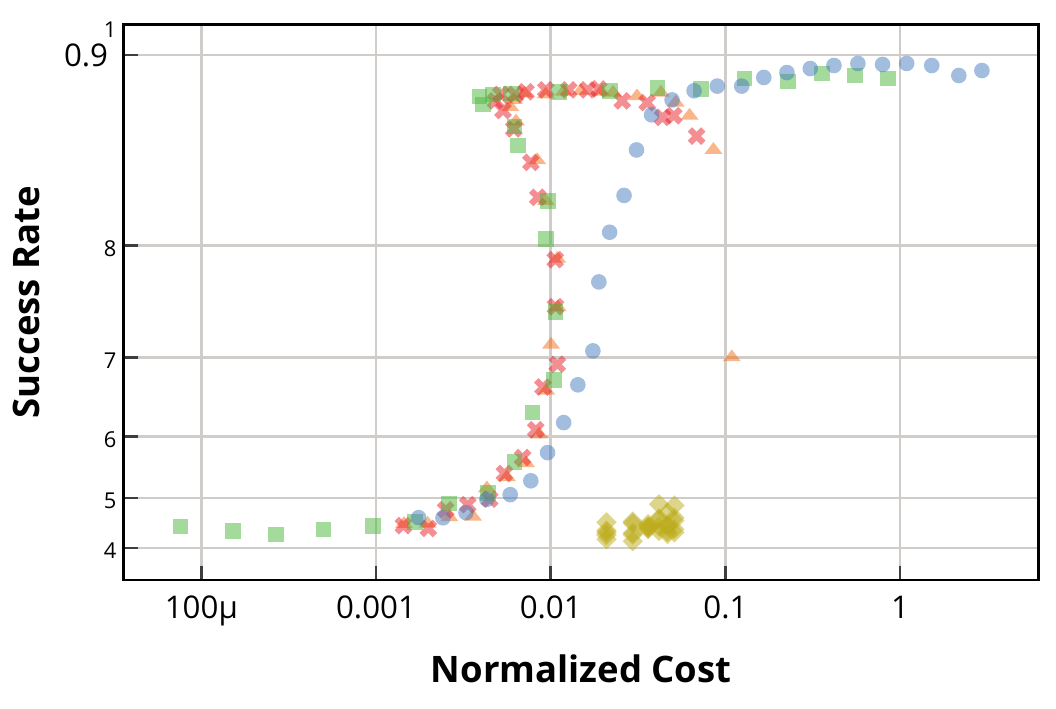}\label{fig:appendix-camelback_rbf_Scatter_overall_avg_log_log_rev-Scatter-success_rate-norm_cum_c_t}}
    \subfloat[\camelback-\maternx{5}{2}]{\includegraphics[width=0.32\textwidth]{img___camelback___Scatter___overall_avg_log_log_rev-Scatter-success_rate-norm_cum_c_t.pdf}\label{fig:appendix-camelback_Scatter_overall_avg_log_log_rev-Scatter-success_rate-norm_cum_c_t_1}}
    \subfloat[\camelback-\maternx{3}{2}]{\includegraphics[width=0.32\textwidth]{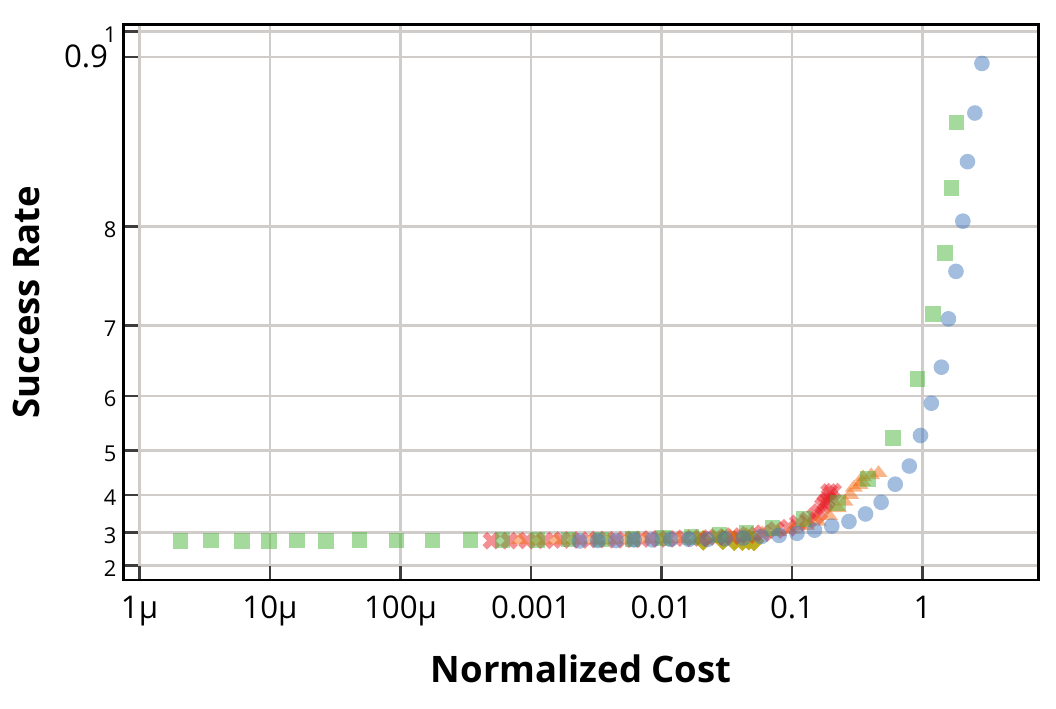}\label{fig:appendix-camelback_matern32_Scatter_overall_avg_log_log_rev-Scatter-success_rate-norm_cum_c_t}}\\
    
    \caption{Experiments comparing different kernels: Radial Basis Function (RBF), \maternx{5}{2}, and \maternx{3}{2}.}
    \label{fig:appendix-diffkernels}
\end{figure*}

\begin{figure*}[!ht]
    \centering
    \includegraphics[height=0.25cm]{img___1d___Scatter___legend-Scatter-success_rate-norm_cum_c_t.pdf} \\
    \subfloat[\varforrester]{\includegraphics[width=0.32\textwidth]{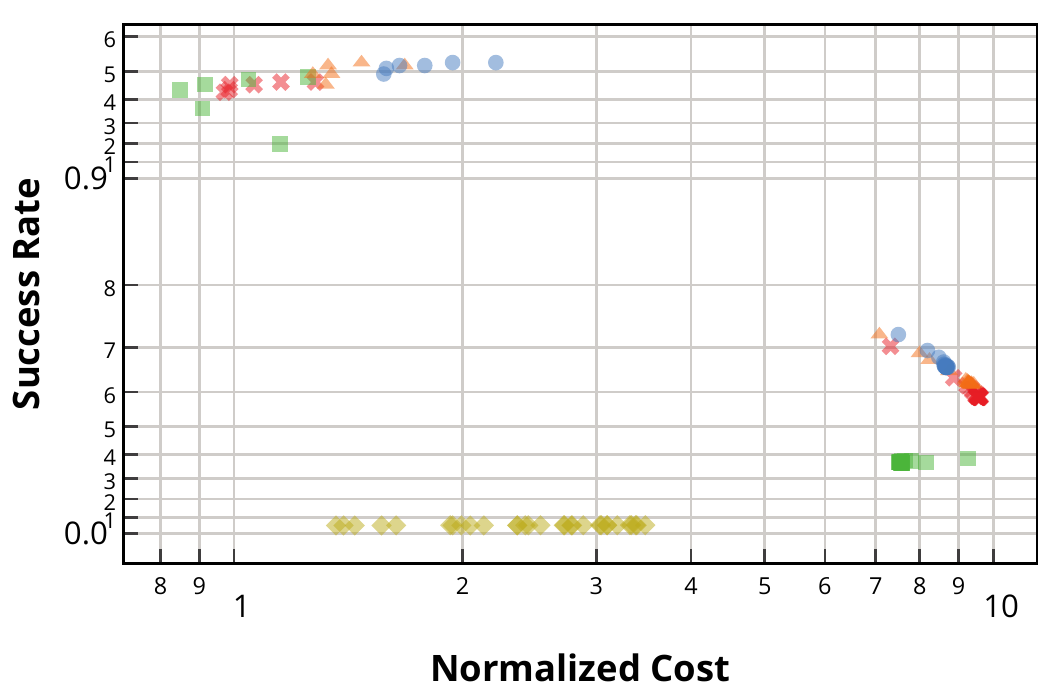}\label{fig:appendix-var_forrester_Scatter_overall_avg_log_log_rev-Scatter-success_rate-norm_cum_c_t}}
    \subfloat[\varcamelback]{\includegraphics[width=0.32\textwidth]{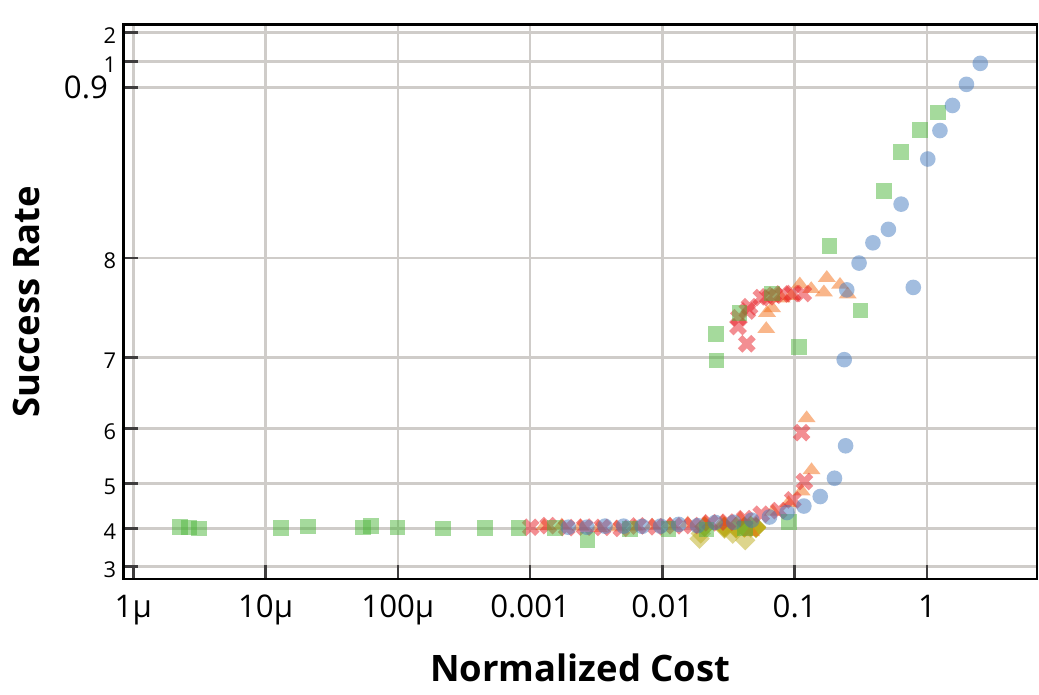}\label{fig:appendix-var_camelback_Scatter_overall_avg_log_log_rev-Scatter-success_rate-norm_cum_c_t}}
    \subfloat[\varrobotsmall]{\includegraphics[width=0.32\textwidth]{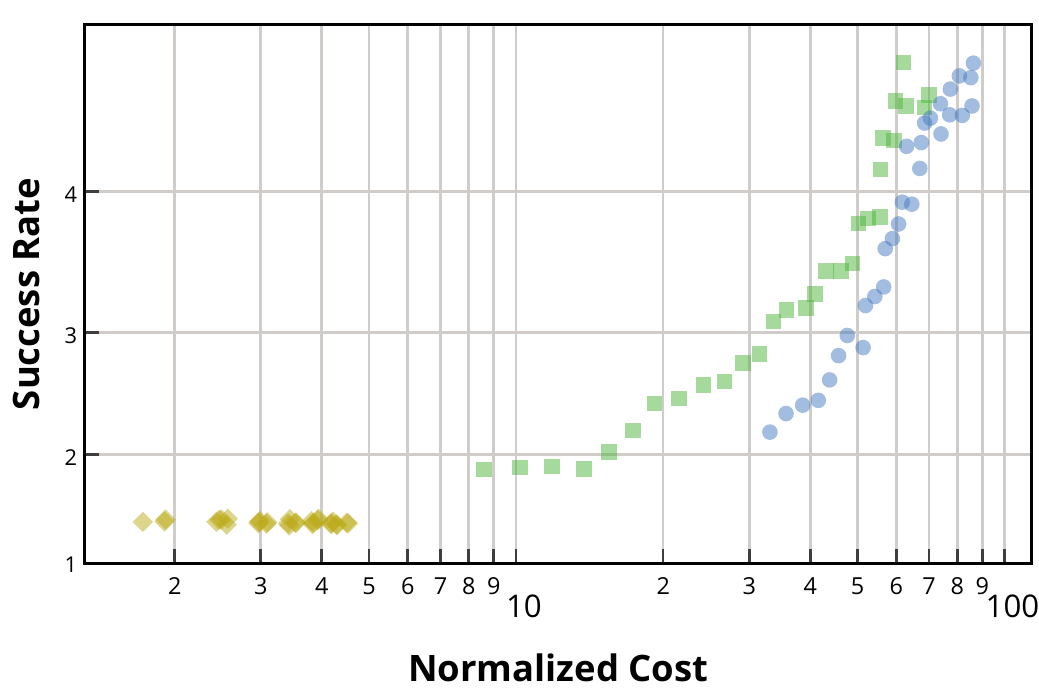}\label{fig:appendix-var_robot_pushing3d_Scatter_overall_avg_log_log_rev-Scatter-success_rate-norm_cum_c_t}}\\
    \subfloat[\varforrester]{\includegraphics[width=0.32\textwidth]{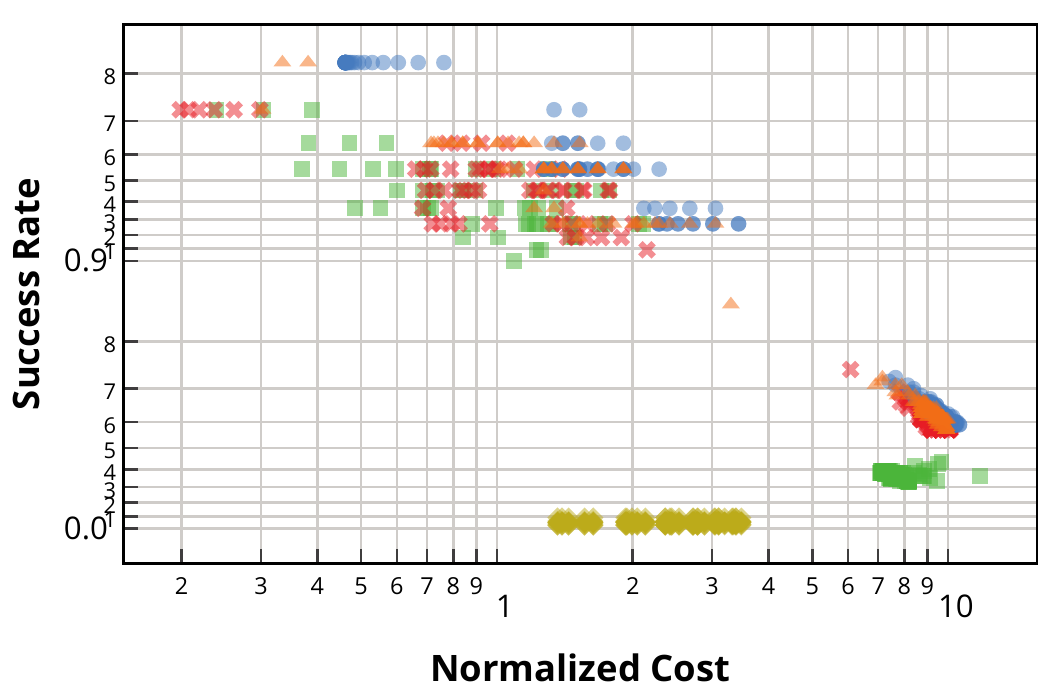}\label{fig:appendix-var_forrester_Scatter_overall_log_log_rev-Scatter-success_rate-norm_cum_c_t}}
    \subfloat[\varcamelback]{\includegraphics[width=0.32\textwidth]{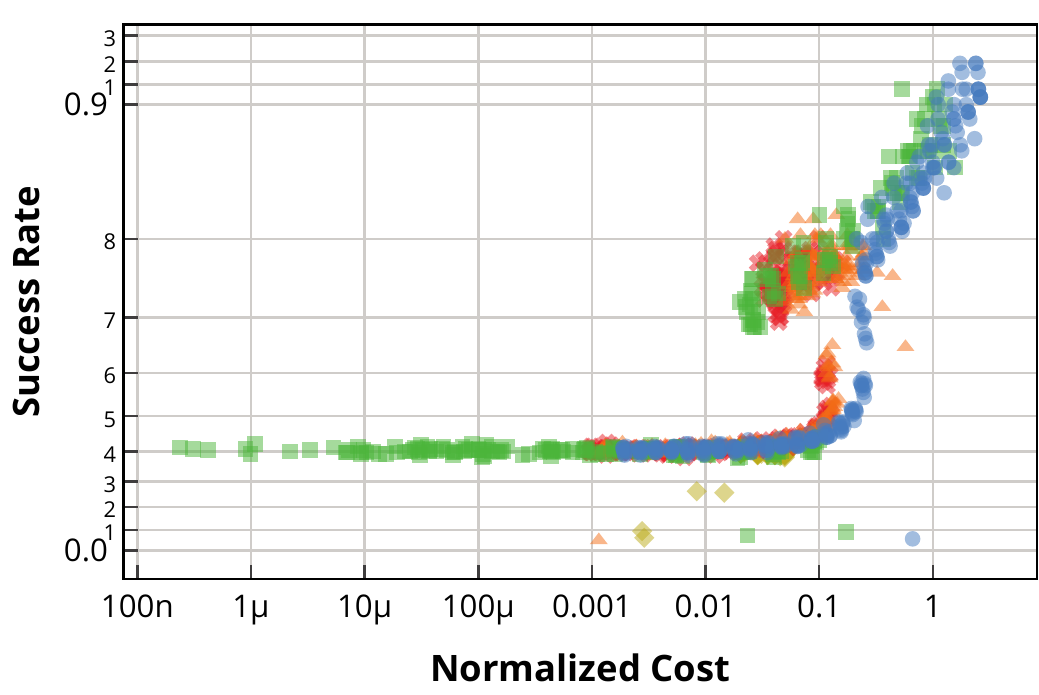}\label{fig:appendix-var_camelback_Scatter_overall_log_log_rev-Scatter-success_rate-norm_cum_c_t}}
    \subfloat[\varrobotsmall]{\includegraphics[width=0.32\textwidth]{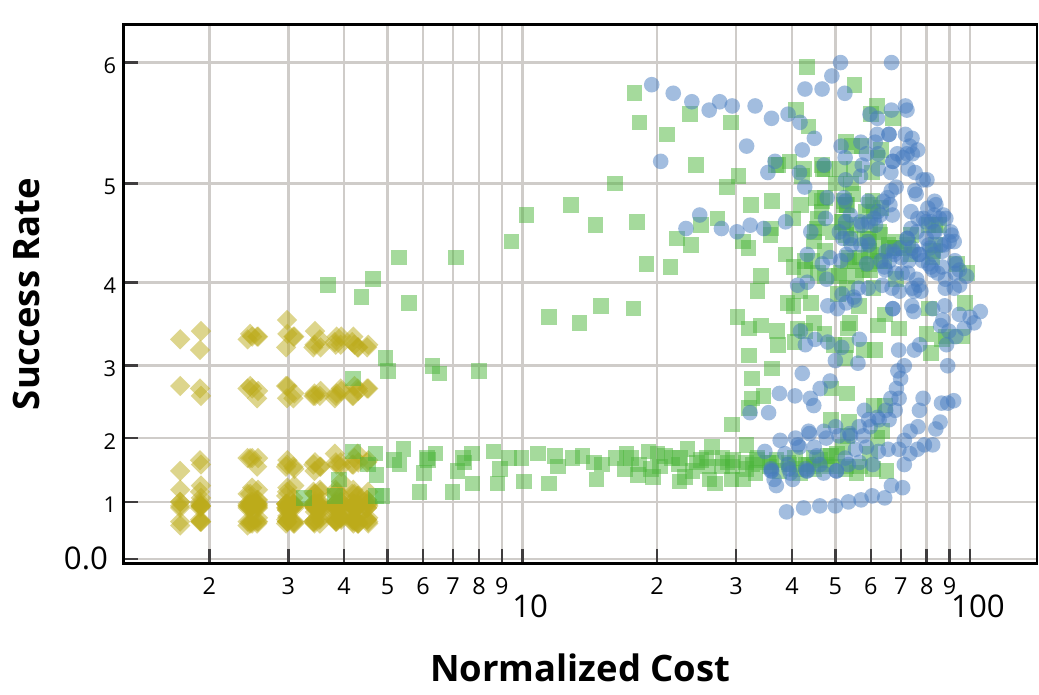}\label{fig:appendix-var_robot_pushing3d_Scatter_overall_log_log_rev-Scatter-success_rate-norm_cum_c_t}}\\
    \caption{We attack the MaxVar + Elimination algorithm in $3$ experiments.}
    \label{fig:appendix-maxvar}
\end{figure*}

\begin{figure*}[!ht]
    \centering
    \includegraphics[height=0.25cm]{img___1d___Scatter___legend-Scatter-success_rate-norm_cum_c_t.pdf} \\
    \subfloat[\synthetic]{\includegraphics[width=0.32\textwidth]{img___1d___Scatter___overall_avg_log_log_rev-Scatter-success_rate-norm_cum_c_t.pdf}\label{fig:appendix-1d_Scatter_overall_avg_log_log_rev-Scatter-success_rate-norm_cum_c_t_1}}
    \subfloat[\forrester]{\includegraphics[width=0.32\textwidth]{img___forrester___Scatter___overall_avg_log_log_rev-Scatter-success_rate-norm_cum_c_t.pdf}\label{fig:appendix-forrester_Scatter_overall_avg_log_log_rev-Scatter-success_rate-norm_cum_c_t_1}}
    \subfloat[\bohachevsky]{\includegraphics[width=0.32\textwidth]{img___bohachevsky___Scatter___overall_avg_log_log_rev-Scatter-success_rate-norm_cum_c_t.pdf}\label{fig:appendix-bohachevsky_Scatter_overall_avg_log_log_rev-Scatter-success_rate-norm_cum_c_t_1}}\\
    \subfloat[\npsynthetic]{\includegraphics[width=0.32\textwidth]{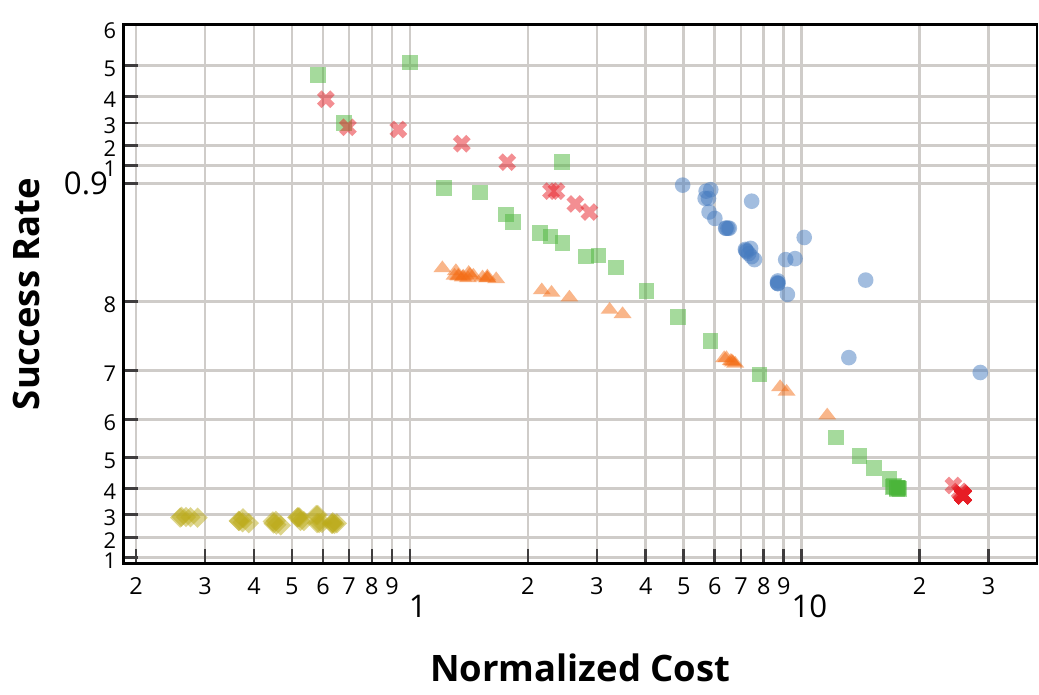}\label{fig:appendix-np_1d_Scatter_overall_avg_log_log_rev-Scatter-success_rate-norm_cum_c_t_1}}
    \subfloat[\npforrester]{\includegraphics[width=0.32\textwidth]{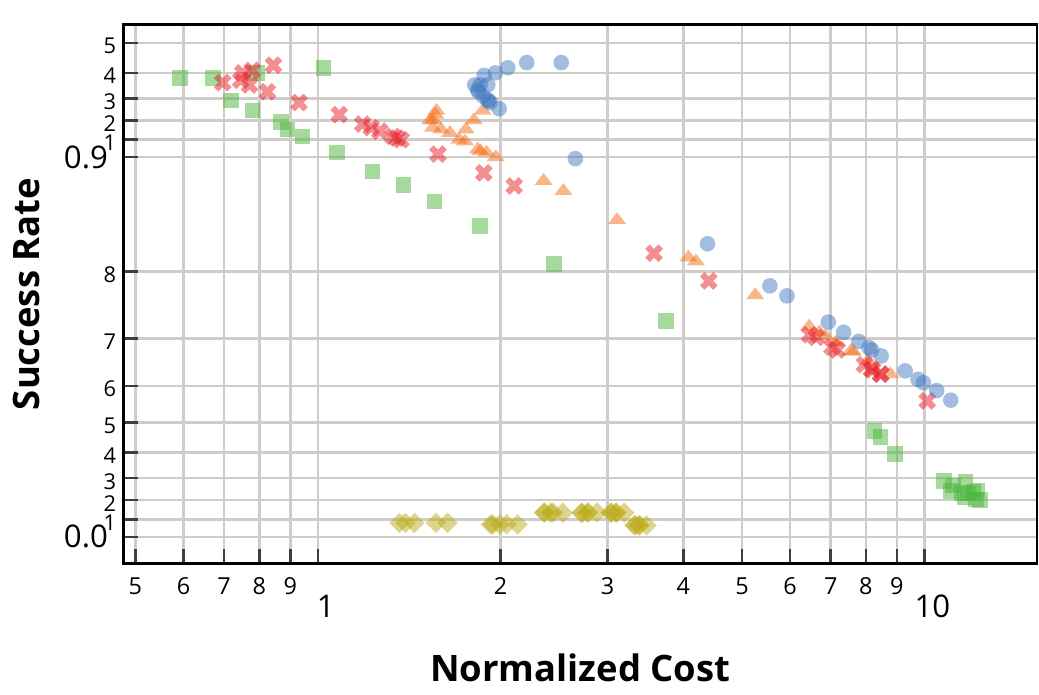}\label{fig:appendix-np_forrester_Scatter_overall_avg_log_log_rev-Scatter-success_rate-norm_cum_c_t_1}}
    \subfloat[\npbohachevsky]{\includegraphics[width=0.32\textwidth]{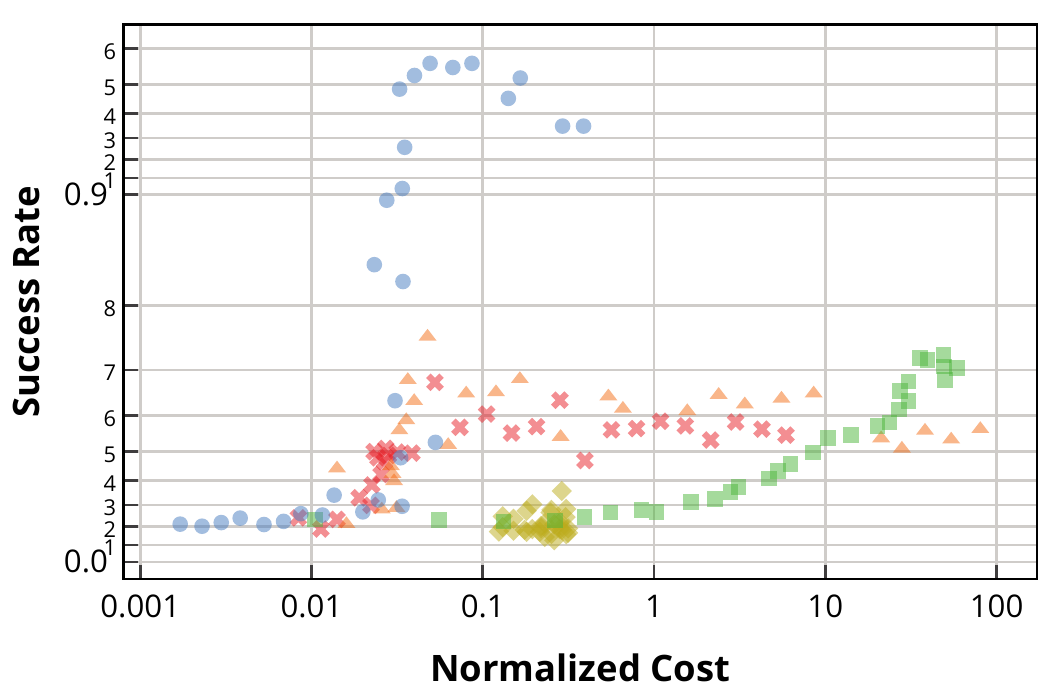}\label{fig:appendix-np_bohachevsky_Scatter_overall_avg_log_log_rev-Scatter-success_rate-norm_cum_c_t_1}}\\

    \subfloat[\synthetic]{\includegraphics[width=0.32\textwidth]{img___1d___Scatter___overall_log_log_rev-Scatter-success_rate-norm_cum_c_t.pdf}\label{fig:appendix-1d_Scatter_overall_log_log_rev-Scatter-success_rate-norm_cum_c_t_1}}
    \subfloat[\forrester]{\includegraphics[width=0.32\textwidth]{img___forrester___Scatter___overall_log_log_rev-Scatter-success_rate-norm_cum_c_t.pdf}\label{fig:appendix-forrester_Scatter_overall_log_log_rev-Scatter-success_rate-norm_cum_c_t_1}}
    \subfloat[\bohachevsky]{\includegraphics[width=0.32\textwidth]{img___bohachevsky___Scatter___overall_log_log_rev-Scatter-success_rate-norm_cum_c_t.pdf}\label{fig:appendix-bohachevsky_Scatter_overall_log_log_rev-Scatter-success_rate-norm_cum_c_t_1}}\\
    \subfloat[\npsynthetic]{\includegraphics[width=0.32\textwidth]{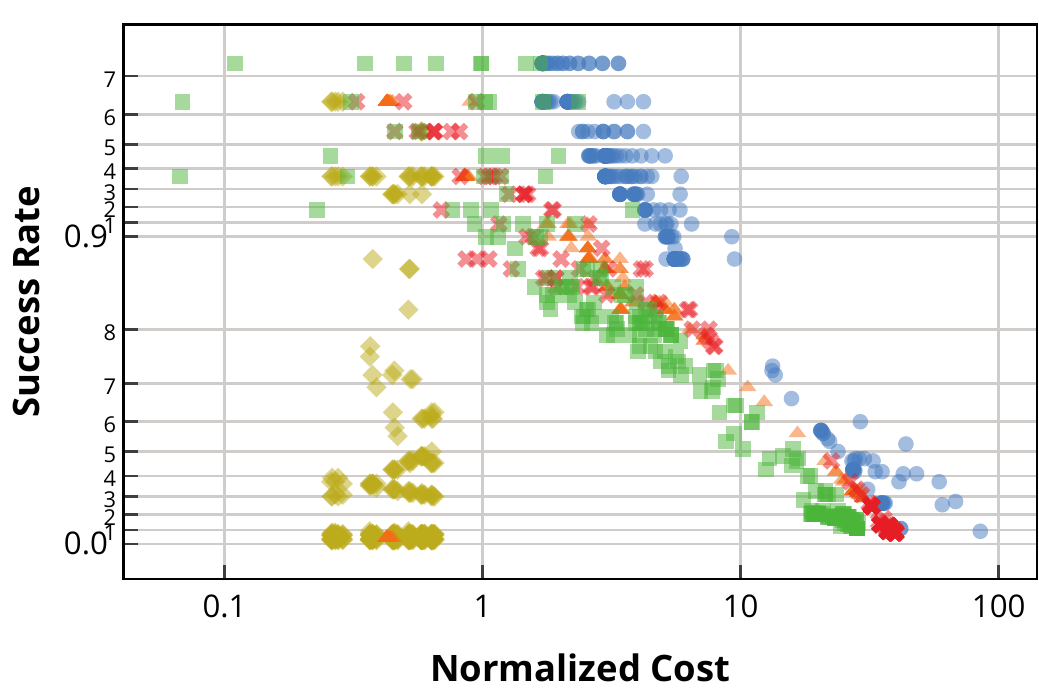}\label{fig:appendix-np_1d_Scatter_overall_log_log_rev-Scatter-success_rate-norm_cum_c_t}}
    \subfloat[\npforrester]{\includegraphics[width=0.32\textwidth]{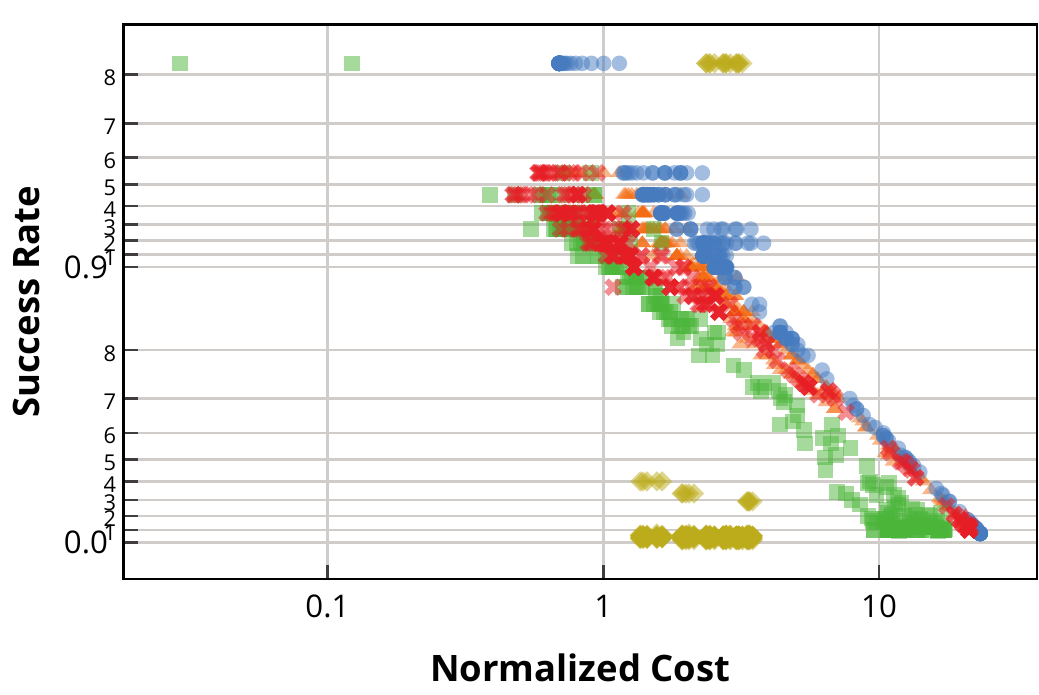}\label{fig:appendix-np_forrester_Scatter_overall_log_log_rev-Scatter-success_rate-norm_cum_c_t}}
    \subfloat[\npbohachevsky]{\includegraphics[width=0.32\textwidth]{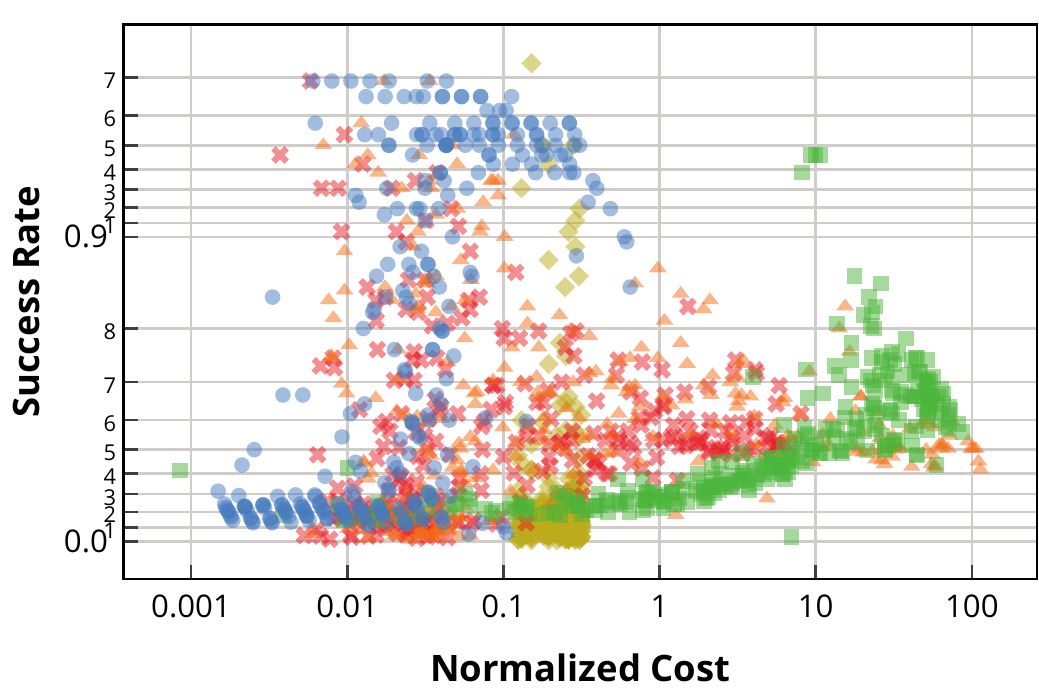}\label{fig:appendix-np_bohachevsky_Scatter_overall_log_log_rev-Scatter-success_rate-norm_cum_c_t}}\\    
    
    \caption{Experiments comparing the kernel parameters being learned online (second and forth rows) vs.~learned from data sampled from $f$ prior to the optimization (first and third rows).  The top two rows average over random seeds, and the bottom two rows show every individual run.}
    \label{fig:appendix-noprefit}
\end{figure*}

\begin{figure*}[!ht]
    \centering
    \includegraphics[height=0.25cm]{img___1d___Scatter___legend-Scatter-success_rate-norm_cum_c_t.pdf} \\
    \subfloat[\levydef{0.5}{0.01}]{\includegraphics[width=0.32\textwidth]{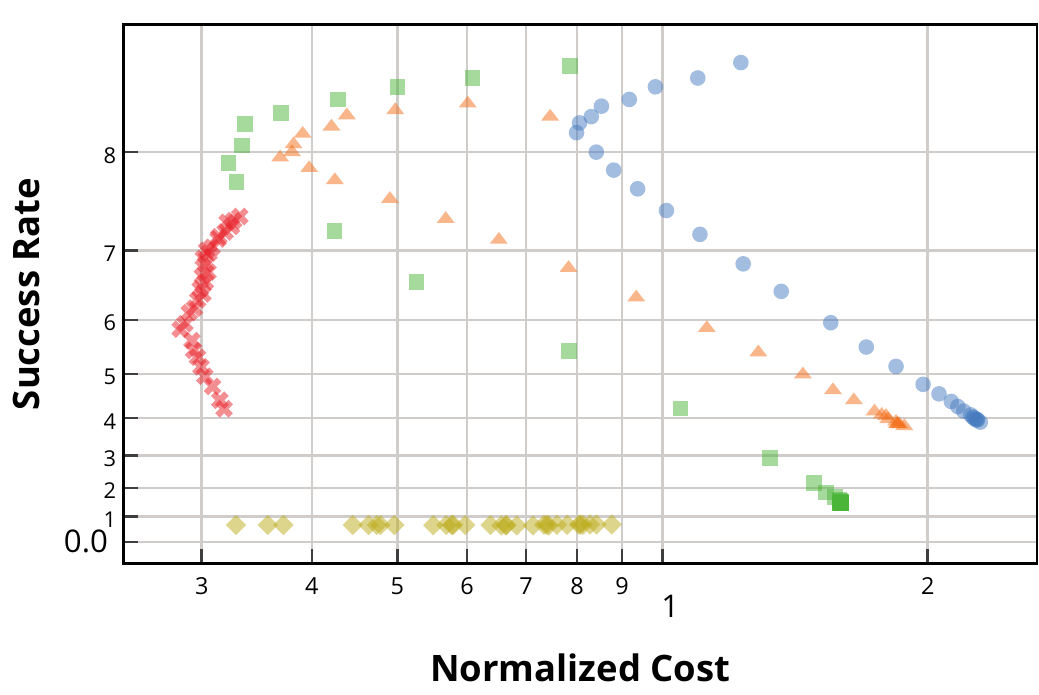}\label{fig:appendix-levy_def_C0_5_S0_01_Scatter_overall_avg_log_log_rev-Scatter-success_rate-norm_cum_c_t}}
    \subfloat[\levydef{0.5}{0.1}]{\includegraphics[width=0.32\textwidth]{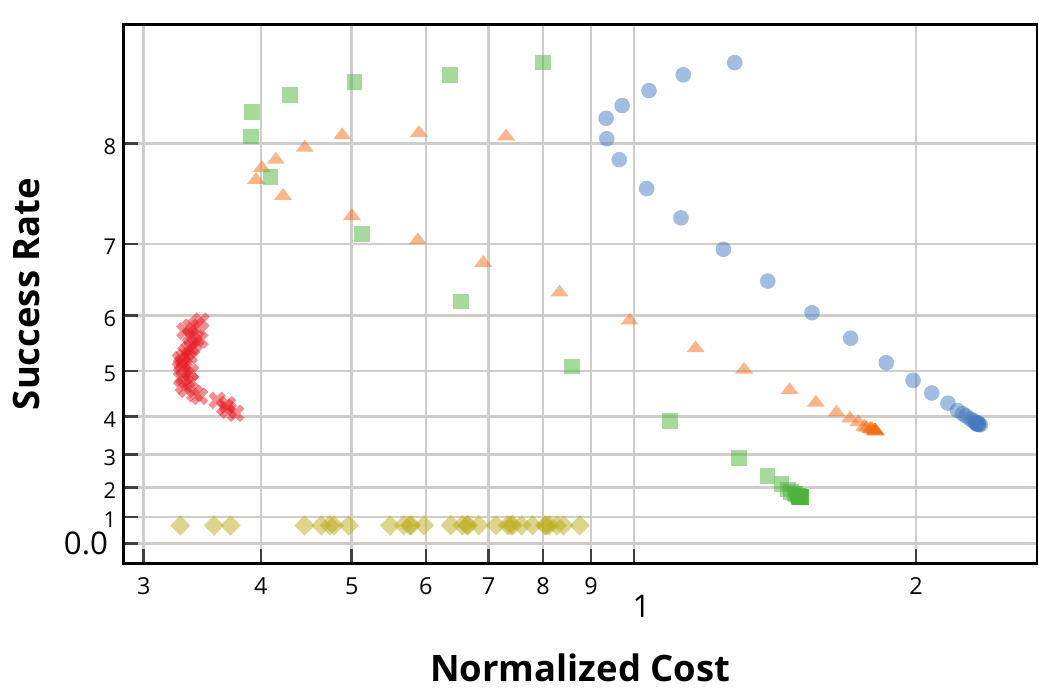}\label{fig:appendix-levy_def_C0_5_S0_1_Scatter_overall_avg_log_log_rev-Scatter-success_rate-norm_cum_c_t}}
    \subfloat[\levydef{0.5}{1}]{\includegraphics[width=0.32\textwidth]{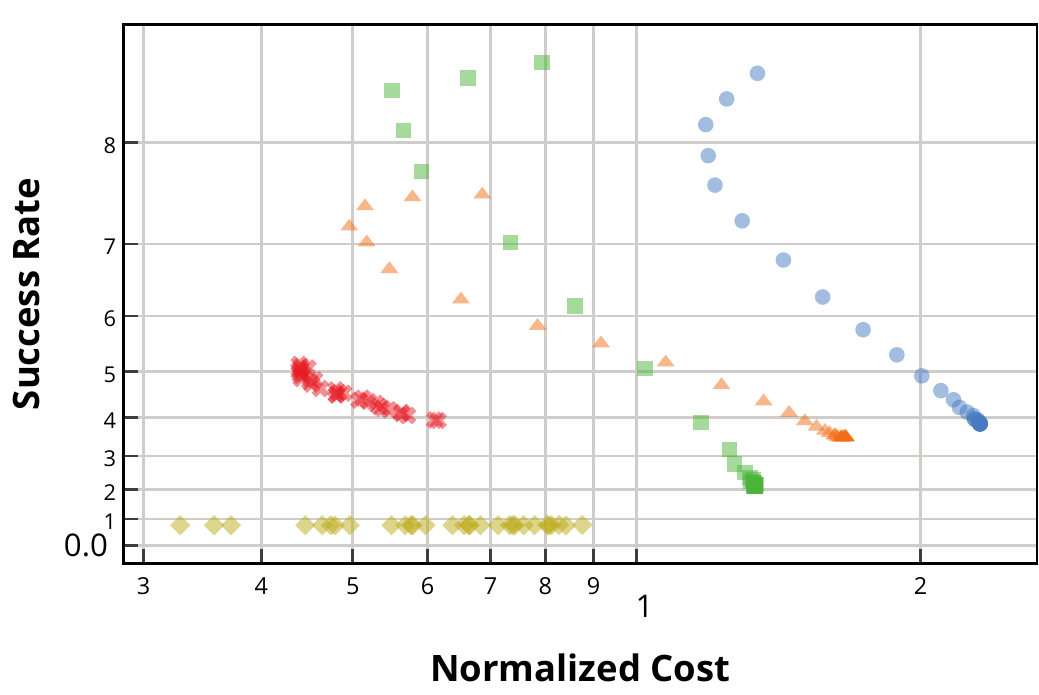}\label{fig:appendix-levy_def_C0_5_S1_Scatter_overall_avg_log_log_rev-Scatter-success_rate-norm_cum_c_t}}\\
    \subfloat[\levydef{2}{0.01}]{\includegraphics[width=0.32\textwidth]{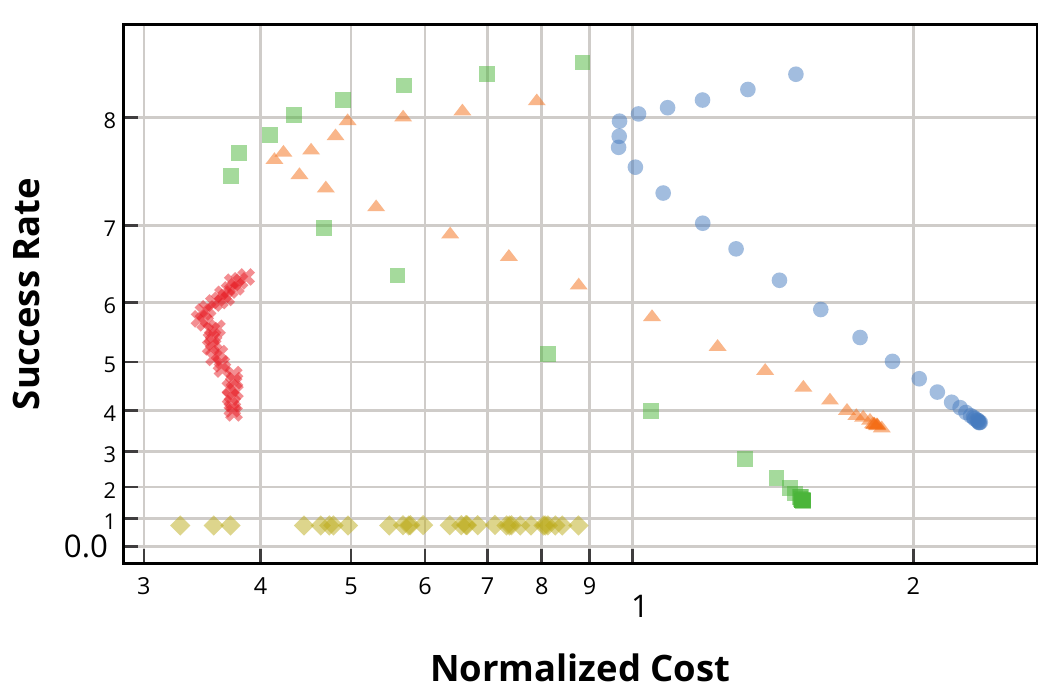}\label{fig:appendix-levy_def_C2_S0_01_Scatter_overall_avg_log_log_rev-Scatter-success_rate-norm_cum_c_t}}
    \subfloat[\levydef{2}{0.1}]{\includegraphics[width=0.32\textwidth]{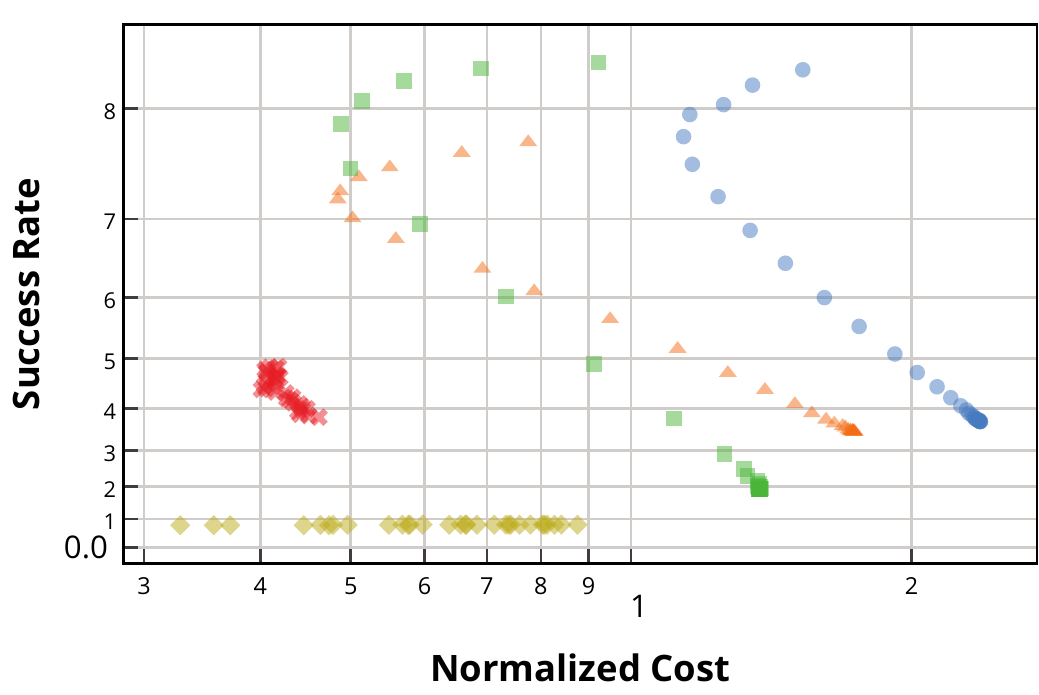}\label{fig:appendix-levy_def_C2_S0_1_Scatter_overall_avg_log_log_rev-Scatter-success_rate-norm_cum_c_t}}
    \subfloat[\levydef{2}{1}]{\includegraphics[width=0.32\textwidth]{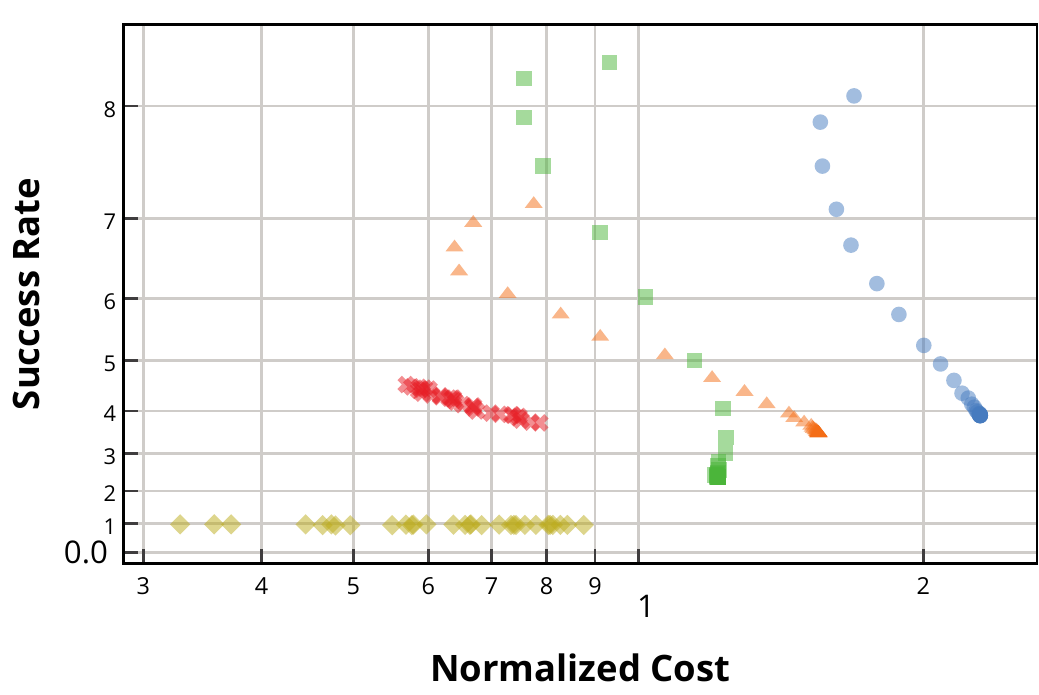}\label{fig:appendix-levy_def_C2_S1_Scatter_overall_avg_log_log_rev-Scatter-success_rate-norm_cum_c_t}}\\
    \subfloat[\levydef{8}{0.01}]{\includegraphics[width=0.32\textwidth]{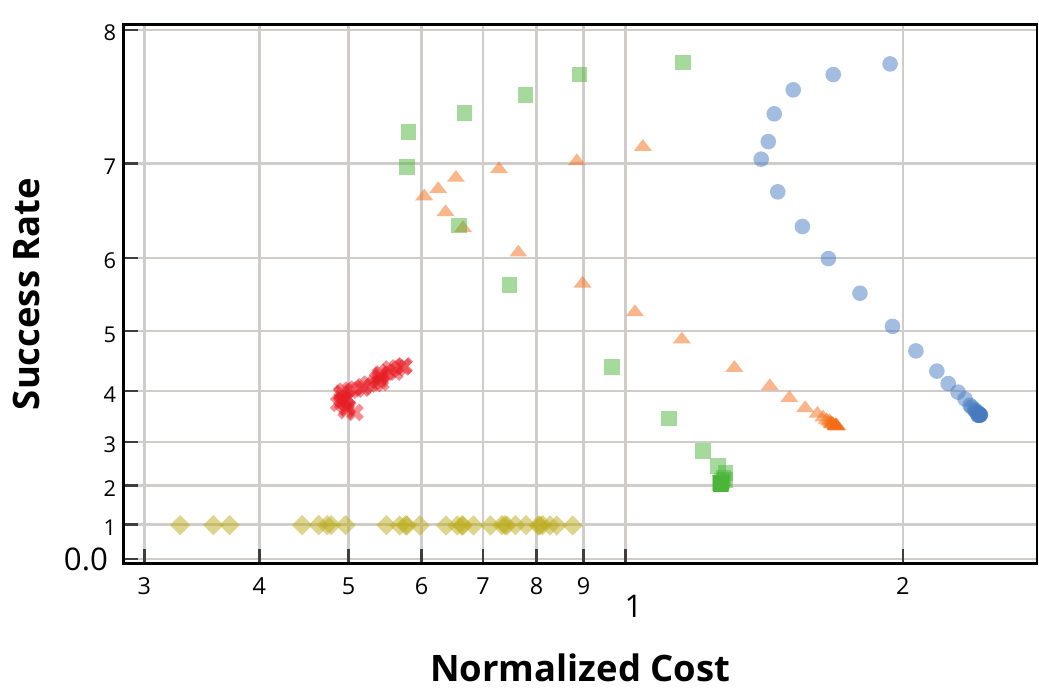}\label{fig:appendix-levy_def_C8_S0_01_Scatter_overall_avg_log_log_rev-Scatter-success_rate-norm_cum_c_t}}
    \subfloat[\levydef{8}{0.1}]{\includegraphics[width=0.32\textwidth]{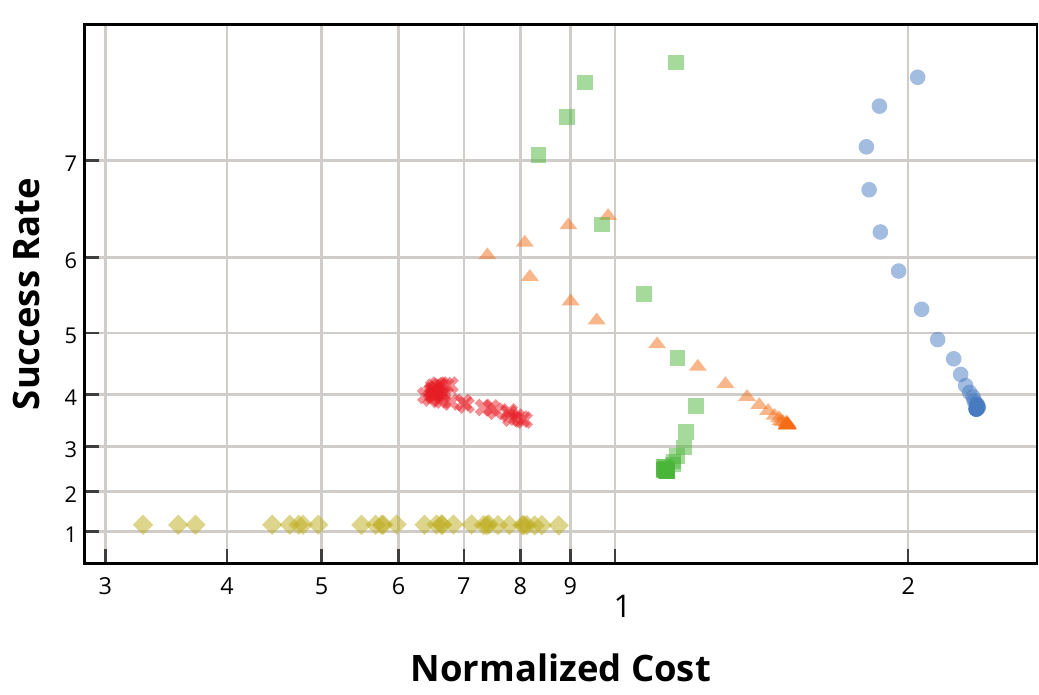}\label{fig:appendix-levy_def_C8_S0_1_Scatter_overall_avg_log_log_rev-Scatter-success_rate-norm_cum_c_t}}
    \subfloat[\levydef{8}{1}]{\includegraphics[width=0.32\textwidth]{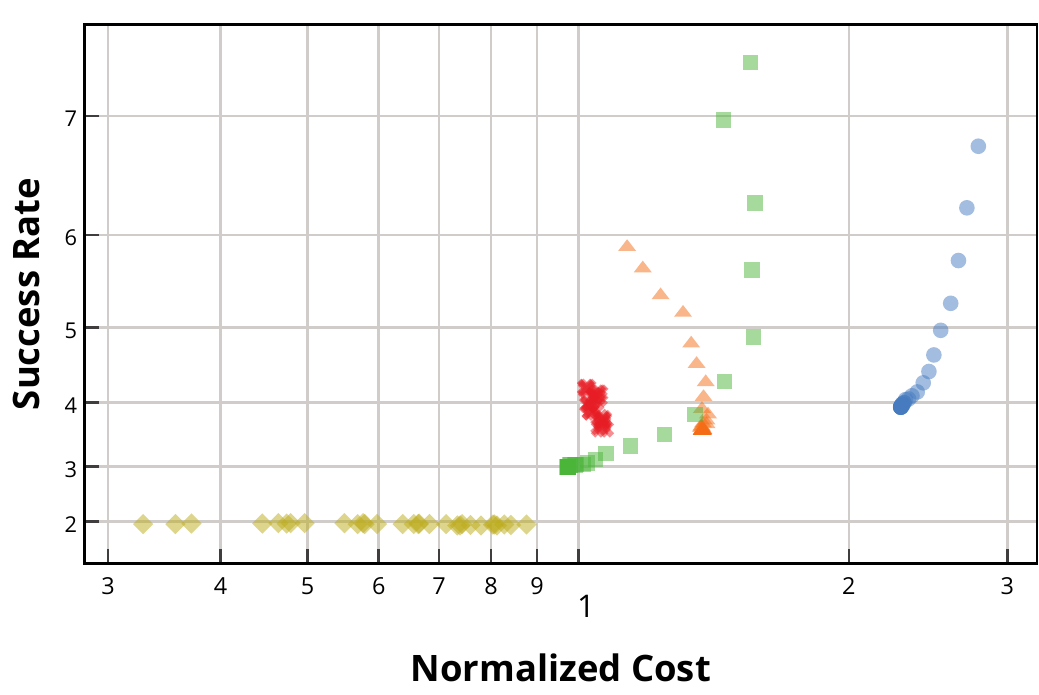}\label{fig:appendix-levy_def_C8_S1_Scatter_overall_avg_log_log_rev-Scatter-success_rate-norm_cum_c_t}}\\
    \caption{Variants of \gpucb which add a constant $C$ to the confidence width in order to defend against the attack.}
    \label{fig:appendix-defend}
\end{figure*}

\begin{figure*}[!ht]
    \centering
    \includegraphics[height=0.25cm]{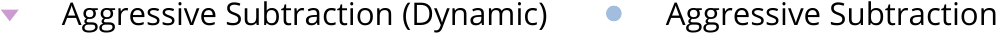} \\
    \subfloat[\synthetic-\dynamic]{\includegraphics[width=0.32\textwidth]{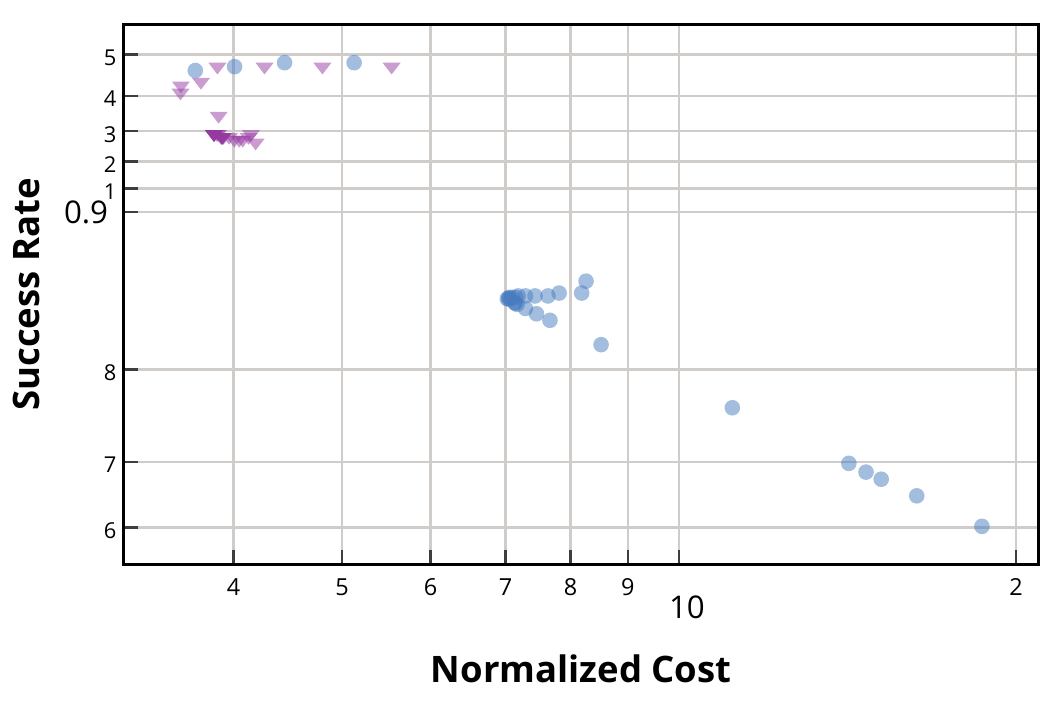}\label{fig:appendix-1d_dynamic_Scatter_dynamic_overall_avg_log_log_rev-Scatter-success_rate-norm_cum_c_t}}
    \subfloat[\forrester-\dynamic]{\includegraphics[width=0.32\textwidth]{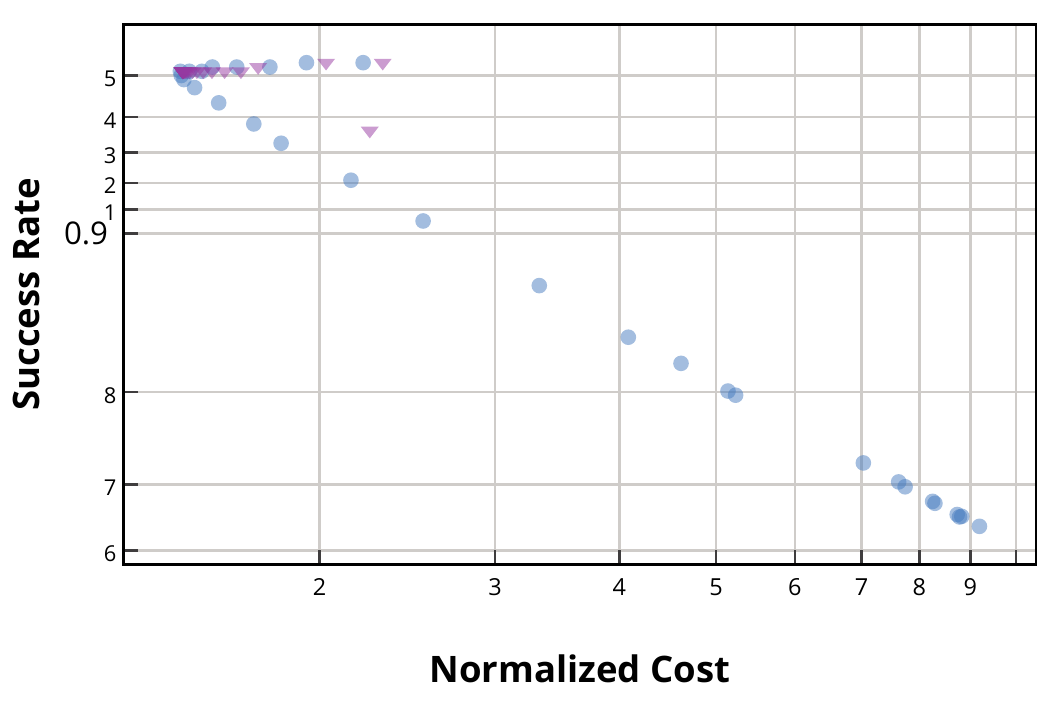}\label{fig:appendix-forrester_dynamic_Scatter_dynamic_overall_avg_log_log_rev-Scatter-success_rate-norm_cum_c_t}}
    \subfloat[\levy-\dynamic]{\includegraphics[width=0.32\textwidth]{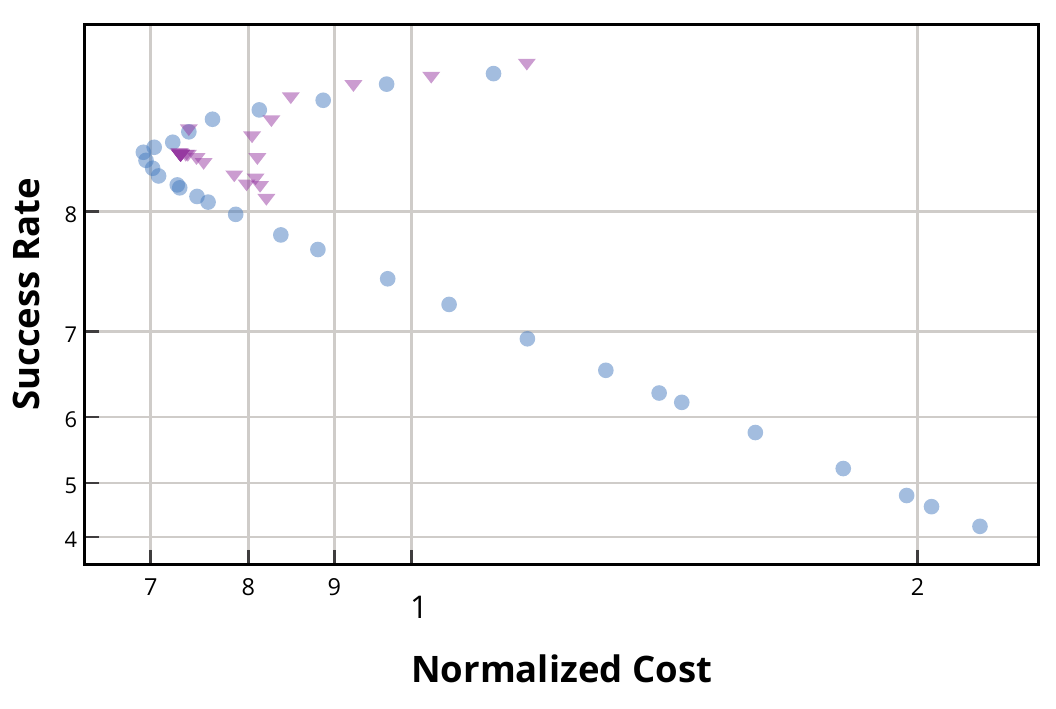}\label{fig:appendix-levy_dynamic_Scatter_dynamic_overall_avg_log_log_rev-Scatter-success_rate-norm_cum_c_t}}\\
    \subfloat[\synthetic-\dynamic]{\includegraphics[width=0.32\textwidth]{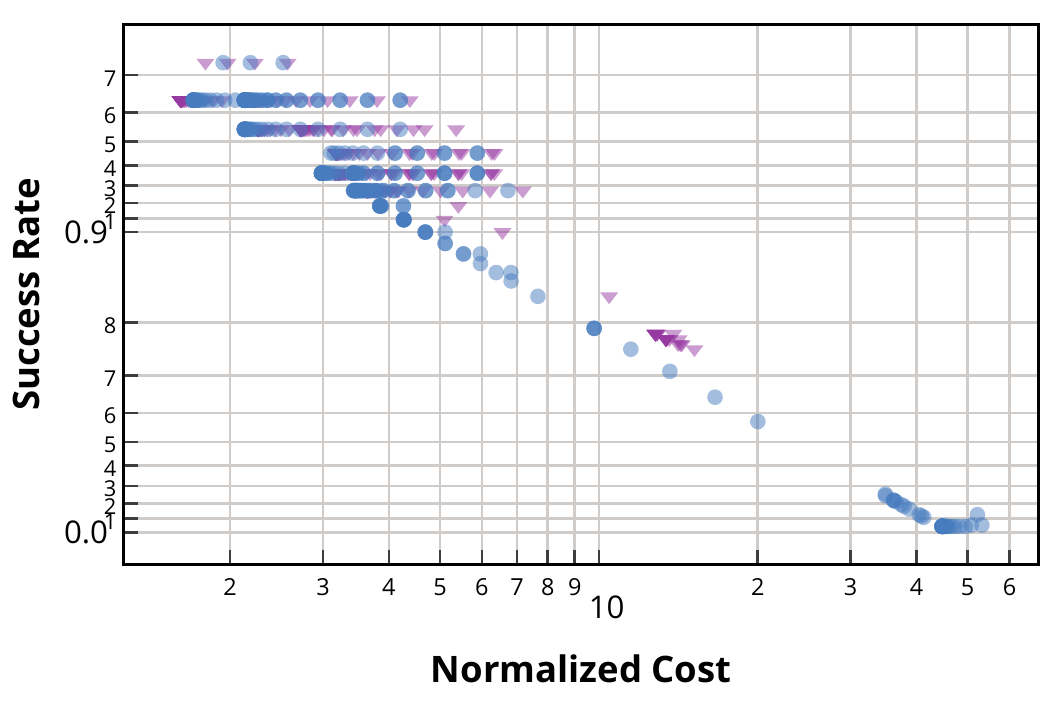}\label{fig:appendix-1d_dynamic_Scatter_dynamic_overall_log_log_rev-Scatter-success_rate-norm_cum_c_t}}
    \subfloat[\forrester-\dynamic]{\includegraphics[width=0.32\textwidth]{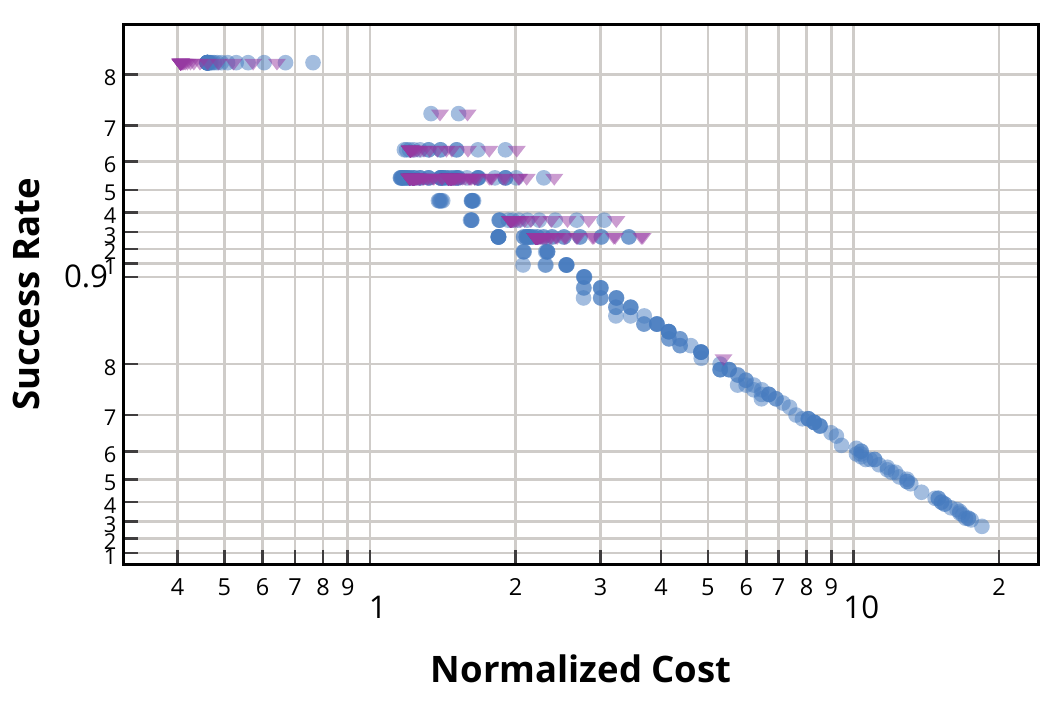}\label{fig:appendix-forrester_dynamic_Scatter_dynamic_overall_log_log_rev-Scatter-success_rate-norm_cum_c_t}}
    \subfloat[\levy-\dynamic]{\includegraphics[width=0.32\textwidth]{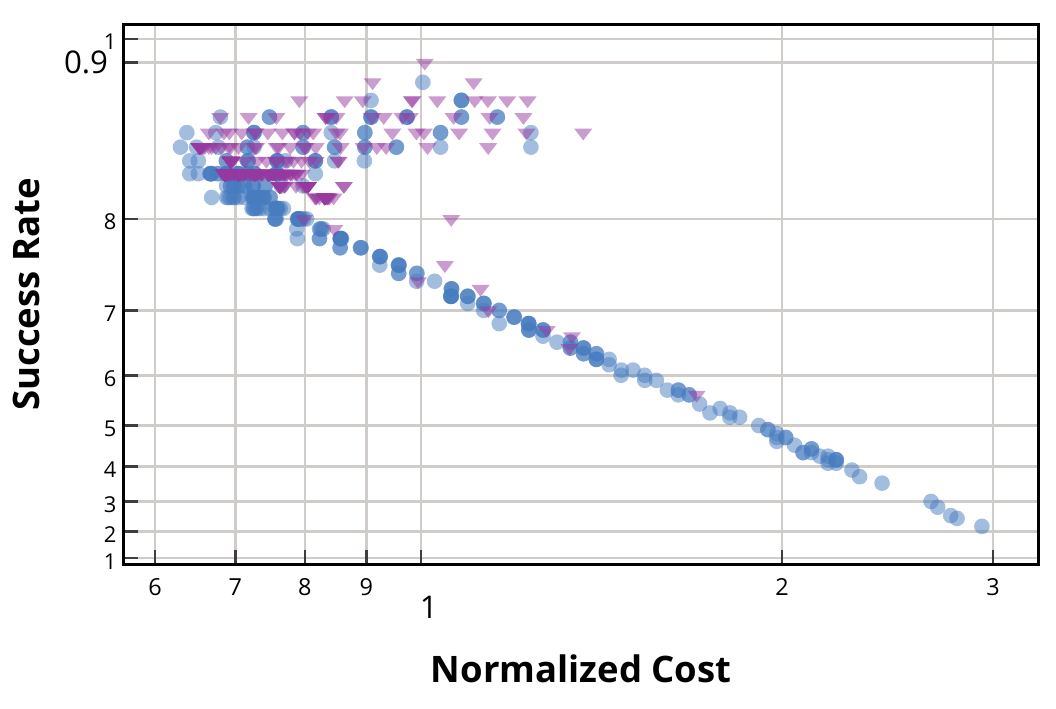}\label{fig:appendix-levy_dynamic_Scatter_dynamic_overall_log_log_rev-Scatter-success_rate-norm_cum_c_t}}\\
    \caption{Experiments applying a simple dynamic hyperparameter strategy for the \aggressivesubtraction attack. The top row averages over random seeds, and the bottom row shows every individual run.}
    \label{fig:appendix-dynamic}
\end{figure*}


\begin{thebibliography}{31}
\providecommand{\natexlab}[1]{#1}
\providecommand{\url}[1]{\texttt{#1}}
\expandafter\ifx\csname urlstyle\endcsname\relax
  \providecommand{\doi}[1]{doi: #1}\else
  \providecommand{\doi}{doi: \begingroup \urlstyle{rm}\Url}\fi

\bibitem[Anaconda(2016)]{anaconda}
Anaconda.
\newblock Anaconda software distribution, 11 2016.
\newblock URL \url{https://docs.anaconda.com/}.

\bibitem[Beland \& Nair(2017)Beland and Nair]{Bel17}
Beland, J.~J. and Nair, P.~B.
\newblock Bayesian optimization under uncertainty.
\newblock {NIPS} BayesOpt 2017 workshop, 2017.

\bibitem[Bertsimas et~al.(2010)Bertsimas, Nohadani, and Teo]{Ber10}
Bertsimas, D., Nohadani, O., and Teo, K.~M.
\newblock Nonconvex robust optimization for problems with constraints.
\newblock \emph{INFORMS journal on Computing}, 22\penalty0 (1):\penalty0
  44--58, 2010.

\bibitem[Bogunovic et~al.(2018)Bogunovic, Scarlett, Jegelka, and Cevher]{Bog18}
Bogunovic, I., Scarlett, J., Jegelka, S., and Cevher, V.
\newblock Adversarially robust optimization with {G}aussian processes.
\newblock In \emph{Conf. Neur. Inf. Proc. Sys. (NeurIPS)}, 2018.

\bibitem[Bogunovic et~al.(2020{\natexlab{a}})Bogunovic, Krause, and
  Scarlett]{Bog20}
Bogunovic, I., Krause, A., and Scarlett, J.
\newblock Corruption-tolerant {G}aussian process bandit optimization.
\newblock In \emph{Int. Conf. Art. Intel. Stats. (AISTATS)},
  2020{\natexlab{a}}.

\bibitem[Bogunovic et~al.(2020{\natexlab{b}})Bogunovic, Losalka, Krause, and
  Scarlett]{Bog20a}
Bogunovic, I., Losalka, A., Krause, A., and Scarlett, J.
\newblock Stochastic linear bandits robust to adversarial attacks.
\newblock https://arxiv.org/abs/2007.03285, 2020{\natexlab{b}}.

\bibitem[Bull(2011)]{Bul11}
Bull, A.~D.
\newblock Convergence rates of efficient global optimization algorithms.
\newblock \emph{J. Mach. Learn. Res.}, 12\penalty0 (Oct.):\penalty0 2879--2904,
  2011.

\bibitem[Cai \& Scarlett(2020)Cai and Scarlett]{Cai20}
Cai, X. and Scarlett, J.
\newblock On lower bounds for standard and robust {G}aussian process bandit
  optimization.
\newblock In \emph{Proceedings of Machine Learning Research (PMLR)}, 2020.
\newblock https://arxiv.org/abs/2008.08757.

\bibitem[Cai et~al.(2021)Cai, Gomes, and Scarlett]{Cai21}
Cai, X., Gomes, S., and Scarlett, J.
\newblock Lenient regret and good-action identification in gaussian process
  bandits.
\newblock In \emph{Int. Conf. Mach. Learn. (ICML)}, 2021.

\bibitem[Chowdhury \& Gopalan(2017)Chowdhury and Gopalan]{Cho17}
Chowdhury, S.~R. and Gopalan, A.
\newblock On kernelized multi-armed bandits.
\newblock In \emph{Int. Conf. Mach. Learn. (ICML)}, 2017.

\bibitem[Contal et~al.(2013)Contal, Buffoni, Robicquet, and
  Vayatis]{contal2013parallel}
Contal, E., Buffoni, D., Robicquet, A., and Vayatis, N.
\newblock Parallel gaussian process optimization with upper confidence bound
  and pure exploration.
\newblock In \emph{Joint European Conference on Machine Learning and Knowledge
  Discovery in Databases}, pp.\  225--240. Springer, 2013.

\bibitem[Eggensperger et~al.(2013)Eggensperger, Feurer, Hutter, Bergstra,
  Snoek, Hoos, and Leyton-Brown]{eggensperger2013}
Eggensperger, K., Feurer, M., Hutter, F., Bergstra, J., Snoek, J., Hoos, H.,
  and Leyton-Brown, K.
\newblock {Towards an empirical foundation for assessing Bayesian optimization
  of hyperparameters}.
\newblock In \emph{{NIPS} Workshop on Bayesian Optimization in Theory and
  Practice}, 2013.

\bibitem[Garcelon et~al.(2020)Garcelon, Roziere, Meunier, Teytaud, Lazaric, and
  Pirotta]{Gar20}
Garcelon, E., Roziere, B., Meunier, L., Teytaud, O., Lazaric, A., and Pirotta,
  M.
\newblock Adversarial attacks on linear contextual bandits.
\newblock In \emph{Conf. Neur. Inf. Proc. Sys. (NeurIPS)}, 2020.

\bibitem[{GPy}(2012)]{gpy2014}
{GPy}.
\newblock {{GPy}: A Gaussian process framework in python}.
\newblock \url{http://github.com/SheffieldML/GPy}, 2012.

\bibitem[Harris et~al.(2020)Harris, Millman, van~der Walt, Gommers, Virtanen,
  Cournapeau, Wieser, Taylor, Berg, Smith, Kern, Picus, Hoyer, van Kerkwijk,
  Brett, Haldane, del R{\'{i}}o, Wiebe, Peterson, G{\'{e}}rard-Marchant,
  Sheppard, Reddy, Weckesser, Abbasi, Gohlke, and Oliphant]{harris2020array}
Harris, C.~R., Millman, K.~J., van~der Walt, S.~J., Gommers, R., Virtanen, P.,
  Cournapeau, D., Wieser, E., Taylor, J., Berg, S., Smith, N.~J., Kern, R.,
  Picus, M., Hoyer, S., van Kerkwijk, M.~H., Brett, M., Haldane, A., del
  R{\'{i}}o, J.~F., Wiebe, M., Peterson, P., G{\'{e}}rard-Marchant, P.,
  Sheppard, K., Reddy, T., Weckesser, W., Abbasi, H., Gohlke, C., and Oliphant,
  T.~E.
\newblock Array programming with {NumPy}.
\newblock \emph{Nature}, 585\penalty0 (7825):\penalty0 357--362, September
  2020.
\newblock \doi{10.1038/s41586-020-2649-2}.
\newblock URL \url{https://doi.org/10.1038/s41586-020-2649-2}.

\bibitem[Jaquier et~al.(2020)Jaquier, Rozo, Calinon, and
  Bürger]{jaquier2020bayesian}
Jaquier, N., Rozo, L., Calinon, S., and Bürger, M.
\newblock {Bayesian Optimization meets Riemannian Manifolds in Robot Learning}.
\newblock In \emph{{Conference on Robot Learning}}, pp.\  233--246. PMLR, 2020.

\bibitem[Jun et~al.(2018)Jun, Li, Ma, and Zhu]{Jun18}
Jun, K.-S., Li, L., Ma, Y., and Zhu, J.
\newblock Adversarial attacks on stochastic bandits.
\newblock In \emph{Conf. Neur. Inf. Proc. Sys. (NeurIPS)}, 2018.

\bibitem[Li \& Scarlett(2021)Li and Scarlett]{Li21}
Li, Z. and Scarlett, J.
\newblock Gaussian process bandit optimization with few batches.
\newblock https://arxiv.org/abs/2110.07788, 2021.

\bibitem[Lykouris et~al.(2018)Lykouris, Mirrokni, and Paes~Leme]{Lyk18}
Lykouris, T., Mirrokni, V., and Paes~Leme, R.
\newblock Stochastic bandits robust to adversarial corruptions.
\newblock In \emph{ACM Symp. Theory Comp. (STOC)}, 2018.

\bibitem[Martinez-Cantin et~al.(2018)Martinez-Cantin, Tee, and McCourt]{Mar18}
Martinez-Cantin, R., Tee, K., and McCourt, M.
\newblock Practical {B}ayesian optimization in the presence of outliers.
\newblock In \emph{Int. Conf. Art. Intel. Stats. (AISTATS)}, 2018.

\bibitem[Nogueira et~al.(2016)Nogueira, Martinez-Cantin, Bernardino, and
  Jamone]{Nog16}
Nogueira, J., Martinez-Cantin, R., Bernardino, A., and Jamone, L.
\newblock Unscented {B}ayesian optimization for safe robot grasping.
\newblock In \emph{IEEE/RSJ Int. Conf. Intel. Robots and Systems (IROS)}, 2016.

\bibitem[Rasmussen(2006)]{Ras06}
Rasmussen, C.~E.
\newblock Gaussian processes for machine learning.
\newblock MIT Press, 2006.

\bibitem[Rolland et~al.(2018)Rolland, Scarlett, Bogunovic, and
  Cevher]{rolland2018high}
Rolland, P., Scarlett, J., Bogunovic, I., and Cevher, V.
\newblock {High-dimensional Bayesian optimization via additive models with
  overlapping groups}.
\newblock In \emph{{Int. Conf. Art. Intel. Stats. (AISTATS)}}, pp.\  298--307,
  2018.

\bibitem[Snoek et~al.(2012)Snoek, Larochelle, and Adams]{snoek2012practical}
Snoek, J., Larochelle, H., and Adams, R.~P.
\newblock {Practical Bayesian optimization of machine learning algorithms}.
\newblock In \emph{{Conf. Neur. Inf. Proc. Sys. (NIPS)}}, pp.\  2951--2959,
  2012.

\bibitem[Srinivas et~al.(2010)Srinivas, Krause, Kakade, and Seeger]{Sri09}
Srinivas, N., Krause, A., Kakade, S.~M., and Seeger, M.
\newblock Gaussian process optimization in the bandit setting: No regret and
  experimental design.
\newblock In \emph{Int. Conf. Mach. Learn. (ICML)}, 2010.

\bibitem[Swersky et~al.(2013)Swersky, Snoek, and Adams]{swersky2013multi}
Swersky, K., Snoek, J., and Adams, R.~P.
\newblock {Multi-task Bayesian optimization}.
\newblock In \emph{{Conf. Neur. Inf. Proc. Sys. (NIPS)}}, pp.\  2004--2012,
  2013.

\bibitem[Szegedy et~al.(2014)Szegedy, Zaremba, Sutskever, Bruna, Erhan,
  Goodfellow, and Fergus]{42503}
Szegedy, C., Zaremba, W., Sutskever, I., Bruna, J., Erhan, D., Goodfellow, I.,
  and Fergus, R.
\newblock Intriguing properties of neural networks.
\newblock In \emph{International Conference on Learning Representations}, 2014.
\newblock URL \url{http://arxiv.org/abs/1312.6199}.

\bibitem[Vakili et~al.(2020)Vakili, Khezeli, and Picheny]{Vak20a}
Vakili, S., Khezeli, K., and Picheny, V.
\newblock On information gain and regret bounds in gaussian process bandits.
\newblock https://arxiv.org/abs/2009.06966, 2020.

\bibitem[Vanchinathan et~al.(2014)Vanchinathan, Nikolic, De~Bona, and
  Krause]{10.1145/2645710.2645733}
Vanchinathan, H.~P., Nikolic, I., De~Bona, F., and Krause, A.
\newblock Explore-exploit in top-n recommender systems via gaussian processes.
\newblock In \emph{Proceedings of the 8th ACM Conference on Recommender
  Systems}, RecSys '14, pp.\  225–232, New York, NY, USA, 2014. Association
  for Computing Machinery.
\newblock ISBN 9781450326681.
\newblock \doi{10.1145/2645710.2645733}.
\newblock URL \url{https://doi.org/10.1145/2645710.2645733}.

\bibitem[Wang \& Jegelka(2017)Wang and Jegelka]{Wan17}
Wang, Z. and Jegelka, S.
\newblock Max-value entropy search for efficient {B}ayesian optimization.
\newblock In \emph{Int. Conf. Mach. Learn. (ICML)}, pp.\  3627--3635, 2017.

\bibitem[Zaharia et~al.(2018)Zaharia, Chen, Davidson, Ghodsi, Hong, Konwinski,
  Murching, Nykodym, Ogilvie, Parkhe, et~al.]{zaharia2018accelerating}
Zaharia, M., Chen, A., Davidson, A., Ghodsi, A., Hong, S.~A., Konwinski, A.,
  Murching, S., Nykodym, T., Ogilvie, P., Parkhe, M., et~al.
\newblock {Accelerating the Machine Learning Lifecycle with MLflow.}
\newblock \emph{{IEEE Data Eng. Bull.}}, 41\penalty0 (4):\penalty0 39--45,
  2018.

\end{thebibliography}
\end{document}